\documentclass[11pt]{article}

\usepackage[utf8]{inputenc}
\usepackage[T1]{fontenc}
\usepackage[margin=1in]{geometry}

\usepackage{amsmath, amssymb, amsfonts, amsthm, mathtools}

\theoremstyle{plain}
\newtheorem{theorem}{Theorem}[section]

\theoremstyle{definition}
\newtheorem{definition}[theorem]{Definition}

\theoremstyle{remark}

\usepackage{graphicx}
\usepackage{booktabs}
\usepackage{multirow}
\usepackage{caption}
\usepackage{subcaption}

\usepackage{algorithm}
\usepackage{algpseudocode}

\usepackage{nicefrac}
\usepackage{microtype}
\usepackage[dvipsnames]{xcolor}

\usepackage[numbers,sort&compress]{natbib}

\usepackage{tabularx}

\usepackage{url}
\usepackage{hyperref}
\hypersetup{
  colorlinks=true,
  linkcolor=MidnightBlue,
  citecolor=MidnightBlue,
  urlcolor=MidnightBlue
}

\long\def\comment#1{}
\def\Hof#1#2{H(\,#1\,|\,#2\,)}

\def\itm#1{\\[0.2ex] {\bf (#1)}}

\def\Pof#1#2{p_{#1}(\, #2\,)}
\def\CPof#1#2#3{\Pof{#1}{#2\,|\,#3}}
\def\aBBald{a_{BBald}}
\def\Dt{D}

\def\Info#1#2{I(\,#1\,|\,#2\,)}
\def\mmode#1{\\ \hspace*{1in} $\displaystyle #1$ \\}

\def\ie{{\em i.e.},\ }
\def\eg{{\em e.g.},\ }

\title{\textbf{%
Budget-constrained Active Learning\\
to Effectively De-censor Survival Data
}}

\author{%
  Ali Parsaee, Bei Jiang, Zachary Friggstad, Russell Greiner \\
  Department of Computing Science \\
  University of Alberta \\
  Edmonton, Alberta, Canada \\
  \texttt{\{parsaee, bei1, zacharyf, rgreiner\}@ualberta.ca}
}

\begin{document}
\maketitle

\begin{abstract}
\label{abstract}

Standard supervised learners attempt to learn a model from a \textbf{labeled} dataset.  Given a small set of labeled instances, and a pool of unlabeled instances, a \textbf{budgeted learner} can use its given budget to pay to acquire the labels of some unlabeled instances, which it can then use to produce a model. Here, we explore budgeted learning in the context of  survival datasets, which include (right) censored instances, where we know only a lower bound on an instance’s time-to-event.  
Here, that learner can pay to (partially) label a censored instance 

– \eg to acquire the actual time 
for an instance [perhaps go from (3 yr, censored) to (7.2 yr, uncensored)], or other variants 
[\eg learn about one more year, so go from (3 yr, censored) to either 
(4 yr, censored) or perhaps (3.2 yr, uncensored)].
This serves as a model of real world data collection, where follow-up with censored patients does not always lead to uncensoring, 
and how much information is given to the learner model during data collection is a function of the budget and the nature of the data itself.

We provide both experimental and theoretical results for how to apply state-of-the-art budgeted learning algorithms to survival data and the respective limitations that exist in doing so. Our approach provides bounds and time complexity asymptotically equivalent to the standard active learning method BatchBALD. Moreover, empirical analysis on several survival tasks show that our model performs better than other potential approaches on several benchmarks.
\end{abstract}

\section{Introduction}
\label{sec:Introduction}

The success of a learner in building an effective model is fundamentally reliant on the quality of the data to which it has access. Data significantly influences the bias, variance, and generalizability of any model produced by a 
learner~\citep{Domingos12}.

However, data availability is often constrained by several factors, including the cost of labeling the data instances
and the need for sufficient diversity in the labeled data to ensure effective generalization to the target distribution. 
The ``\textit{data selection}'' challenge is deciding 
which unlabeled instances to label 
(from a pool of unlabeled instances), to build an effective model. For example, given a dataset with 50 labeled instances and a budget to acquire labels for 10 additional unlabeled instances from a large pool of unlabeled data, the goal is to leverage the initial 50 
labeled instances to identify the 10 whose labels will best augment the first 50 labels, to produce the most accurate prediction model.

This paper extends this technology to the domain of survival prediction, which involves time-to-event data, which is essential in various fields as 
the event could be
failure, recovery, or death, for an instance. 
In a breast cancer study, researchers could monitor 300 cancer patients over a 5-year span to measure the time until their death from breast cancer. The goal here is to develop a model to predict the time of a patient diagnosed with that  cancer. 
Similarly, in a financial study, the objective might be to produce a model that could predict the time until a stock crashes. Unlike standard regression tasks, survival models are often trained and evaluated on data where the time-to-event label is not always precisely observed; instead, for some instances, only a lower bound on the event time is known~\citep{SA_textbook}.

Censoring is common in real-world datasets
– \eg in the above medical example, some patients might leave the study early (and so be lost to follow-up), others might die of another non-cancer cause (think “hit by a bus”), and yet others might be alive when the study is completed. It is difficult to train (and evaluate) a model using such censored data, as most supervised learning approaches handle data. One na\"\i{}ve approach is to simply exclude these censored instances when building the model -- however, this can remove a large amount of information from the learning system (as often more than $50\%$ of the instances in the dataset are censored) and may result in a biased model – note this will typically remove the long-time survivors, 
etc.~\citep{Qi2024_Calibration,Turkson2021_Censoring}. Although there are now many methods to learn models from survival data~\citep{Haider2020_ISD}, very few methods have been developed for data selection with survival data~\citep{Incremental_AL}, and so far, no method has been developed for data selection that incorporates budget constraints. This paper addresses that gap. 

The field of \textit{Active Learning (AL)} aims to identify which unlabeled instances to label, to construct the dataset that can produce the best model. Its goal is to minimize the amount of additional training data needed to achieve a target accuracy~\citep{Ren2021}.
Although AL is a well-established field with many results, its specific objective is often imprecise, as many papers discuss training models until ‘convergence’, then compare how quickly different algorithms get to this point. For example, convergence is commonly declared when the loss changes by less than a small value, $\epsilon$, over a predefined number of time steps referred to as ``{\em patience}"~\citep{Ren2021}. This imprecision arises because both $\epsilon$ and {\em patience} are chosen heuristically, with their values varying between tasks and models. As a result, the definition of true convergence remains unclear and inconsistent. 

The field of \textit{Budgeted Learning (BL)} addresses this issue by focusing on selecting data instances {\em within a budget} to minimize model loss, thus providing a clearer framework for active learning~\citep{Budgeted_learning_Naive_Bayes}.
BL augments AL by giving the learner a specific overall budget $B \in \Re^+$, before it begins data selection, to spend to obtain labels. A good budgeted learning should perform well at \textbf{any} given budget (rather than only at large budgets or at converging faster). Here, we apply BL to real-world survival datasets, where the budgeted learner can pay to (partially) “de-censor” a censored instance by determining the status of that instance $k$-years in the future:
\eg learn about $k$ = 1 more year, so go from (3yr, censor) to perhaps (3.2yr, uncensored) or (4yr, censor).
We refer to this $k$ as the ``probe depth'', which is an input to our BL system, and
note that, if $k = \infty$, every probed censored instance 
would be uncensored (or censored at the latest possible time knowable by the oracle).

In some cases, the cost of each $k$-year probe is the same for all instances -- which we call the \textit{uniform setting}.
We will also consider the non-uniform situation, where
there may be different costs for obtaining $k$-year information about different instances.\\[-1ex]

\textbf{Significance of our Work:} 
The settings explored in this work cover a wide range of applications;
(1)~Consider that you can conduct a follow-up 5-year clinical study on some of the patients in the study mentioned above. This is equivalent to having an oracle that informs for $k = 5$ more years of the time to event for any participant in the study.
(2)~Industrial reliability settings explore the time until failure~\citep{Reliability} -- \eg testing time until a light bulb stops working -- where a cost may be attached to testing for an additional $k$ days. 
(3)~Many computational problems can be framed as survival problems where the event in question is the time an algorithm will finish~\citep{kevinLBSATtosurvival}. Running the algorithm for an additional $k$ minutes comes with
a given price (compute for example) as well.
Everyone of the settings above can be imagined with unique instance costs (different light bulbs may cost different amounts to test or different algorithm inputs may require different compute for the same amount of time). Appendix A.1 discusses further the difference between the terms used in this paper (such as BL, AL, SA etc..) as well as how our work contributes to these areas of study.

These settings force us to extend the known BatchBALD algorithm to scenarios it was not designed for. For example, BatchBALD assumes that the oracle provides COMPLETE information about the label of an instance; here, however, the oracle sometimes provides only partial information (as mentioned above). It proves useful to value instances whose predicted time-to-event is close to their current censored time as the oracle has a higher chance of revealing their true time. 
It is also important to value patients censored at a later time as more information is known about them. BatchBALD also does not account for individual instance costs.

{\bf Our Contributions:}
%
\itm{1} We formally define a budgeted learning problem of de-censoring censored instances in a dataset with the goal of 
using the resulting enhanced dataset to learn a model with maximal performance, under a specified budget constraint.
\itm{2} We present a comprehensive review of alternative data selection methods applicable to our problem setting and introduce a novel algorithm, $BB_{surv}$ by adjusting BatchBALD~\citep{BatchBALD_Paper},
to work with both survival data and different instance costs. We show that our algorithm achieves the optimal lower bound within polynomial time. 
\itm{3} We conduct a comprehensive empirical analysis, using various evaluation metrics on three real-world survival datasets, demonstrating that our method consistently performs well in diverse scenarios, showing robust generalization even under variable instance costs.
\itm{4} As we do not know of any other tools for specific task, 
we present eight other 
plausible
data selection algorithms to compare against our method and 
then 
demonstrate that our method outperforms them, achieving superior results across 
various
evaluation metrics. \vspace{-1em}
\section{Related Work} 
\label{sec:Related Work}\vspace{-1em}
While there are many active learning methods, batch active learning focuses specifically on scenarios where a predefined number of instances are selected at once, rather than the common “sequential active learning” approach of iteratively selecting a few instances to label at a time~\citep{Ren2021}. 
Batch active learning is particularly advantageous for deep learning applications, where large amounts of data are needed to achieve meaningful improvements in model performance, making retraining after each individual instance computationally expensive \citep{Ren2021}. Various batch active learning methods have been proposed, including uncertainty sampling, which selects the data instances where the model has the least confidence \citep{Liu_Uncertainty_Sampling}; query-by-committee, which trains multiple models  and selects the data instances with the highest disagreement \citep{Burbidge_Active_Learning}; and diversity-based approaches, which focus on selecting diverse samples to enhance model generalization  \citep{Moses_Diversity_Based_Active_Learning}. \vspace{-1em}
\subsection{Applying  Active Learning to Survival analysis}
\label{sec:Related_work:AL to SA}
There are very few works in the intersection of AL and survival analysis. \citet{COX_with_Survival,COX_with_Survival2} employ a semi-parametric deep learning model based on the Cox Proportional-Hazards framework \citep{Cox_paper} --
a well-known model in survival analysis
that 
assumes proportional hazard ratios that are constant relative risks between individuals, irrespective of time 
(note this assumption does not 
hold for all datasets~\citep{Bayesian_model}). 
Moreover, 
they do not incorporate budget constraints nor consider the partial de-censoring of labels. 

\citet{BALD_with_Survival} extends the BALD framework to right-censored data by introducing ``C-BALD",
which models both the probability of censorship and the distributional uncertainty in censored regression, enabling principled active learning under right-censored labels. 
Furthermore, \citet{Bemporad2023} proposes IDEAL (Inverse-Distance based Exploration for Active Learning), an active learning method for regression that selects new queries by computing the inverse distance weighting based on the distance of a candidate point to previously queried samples. 
While these methods both perform well, and have provided us with ideas, 
our work extends beyond them in key ways. Our work employs the BatchBALD architecture, which selects diverse batches rather than independent points -- a known problem for BALD~\citep{BatchBALD_Paper}.
Additionally, our method adjusts the BatchBALD framework to incorporate budget constraints and probe depths, and to handle survival analysis with non-uniform instance costs. 
We compare our method against these two in Section~7.

Finally, \citet{Incremental_AL} propose an AL approach to survival data that includes a mechanism to incrementally update the label via an oracle, essentially the same as our partial de-censoring.
However, the amount of information gained by a query is determined randomly -- that is, for each probed instance,
the oracle arbitrarily picks a probe depth between the given censor time and the patient's actual time to death -- and so provides no way to incorporate this into the decision-making of the acquisition function. In addition, they do not address scenarios where instances can have different costs.\vspace{-1em}
\section{Formulation of the problem}
\label{sec:Formulation of the problem}
In this section, we precisely formulate our task. Initially, the budgeted learner is given 
(1)~a “probe depth” $k \in \Re^+$, which indicates how many years each probe of a (censored) instance will explore, 
(2)~a budget $B \in \Re^+$, which is the total amount the learner gets to spend for all of its probes, and 
(3)~a training dataset  $D = \{[x_i, t_i , \delta_i, c_i ]\}^L_{i=1}$, 
where $L$ is the size of the training data, $x_i$ represents the covariates of the $i$-th instance, 
$t_i$ denotes the time for that instance, and $\delta_i$ denotes the associated censored bit 
(1 if uncensored, meaning $t_i$ is the time of the event and 0 if censored, meaning $t_i$ is a lower bound), and $c_i$ is the cost of probing this instance. 

To explain the “probe depth”, recall the breast cancer study previously mentioned in the introduction. A probe depth of $k=5$ means the oracle will give the time of breast-cancer death of a person being probed 
if it occurred within the next 5 years, 
and if not,
will leave the patient censored but at the later time (5 years later). So if the patient was initially (2.4, censored),
the probe might update this to be
(3.5, uncensored) if she died 1.1 years after the earlier censor time,
or (7.4, censored) if she remained alive 5 years after that first time...
or perhaps (4.7, censored) if she died of a non-breast-cancer cause, 
2.3 years later.

In general, the budgeted survival learner 
will use an ``acquisition function"
to identify which set of instance indices $F$ 
to send to the oracle, 
to obtain this (partial) de-censoring information
-- \hbox{\ie} the new $\{[t_i, \delta_i]\}_{i \in F}$ 
based on 
{$B$, $k$ and the instances, including their costs}. 
Note the costs of the probed instances must not exceed $B$: $\sum_{j \in F} c_j \leq B$.
Our approach differs from traditional AL formulations as 
it allows (partial) de-censoring, and 
will only probe for information one time and then evaluate the performance of a learned model trained on the newly enriched dataset, 
rather than repeating the probing-and-model-building process multiple times. 

We define the model that will be trained in this protocol. We have a Bayesian model $M$ where the model parameters $\omega$ follow the distribution $\CPof{}{\omega}{D}$.
While time is continuous, we chose to discretize it into a finite sequence of disjoint time bins $\{b_i\}^n_{i=1}$ -- 
perhaps the $n=5$ intervals \{[0,1), [1, 2), [2, 3), [3,4), [4, $\infty$)\}.
For the $i$-th instance, we have a  classification outcome $y \in \{b_j\}^n_{j=1}$, which identifies the interval containing $t_i$.
So here, if $t = 1.3$, this would be the second interval $b_2$; meaning $y= b_2$.

Our underlying learning system will produce a survival model that maps the time $y$ for each instance $x$ to a multinomial distribution,
$\{\CPof{}{y=b_i}{x,\omega,D}\}^n_{i=1}$  which can then be used to produce an Individual Survival Distribution (ISD),
which is a distribution representing the probability the event has not occurred at each time $t>0$ -- that is, \hbox{$\CPof{}{E > t}{x,\omega}$}, where $E$ denotes the time of the event. In the example bins stated above, for a given instance $i$, our model may output $[0.20,0.25,0.20,0.05,0.30]$, which is claiming $20\%$ chance of survival in the first bin (0-1) bin, $25\%$ in the second bin (1-2), and so on.
Using a large number of bins can effectively approximate continuous time, and we found that this simplification yields good results even without using a very large number of bins.

Finally, we want to evaluate the quality of our budgeted learning systems using Mean Absolute Error (MAE), 
but this is challenging as the test data also includes censored instances.  Here, we use the MAE-PO measure, as~\citet{Effective_evaluations} shows that this metric is effective and often closely estimates MAE. We also evaluate using other more traditional survival evaluation metrics as well. 


Our formulation assumes that \textbf{Censoring is independent of the features}. This claim $C \bot X$ (where $C$ is the time of censoring, and $X$ describes the instance)
is known as 
\textit{non-informative censoring}~\citep{SA_textbook}.
{This assumption is required to show that our algorithm is optimal; 
our empirical evidence suggests that it still works effectively even when it does not hold.}\vspace{-1em}
\section{Our Method}
\label{sec:Method} 

The BatchBALD algorithm~\citep{BatchBALD_Paper} is a state-of the-art active learning (AL) method, which involves computing mutual information between multiple data instances and model parameters. 
Mutual information, rooted in information theory, quantifies the amount of information one random variable provides about another~\citep{Minimising_Entorpy}, and so how much an observation reduces uncertainty about a model’s parameters.
If the mutual information between two variables A and B is 
v 
bits, then knowing A’s value reduces the uncertainty of B to $\frac{1}{2^v}$ of the earlier amount. 

\citet{Minimising_Entorpy} establishes that, in the asymptotic limit, objectives such as maximizing the determinant of the Fisher Information Matrix, minimizing the L2 Bayes risk, and minimizing the Expected Prediction Error (EPE) ultimately converge to the same optimal solution as the number of instances approaches infinity. \citet{BatchBALD_Paper} define the BatchBALD acquisition function using mutual information as
\mmode{\aBBald(\,\{x_{1:b}\},\, \CPof{}{\omega}{ \Dt}\,)
  \quad = \quad \Info{\,y_{1:b} ;\, \omega}{x_{1:b},\, \Dt\,)}
  } 
   where $I(a, b | c)$ represents the mutual information of $a$ to $b$, conditioned on $c$. 
   \citet{BatchBALD_Paper} further define mutual information between the model parameters and a batch of $b$ data instances as
   :
\\[-2ex]
\begin{equation}
I(\,y_{1:b}; \,\omega \mid x_{1:b}, \,\Dt\,)\quad = \quad 
H(y_{1:b} \mid x_{1:b}, \Dt) \ -\ \mathbb{E}_{p(\omega \mid \Dt, x_{1:b})} \left[ H(y_{1:b}\,\mid \, x_{1:b},\, \omega,\, \Dt)\,\ \right]
\label{eq:MutInfo} 
\end{equation}
\\[-2ex]
where $H(y_{1:b} | x_{1:b}, \Dt)$ determines the information entropy of the labels of the batches given the features and training data, while
$E_{p(\omega|D_{train},x_{1:b})} [H(y_{1:b} | x_{1:b}, \omega, \Dt)]$ 
defines the expected labels conditioned on the features, 
training data and model parameters.
To simplify the equations, we do not show the
conditioning on $x_{1:b}$ and $\Dt$. 

We can compute the right term as:
\\[-2ex]
\begin{equation}
E_{p(\omega)} [H(y_{1:b} |\omega)]
\quad \approx \quad \frac{1}{k} \sum_{i=1}^{b} \sum_{j=1}^{k} 
\Hof{y_i}{\hat{\omega}_j}. \label{eq:right term}
\end{equation}
\\[-2ex]
\citet{BatchBALD_Paper} provide a detailed discussion of this factorization, with a key underlying assumption that, when conditioned on $\omega$, instances are treated as independent as their dependencies are captured within the parameters. Equation~\ref{eq:right term} estimates the expectation by taking $k$ samples of the model parameters. We can compute the left term of Equation~\ref{eq:MutInfo} as: 
\begin{equation} 
\small
H(y_{1:b})\, \approx\, 
-\!\sum_{\hat{y}_{1:b}} \left( \! \frac{1}{k} \sum_{j=1}^{k} p(\hat{y}_{1:b} | \hat{\omega}_j) \right)\! 
\log\! \left( \frac{1}{k} \sum_{j=1}^{k} p(\hat{y}_{1:b} | \hat{\omega}_j) \! \right).
\end{equation}
%
%
%
The difference in averaging between the left and right terms of the mutual information means they only cancel out 
(and make Equation~\ref{eq:MutInfo} equal 0) 
when the model outputs show minimal variation, indicating high confidence and low information gain. When model uncertainty is high, the entropy of the model (left term) increases, while the expected entropy of predictions (right term) decreases, leading to higher mutual information \citep{BatchBALD_Paper}. BatchBALD was originally designed for classification tasks with mutually exclusive classes. It can be applied to time-series tasks by dividing time into exclusive bins –-
see Section~\ref{sec:Formulation of the problem}.\vspace{-1em}
\subsection{Adjusting for Survival Data} 
\label{sec:Our Method:adj-surv}
We can modify the probabilities assigned by the learned model, by setting to 0 the probability of the event occurring in any time bin before the censoring time for an instance. 
If the the instance's censored time is in the $j$th bin, $b_j$, then \( p_{cens}(y = b_i \mid \omega) = 0 \) for all $i < j$. 
We can then normalize the probabilities  to get new probabilities
\mmode{ 
\CPof{cens}{y = b_i}{\omega} 
\quad=\quad \frac{\CPof{}{y = b_i }{\omega}}{\sum_{r=i}^{n} \CPof{}{y = b_r }{\omega}}. } 

\begin{figure*}[b] 
\begin{center}
 \includegraphics[width=\linewidth]{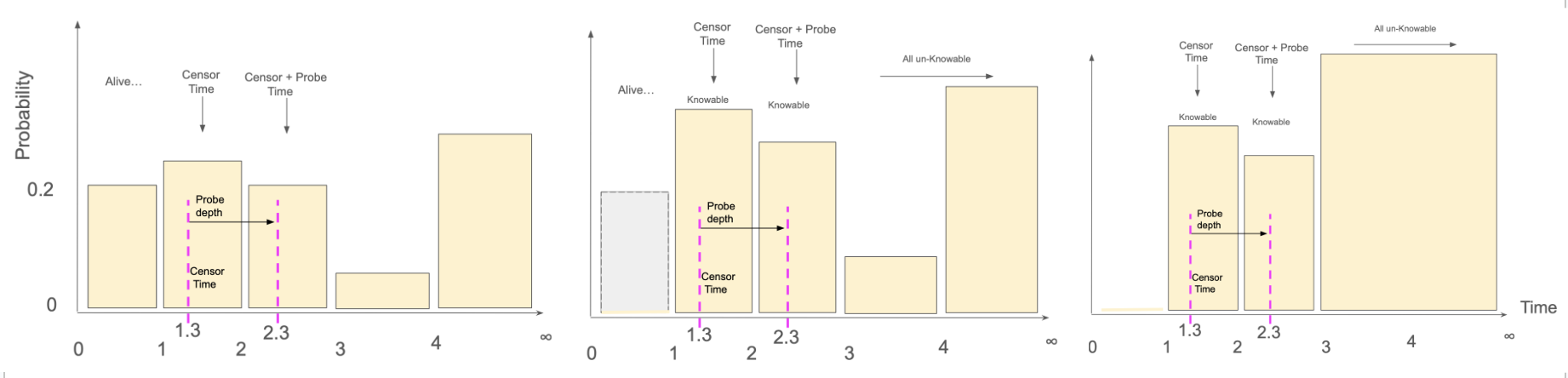}
\end{center}
\vspace*{-3ex}
\caption{\label{fig:TweakBBald}
Showing how we modify the probabilities assigned to each of 5 time bins, for a person censored at a time 1.3, when the Probe Depth $k=1$.
(left)~Original probabilities for the 5 bins [ 0.20, 0.25, 0.20, 0.05, 0.30].
(middle)~Showing  $p_{cens}$, which corrects as the patient is alive for the first bin, so has 0 probability of dying.
(right)~Collecting all of the ``un-knowable bins'' into one: $p_{final}$.}
\comment{
https://docs.google.com/presentation/d/1aARcqjYAIQr8teYhZp7zK5YONAwF6c30vnje6M5T5Gk/edit#slide=id.g32b59c4b56e_0_113 
Notes for next iteration:  the "4" is not level.  Larger fonts for "Censor time", etc.  Should use 2.1 rather than 2.3 (as an interval is [2,3).  Compute the actual heights, more correctly.
}
\end{figure*}

We show this in Figure~\ref{fig:TweakBBald}:
As the patient's censoring time $c=1.3$, is in the second bin [1,2), there is 0 chance they will die in the first bin [0,1);
the middle plot shows the distribution where this probably is now 0,
and the other 4 bins are all multiplied by 
$1/(1 - 0.2)$ (So of course, these 4 probabilities add up to 1). Now notice there are still 4 bins with non-0 probability of dying. Given the censor time of $c =1.3$ and probe depth $k=1$,
we will only learn about the patient until the time 
$1.3+1 = 2.3$, which means we might find that they died later in the second bin, 
or perhaps in the third bin, $b_3\, =\, [2, 3)$,
but otherwise, we will just see that this patient is censored.
We therefore view bins 2 and 3 as (possibly) ``knowable''.

Recall that BatchBALD needs the probability of the label being in a class to compute its mutual information,
but for classes beyond the probe depth, 
as the label cannot be in them after probing, 
we consider these as "un-knowable", 
and view them all as a single class.
Here, as there is no chance of finding them dying in 
either the fourth or fifth bins, 
BatchBALD can combine them into a single 
unknowable bin $b_{uk}$,
with probably of (0.05 + 0.30)/(1- 0.2) = 0.4375
-- Figure~\ref{fig:TweakBBald}(right).

This means BatchBALD will then need to deal with just three remaining bins: $b_2$, $b_3$ and $b_{uk}$ -- where the probabilities of $b_2$ and $b_3$ are the same as they were within $p_{cens}$. In general, there might be many unknowable bins; they would all be combined, in this fashion. The result is the $p_{final}$ distribution. 
To specify precisely, 
we want BatchBALD to only see the classes
that are knowable (that intersect the interval
$[c,\ c+k]$), and then an additional class that represents the cumulation of all classes outside that range. Here, we create a new probability distribution $p_{final}$ that is equal to $p_{cens}$ for all bins before the bin that includes $c + k$, and only includes one other bin $b_{uk}$: 
\mmode{\CPof{final}{y = b_{uk}}{\omega} 
\quad=\quad \sum_{i=j + 1}^{n} \CPof{cens}{y = b_i }{\omega}}
where $j$ is the bin $c+k$ lies in. We denote \textbf{ $BB_{surv}$} to be the version of 
BatchBALD that uses the $p_{final}$ probabilities. Furthermore, we let $a_{BB_{surv}}$ denote the acquisition function that takes in the data pool to further label the model and returns a batch using $BB_{surv}$. 
For all methods other than $BB_{surv}$, we use the probabilities of $p_{cens}$ rather than $p_{final}$, 
since our goal is to compare our method, which incorporates $p_{final}$, against controls that do not use it. 

\subsection{Maximum Coverage and Generalizing over Uniform Costs}
\label{sec:Our Method:Max coverage and generalising}
Given the BatchBALD implementation  already mentioned, we have a metric attached to each batch of instances that tells us how informative each batch is for the learner. The problem is now a strictly computational problem about choosing the batch that maximizes the metric while fitting within the budget. We call this problem the "Budgeted Learning Computation problem" (BLCP) and define it further in Appendix~A.4. The \textit{maximum coverage problem} is a combinatorial problem; given a set of elements and a collection of subsets, select a specified number of these subsets such that the total number of distinct elements covered by the chosen subsets is maximized \citep{Budgeted_MC}. The \textit{weighted maximum coverage problem} extends this by associating a weight with each element, and the objective is to select subsets to maximize the total weight of the covered elements. Appendix~A.3, further discusses maximum coverage and Appendix~A.4, reduces BLCP with uniform costs to the weighted maximum coverage problem. The \textit{budgeted maximum coverage problem} extends the weighted version by introducing a cost constraint to each set. In this variant, each subset has an associated cost, and the goal is to select subsets such that the total coverage is maximized while keeping the total cost within a given budget. The non-uniform BLCP problem can reduce to the budgeted maximum coverage problem. Fortunately, \citep{Budgeted_MC} introduces an alternative greedy algorithm that achieves the same desirable lower bound of $(1-1/e)$ of the optimal solution as the original greedy algorithm. We show this new greedy method in Algorithm 2 in the Appendix~A.5.

Algorithm~2 is very computationally expensive as it involves costly operations to meet the theoretical bound.  Appendix~A.5, argues that Algorithm 2 can be simplified and approximated for our settings by focusing  on the ratio of mutual information to the cost of an instance. This simplification reduces the algorithm's complexity from cubic to linear in the number of instances. Although designed for non-uniform instance costs, the new greedy algorithm simplifies to the original greedy algorithm when costs are uniform, making it applicable to both the uniform and non-uniform cost settings. Therefore, we will be using this algorithm for selecting the batch that attempts to maximize the mutual information metric. The reduced form of Algorithm~2 is shown in Algorithm~1 which illustrates our novel adaptation of the greedy algorithm, incorporating the proposed acquisition function, $a_{BB_{surv}}$. The time complexity of \( a_{BB_{\text{surv}}} \)is equivalent to that of the BatchBALD acquisition function.



\begin{algorithm}[t]
\caption{$BB_{Surv}$ ($1 - 1/e$ approximation)}
\begin{algorithmic}[1]
\Require Budget $B$, Pool $D = \{d_i\}_{i=1}^z$, Costs $\{c_i\}$, Model parameters $\omega$
\State $A \gets \emptyset$, $cost \gets 0$
\State $D \gets \{d_i \in D \mid c_i \leq B - cost\}$
\While{$D \neq \emptyset$}
    \State $d_j \gets \arg\max_{d_i \in D} \frac{a_{\text{BB}_{\text{surv}}}(A \cup \{d_i\}, p(\omega|\mathcal{D}_{\text{train}}))}{c_i}$
    \State $A \gets A \cup \{d_j\}$, $cost \gets cost + c_j$
    \State $D \gets D \setminus \{d_j\}$
    \State $D \gets \{d_i \in D \mid c_i \leq B - cost\}$
\EndWhile
\State \Return $A$
\end{algorithmic}
\end{algorithm}
BatchBALD requires a measure for the uncertainty of the output of the model for an instance to calculate the mutual information \cite{BatchBALD_Paper}. Bayesian models assign a probability distribution to each parameter, enabling the quantification of uncertainty within the model itself when making evaluations. Bayesian models often  provide similar performance as other regression models but they also provide a distribution for the values they provide, making it possible to measure confidence intervals and other important measures \citep{Bayesian_model}. However, Bayesian models are often more computationally expensive as they compute the distributions of the parameters in play. Furthermore, the chosen priors of the distribution heavily influence the outcome of the model, thus the use of effective priors is very important \citep{Bayesian_model}.

\citet{Bayesian_model} uses the Evidence Lower Bound (ELBO) for balancing prior distribution alignment (via KL divergence) and data likelihood maximization. 
The work adapts and tests two survival frameworks: 
(1)~\textbf{Cox Proportional Hazards} models hazard functions with baseline hazards and covariate effects \citep{Cox_paper}. 
(2)~\textbf{Multi-Task Logistic Regression (MTLR)} partitions time into discrete intervals and fits a series of logistic regressions across those intervals to estimate survival probabilities \citep{MTLR_paper}, thereby avoiding the hazard ratio assumption. We use the Bayesian MTLR model in \citet{Bayesian_model} since it directly utilizes censored observations during model training, systematically evaluates and identifies appropriate priors for different clinical scenarios, maintains reasonable computational complexity despite its Bayesian formulation, requires no proportional hazards assumptions, and provides per-instance event probability estimates within discrete time intervals. Furthermore, the work also looks at two separate types of priors:
(1)~\textbf{Spike-and-Slab}: combines Dirac delta (exact zeros) and Gaussian distributions for explicit feature selection. 
(2)~\textbf{Horseshoe}: allows strong shrinkage of irrelevant features while preserving significant effects. \vspace{-1em}
\section{Alternative Algorithms} 
\label{sec:AltAlgs} 
To evaluate our algorithm, it is important to compare it with other approaches.
We considered the following methods which include some active learning algorithms adjusted to our settings, and some hueristic sampling methods. We discuss these further in Appendix~A.6.

\textbf{Inverse Distance Weighting (IDEAL) Acquisition:}
Inverse Distance Weighting extends traditional IDW \citep{Bemporad2023} computes a weighted combination of model uncertainty and feature space exploration, $a(x)\ =\ s^2(x) + d \cdot z(x)$ where $s^2(x)$ represents uncertainty through weighted variance of predictions, $z(x)$ is a diversity term based on distances to existing observations, and $d$ is a parameter controlling the exploration-exploitation trade-off. We adjust IDEAL to first select instances based on information value, then prune the high-cost instances until it is within the budget.

\textbf{Censoring-aware Bayesian Active Learning by Disagreement (C-BALD):}
C-BALD adapts the BALD framework~\citep{Huttel2024} but processes entire survival curves rather than classification probabilities and implements differential weighting between censored and uncensored observations.

\textbf{Entropy Sampling:}
Entropy sampling is a well-established technique in AL \citep{Ren2021} that focuses on targeting instances where the model's predictive distribution has the highest entropy.
$H(y_i | x_i)\ =\ - \sum_{j=1}^n p_{\text{cens}}(y_r = b_j \mid x_i) \log_2 p_{\text{cens}}(y_r = b_j \mid x_i).$
The acquisition function then takes the instances with the largest entropy values until the budget is used up.

\textbf{Variance Sampling:} This method involves using variance in the predicted probabilities for an instance to assign a value to each instance,
using 
$ 
\text{Var}(x)\ \ = \ \ \frac{1}{n} \sum_{i=1}^n (p_{cens}(y = b_i \mid x) - \bar{p_i})^2
$, 
where \( \bar{p_i} \) is the mean predicted probability value among the bins for the ith instance. Take the highest-valued instances 
until the budget is used up.

\textbf{Closest to Half (CtH):}
Let \( p_{window,i} \) denote the predicted probability that the event will occur within the probe window based on the current learned model for the \( i^{th} \) instance. 
For each instance, we compute its absolute distance from 0.5:
$ 
dist_i = |p_{window,i} - 0.5|
$. The method  chooses instances with lowest $dist_i$ values until the budget is exhausted.

\textbf{Mean Closest to Middle (MCtM):}
Let the midpoint bin  $T_{\text{mid}} = \frac{b_n + b_1}{2}$.
For each data instance $i$, we calculate the distance to the midpoint $dist_i\ =\ |\bar{b_i} - T_{\text{mid}}|$, where \( \bar{b_i} \) is the mean predicted class. The method chooses instances with lowest $dist_i$ values until the budget is used.

\textbf{Clusters to form Batches (CfB):}
This method leverages clustering and censoring measures for instance selection. We use Principal Component Analysis (PCA) to reduce the feature space, followed by K-means clustering to group the data. Clusters with higher average censoring measures are prioritized, and instance selection is guided by proximity to cluster centers while respecting cost constraints. Distance
$dist_i =  \text{Proximity}(i, \mu_j) \cdot Censoring_j$ where $i$ is in cluster $C_j$.
For each cluster, we calculate the average censoring measure based on the time-to-event data and the censoring status. The 
clusters with higher average censoring measures represent areas with greater uncertainty or incomplete information. Within each prioritized cluster, instances are selected based on their proximity to the cluster center \( \mu_j \) as
$Proximity(i, \mu_j) = \| X_{\text{PCA}, i} - \mu_j \|$, where $X_{\text{PCA}, i}$ describes the lower-dimensional representation of instance 
$i$ in the PCA-transformed feature space. Given the $dist_i$ we choose the batch the same way we do for the other methods.

\textbf{Random:}
This method picks random data instances in such a way that the probability of choosing any given instance is proportionate to the reciprocal of the instance cost.\vspace{-1em}
\section{Experiments}
\label{sec: Experiments}

Datasets, are divided into training and test sets balanced to have a similar proportion of censored instances. Labels are assigned to time bins as quantiles of the time to event label. We use 10 bins (however preliminary experiments show similar results with more bins as well). We artificially censor points (censored points get further censored) in the training data so that we can de-censor them in our experiments. This can be done for censored data as well. We further censor $t_i$ and $\delta_i$ to $t^{leaner}_i$ and $\delta^{learner}_i$. We define $t^{leaner}_i$ as a uniformly random time between 0 and $t_i$ to maintain the non-informative censoring assumption discussed in Section~\ref{sec:Formulation of the problem}, and set $\delta^{learner}_i = 0$. The learner only has access to $t^{leaner}_i$ and $\delta^{learner}_i$, and upon probing the oracle for instance $i$, the oracle --with a given probe depth $k$-- returns to the learner $t^{probed}_i = \min(t^{leaner}_i + k , t_i)$, and $\delta^{probed}_i = \mathbf{1}_{\{ t^{probed}_i = t_i \}}*\delta_i$.

We explored 3 real-world survival datasets,
with an assortment of number of patients, number of features, and percentages of censored data: 
(MIMIC-IV; 38520 patients, 93 features, 67$\%$ censored)~\citep{MIMIC_data}, 
the {\em Northern Alberta Cancer Dataset}~(NACD; 2402 patients, 53 features, 36$\%$ censorship)~\citep{Haider2020_ISD}, and the {\em Study to Understand Prognoses Preferences Outcomes and Risks of Treatment}~(SUPPORT; 9,105 patients, 42 features, 32$\%$ censored)~\citep{SUPPORT_data}, 
{\em Medical Information Mart for Intensive Care}. 
We used 5000 epochs with a Bayesian MTLR model with the initial parameters provided in \citep{Bayesian_model} with a spike and slab prior. While Table~\ref{table:uniform:budget10} shows our primary  evaluation metric, 
MAE-PO~\citep{Effective_evaluations},
Appendix~B
shows results for other survival metrics, including C-index, 
integrated Brier Score 
and MAE with only uncensored data. While these experiments deal with the uniform-cost case
-- where the cost is 1 for each instance --
we also consider non-uniform, where we assign each instance a random real-value cost between 0.2 and 0.8. The MIMIC dataset required more data for the model to learn and for the MAE-PO value to change therefore, we set the costs of instances to $1/5$ that of the costs in SUPPORT and NACD. 
For all settings, each experiment involved a subset of instances
drawn from the full dataset, with 9 censored instances for each uncensored one. We used Google Colab with access to CUDA-enabled GPUs (NVIDIA Tesla T4). Each Colab session had approximately 12 GB of RAM and a standard disk quota.\vspace{-1em}
\section{Results} 
\label{sec:Results}

\begin{table}[t]
\centering
\caption{MAE-PO across datasets and probe depths, budget = 10. Uniform setting. M = MIMIC, N = NACD, S = SUPPORT, probe depth represented as +$k$y.}
\label{table:uniform:budget10}
\resizebox{\textwidth}{!}{%
\begin{tabular}{@{}lcccccccccc@{}}
\toprule
\textbf{Dataset} & \textbf{BB surv} & \textbf{BatchBALD} & \textbf{Entropy} & \textbf{Var} & \textbf{CtH} & \textbf{CfB} & \textbf{MCtM} & \textbf{Random} & \textbf{C-BALD} & \textbf{IDEAL} \\ \midrule
M +5y   & \textbf{4.23 $\pm$ 0.01} & 4.34 $\pm$ 0.02 & 4.28 $\pm$ 0.01 & 4.28 $\pm$ 0.02 & 4.45 $\pm$ 0.01 & 4.32 $\pm$ 0.02 & 4.29 $\pm$ 0.02 & 4.45 $\pm$ 0.01 & 4.34 $\pm$ 0.01 & 4.40 $\pm$ 0.01 \\
M +10y  & \textbf{4.25 $\pm$ 0.01} & \textbf{4.27 $\pm$ 0.02} & 4.28 $\pm$ 0.01 & 4.33 $\pm$ 0.02 & 4.31 $\pm$ 0.01 & 4.27 $\pm$ 0.02 & 4.28 $\pm$ 0.01 & 4.32 $\pm$ 0.02 & 4.32 $\pm$ 0.01 & 4.36 $\pm$ 0.01 \\
M +100y & \textbf{4.18 $\pm$ 0.02} & \textbf{4.18 $\pm$ 0.01} & \textbf{4.18 $\pm$ 0.02} & 4.27 $\pm$ 0.01 & 4.27 $\pm$ 0.02 & 4.23 $\pm$ 0.01 & 4.19 $\pm$ 0.02 & 4.31 $\pm$ 0.01 & 4.30 $\pm$ 0.02 & 4.31 $\pm$ 0.02 \\
N +5y   & \textbf{3.63 $\pm$ 0.01} & 3.67 $\pm$ 0.02 & 3.64 $\pm$ 0.01 & 3.66 $\pm$ 0.02 & 3.73 $\pm$ 0.01 & 3.80 $\pm$ 0.02 & 3.81 $\pm$ 0.01 & 3.85 $\pm$ 0.02 & 3.74 $\pm$ 0.01 & 3.80 $\pm$ 0.01 \\
N +10y  & \textbf{3.60 $\pm$ 0.01} & 3.63 $\pm$ 0.02 & \textbf{3.61 $\pm$ 0.01} & 3.65 $\pm$ 0.02 & 3.71 $\pm$ 0.01 & 3.64 $\pm$ 0.02 & 3.66 $\pm$ 0.01 & 3.80 $\pm$ 0.02 & 3.65 $\pm$ 0.01 & 3.69 $\pm$ 0.01 \\
N +100y & \textbf{3.57 $\pm$ 0.01} & \textbf{3.57 $\pm$ 0.01} & 3.60 $\pm$ 0.01 & 3.63 $\pm$ 0.02 & 3.70 $\pm$ 0.01 & 3.65 $\pm$ 0.02 & 3.68 $\pm$ 0.01 & 3.65 $\pm$ 0.02 & 3.65 $\pm$ 0.01 & 3.67 $\pm$ 0.01 \\
S +5y   & 2.10 $\pm$ 0.01 & 2.10 $\pm$ 0.01 & 2.10 $\pm$ 0.02 & 2.09 $\pm$ 0.01 & 2.09 $\pm$ 0.02 & \textbf{2.08 $\pm$ 0.01} & 2.09 $\pm$ 0.01 & 2.10 $\pm$ 0.02 & 2.11 $\pm$ 0.01 & 2.12 $\pm$ 0.01 \\
S +10y  & \textbf{2.06 $\pm$ 0.01} & 2.08 $\pm$ 0.02 & 2.10 $\pm$ 0.01 & 2.09 $\pm$ 0.01 & 2.09 $\pm$ 0.02 & 2.10 $\pm$ 0.01 & 2.09 $\pm$ 0.01 & 2.10 $\pm$ 0.02 & 2.10 $\pm$ 0.02 & 2.10 $\pm$ 0.02 \\
S +100y & \textbf{2.06 $\pm$ 0.01} & \textbf{2.06 $\pm$ 0.01} & 2.10 $\pm$ 0.02 & 2.09 $\pm$ 0.01 & 2.10 $\pm$ 0.01 & 2.10 $\pm$ 0.02 & 2.08 $\pm$ 0.01 & 2.10 $\pm$ 0.02 & 2.10 $\pm$ 0.02 & 2.10 $\pm$ 0.02 \\ \bottomrule
\end{tabular}%
}
\end{table}

Each cell in Table~\ref{table:uniform:budget10} represents the mean value as well as a 95$\%$ confidence interval. 
A bolded cell indicates that it is lower than 
the
unbolded cells in that row by a statistically significant portion (two sample t-test with p<0.05).
The table shows that, across 3 different probe depths for the oracle, $BB_{surv}$ outperforms other algorithms when 
the
budget is equal to 10.
This is promising as it shows that our method achieves better performance than other acquisition functions, 
including BatchBALD. 
This suggests that the method we used to account for the probe depth helped the model's performance. 
We also see that as the probe depth increases, the traditional BatchBALD method and our altered method converge -- which makes sense, as the original BatchBALD is theoretically equivalent to $BB_{surv}$ as $k \rightarrow\infty$. 
Appendix~B
shows more information in results including different metrics, and budgets. Figures 5 and 6 in Appendix~B show a visualization of the instances chosen in this setting using PCA, which suggests $BB_{surv}$'s superior performance is through increased diversity of instances chosen in these settings. In Table~\ref{table:uniform:budget10}, CfB occasionally ties with $BB_{surv}$ for the lowest MAE-PO values. 
CfB is a more involved method that, in certain budget settings, does surprisingly well. 
Note this unexpected finding does not hold in the non-uniform costs setting.

Figure~\ref{fig:mae_po_comparison_figure}(left) shows that, across budgets for all 3 datasets, $BB_{surv}$ does the best.
There appears to be more inherent randomness when we compare across budgets, as each time represents a new training of a model rather than the same model trained further. 
Figure~\ref{fig:mae_po_comparison_figure}(right) shows that similar to the uniform costs case, the non-uniform costs case shows that $BB_{surv}$ performs far better than all 8 other acquisition functions across budgets. However, now we note that $BB_{surv}$ and BatchBALD have an even more distinct advantage at the end. We believe this is because dealing with instance costs by dividing by the cost assumes a submodular metric (Appendix~A.4). 
Many of the other metrics do not measure mutual information, they do not reduce to the maximum coverage problem and thus they need to handle the non-uniform costs differently. Mutual information is a flexible metric that allows us to handle budget in our method. These results held consistently for other evaluation metrics.

\begin{figure}[t]
\centering
\begin{subfigure}{0.39\linewidth}
    \centering
    \includegraphics[width=\linewidth]{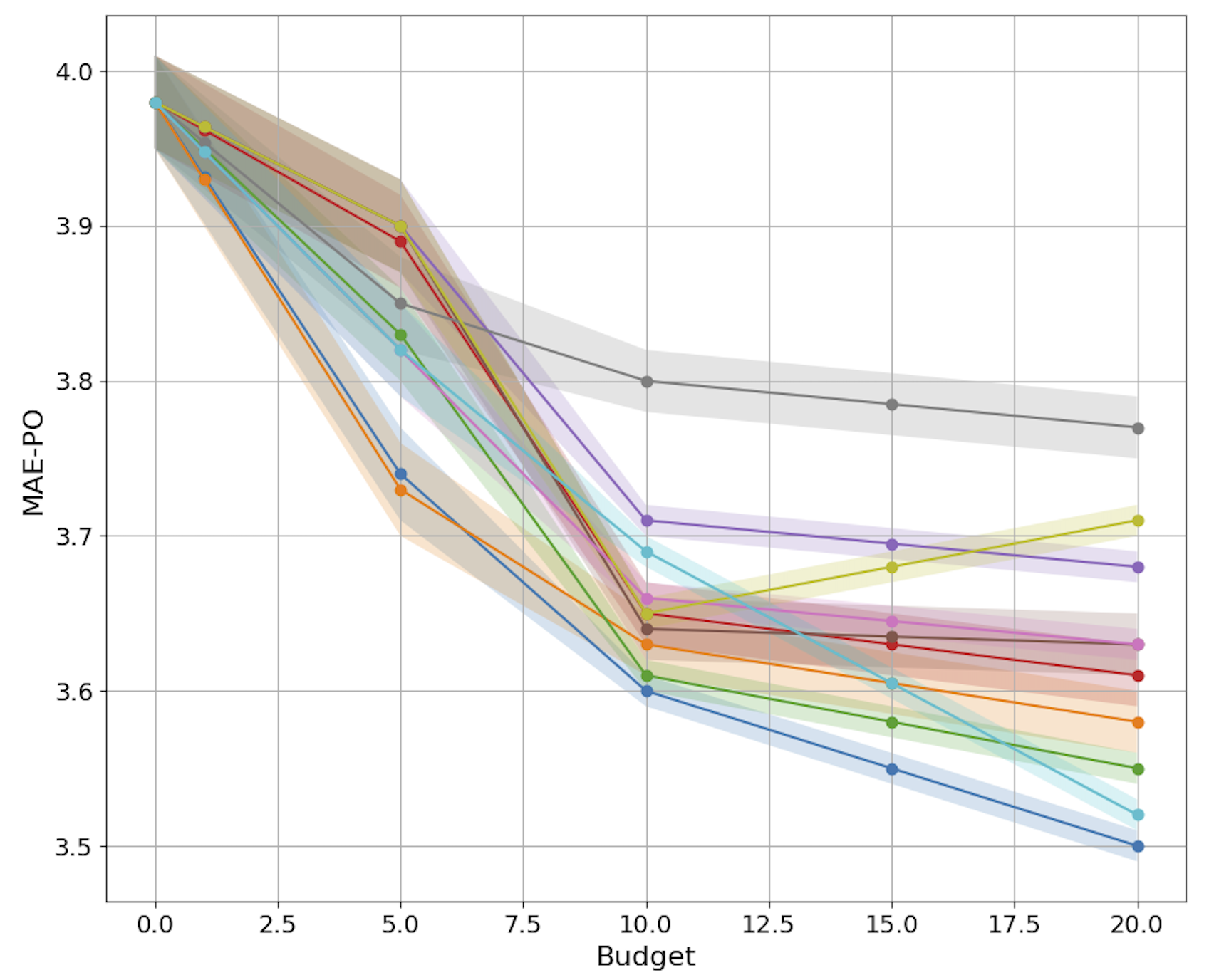}
\end{subfigure}%
\hfill
\begin{subfigure}{0.48\linewidth}
    \centering
    \includegraphics[width=\linewidth]{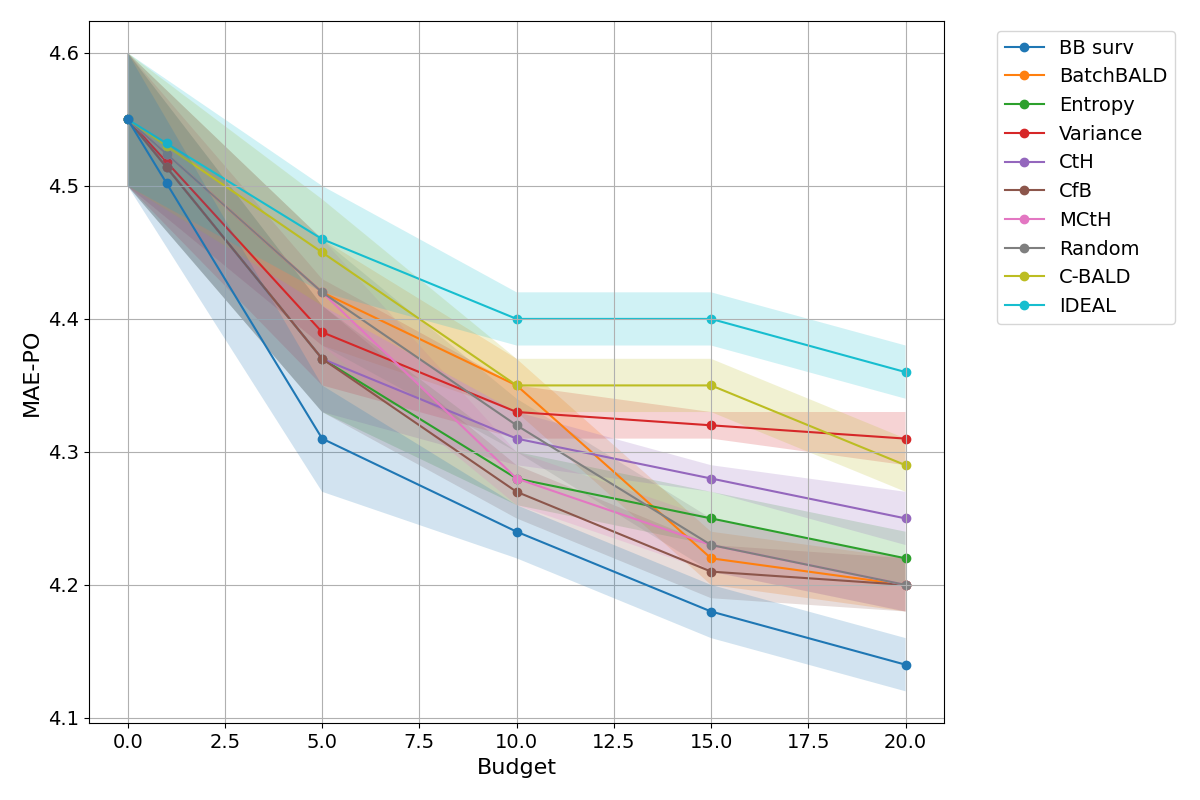}
\end{subfigure}
\vspace{-0.5em} 
\caption{Plot of MAE-PO evaluation of different acquisition functions with error bars as average of 40 predictions by the model. Both settings include a probe depth of 10. The left plot uses NACD dataset with 100 uncensored and 900 censored points in the uniform setting. The right plot uses MIMIC dataset starting with 300 uncensored and 2700 censored points in the non uniform setting.}
\label{fig:mae_po_comparison_figure}
\vspace{-1em} 
\end{figure}\vspace{-1em}
\section{Conclusion}
\label{sec:Conclusion}
This paper discussed a generalized form of active learning that incorporates budget constraints and the individual costs of queried instances. We explored methods to extend acquisition functions for use with censored data and to account for scenarios where only partial information is gained during queries. Our proposed method was evaluated across three real-world datasets; however, we anticipate that this emerging area of research will inspire many future studies. The study does have limitations that can be improved by future studies. In particular, one promising direction is to combine $BB_{surv}$ with a semi-supervised approach. As noted in \citet{BatchBALD_Paper}, such an approach may enhance performance, and given the success of CfB, we believe this is a worthwhile avenue to explore. Additionally, it would be valuable to revisit some of our assumptions, such as whether the probe depth changes over time or if we use more informative censoring (as apposed to non-informative censoring). Finally, there is a substantial body of literature on methods for approximating the maximum coverage problem. Alternative approximation schemes beyond the greedy approach may prove advantageous and warrant further investigation in future work.

\bibliographystyle{plainnat}
\bibliography{main}

\appendix

\section{Theoretical Analysis}
\label{sec:Theoretical Analysis}

\subsection{Motivation behind work}
\label{sec:Motivation behind work}

Below we provide a more in depth explanation of each field and terms we used in the paper, as well as how the terms mend together in our work. Importantly, we are not claiming that our method works only in settings where ALL the following areas are relevant, rather we try and build a general method that excels in ANY of the following domains! Figure ~\ref{fig: venn diagram} depicts the relations between the fields. Table~\ref{table: different works} shows how different works in the literature account for the different areas of study mentioned here.

\textbf{Survival Analysis:} A branch of statistics that deals with time-to-event data, often used to estimate the time until an event of interest (e.g., death, failure) occurs. It accounts for censored data, where the event may not be observed within the study period. In our work we specifically deal with right censored data where only a lower bound is known until the time of event. Practical applications of this field include:

\begin{enumerate}
    \item Predicting patient survival times in healthcare and clinical trials.
    \item Estimating time to failure for mechanical systems in reliability engineering.
    \item Modeling customer churn in business and subscription services.
    \item Estimating time until an algorithm completes
\end{enumerate}

Another important point is that survival analysis does not need to deal with time always. For example if estimating the price of an object, a censored instance may be if you only know a lower bound on said price.

\textbf{probe depth}
Since survival analysis involves regression tasks with partial information. A natural followup would be what if the information obtained from further study is also incomplete in some way. From this reasoning we say probe depth is a subset of survival analysis -- Shown in Figure~\ref{fig: venn diagram} where information may be limited by $k$ amount of years into the future.

\textbf{Active Learning:} A machine learning paradigm where the model selectively queries the most informative data points to label, aiming to improve performance with fewer labeled examples.

\textbf{Budgeted Learning:} A setting in machine learning where the learning algorithm must operate under a fixed resource constraint (e.g., labeling budget), requiring careful selection of data points to maximize utility within the budget.

\textbf{Non-Uniform Costs:}
A subset of Budgeted learning where each instance comes with its own (not necessarily distinct) price for gathering more information.

\begin{table}[h]
  \caption{Comparison of relevant works across study aspects}
  \label{table: different works}
  \centering
  \begin{tabularx}{\textwidth}{l*{5}{>{\centering\arraybackslash}X}}
    \toprule
    \textbf{Papers} & \textbf{Active learning} & \textbf{Budget constraints} & \textbf{Survival analysis} & \textbf{Probe depth} & \textbf{Non-uniform costs} \\
    \midrule
    Lizotte et al. (2012)     & \includegraphics[height=1em]{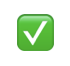} & \includegraphics[height=1em]{figures/check.png} & \includegraphics[height=1em]{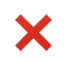} & \includegraphics[height=1em]{figures/cross.png} & \includegraphics[height=1em]{figures/cross.png} \\
    Dedja et al. (2023)       & \includegraphics[height=1em]{figures/check.png} & \includegraphics[height=1em]{figures/cross.png} & \includegraphics[height=1em]{figures/check.png} & \includegraphics[height=1em]{figures/cross.png} & \includegraphics[height=1em]{figures/cross.png} \\
    Hüttel et al. (2024)      & \includegraphics[height=1em]{figures/check.png} & \includegraphics[height=1em]{figures/cross.png} & \includegraphics[height=1em]{figures/check.png} & \includegraphics[height=1em]{figures/cross.png} & \includegraphics[height=1em]{figures/cross.png} \\
    \textbf{Our Work}         & \includegraphics[height=1em]{figures/check.png} & \includegraphics[height=1em]{figures/check.png} & \includegraphics[height=1em]{figures/check.png} & \includegraphics[height=1em]{figures/check.png} & \includegraphics[height=1em]{figures/check.png} \\
    \bottomrule
  \end{tabularx}
\end{table}

\begin{figure*}[h] 
\begin{center}
 \includegraphics[width=\linewidth]{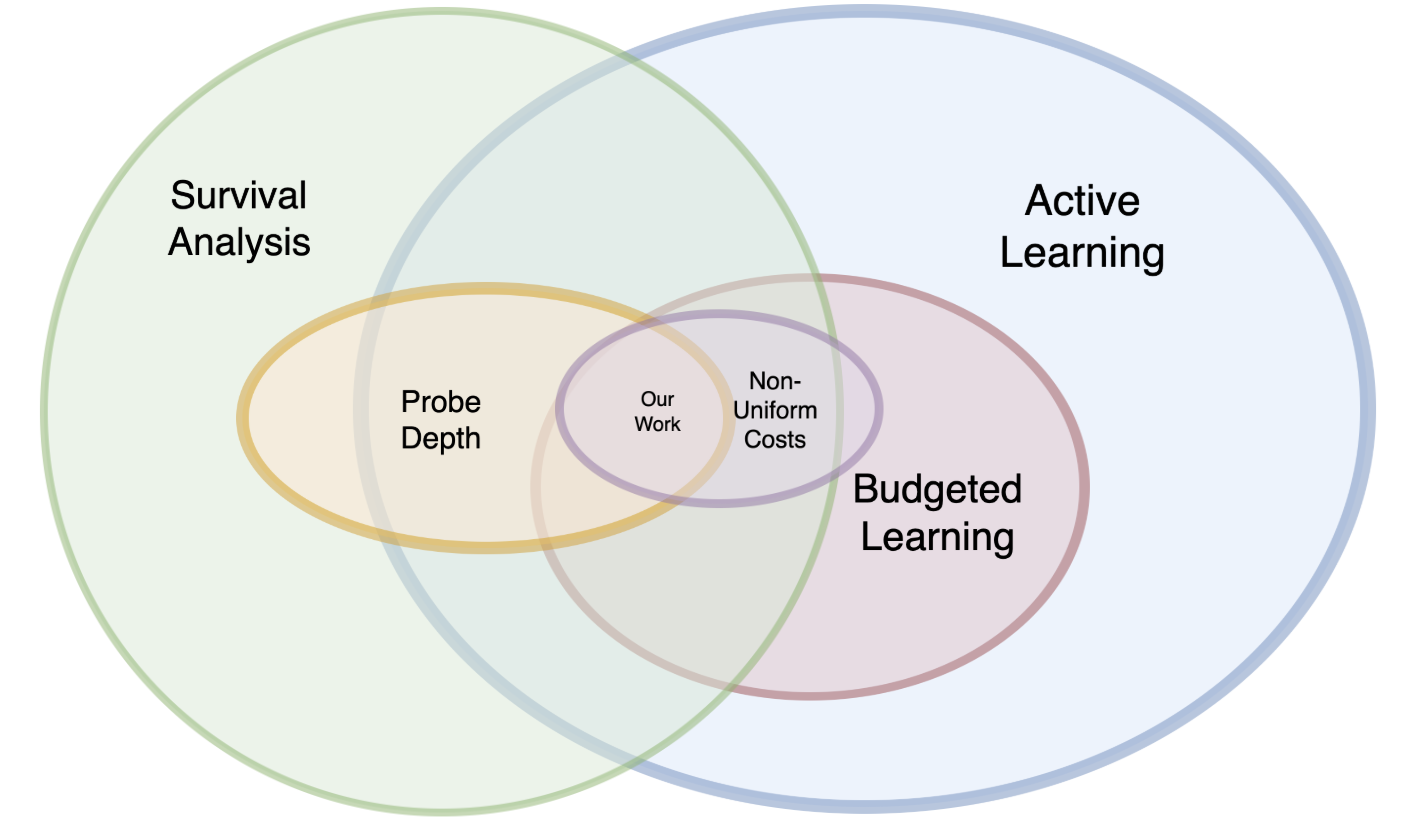}
\end{center}
\caption{A Venn Diagram of the different areas of study presented in this work.}
\label{fig: venn diagram}
\end{figure*}

\subsection{Comparison between BALD and BatchBALD}
\label{Appendix A:Appen bald v bb}

Works in the batch active learning domain include the BALD algorithm, which uses the disagreement between model predictions for a single data instance to compute a mutual information measure, and then selects the top \(K\) data instances to compose a batch of size \(K\). The BatchBALD algorithm extends this idea by computing the mutual information measure for an entire batch of instances (rather than a single instance at a time) and the model parameters, thereby reducing errors in BALD related to double-counting the information provided by multiple data points.

To illustrate this point, consider a 5-year study involving three patients: Alice, Bob, and Carol. You aim to build a model that predicts time-to-event based on instance features. Suppose Alice and Bob have nearly identical features, while Carol's features are distinctly different. BALD might assign a high value to querying Alice or Bob individually, but when selecting a batch of size 2, it evaluates the batch value as the sum of the individual query values. As a result, BALD might deem the batch \(\{ \text{Alice, Bob} \}\) more valuable than \(\{ \text{Alice, Carol} \}\). 

However, this approach overlooks redundancy: since Alice and Bob have highly similar features, the information gained from querying both might be nearly identical. Consequently, querying \(\{ \text{Alice, Carol} \}\) or \(\{ \text{Bob, Carol} \}\) could provide more diverse and useful information. BatchBALD addresses this limitation by evaluating the value of the entire batch as a whole, rather than summing the individual values of its elements. \citet{BatchBALD_Paper} demonstrates experimentally that BatchBALD outperforms BALD, particularly in scenarios where data instances exhibit less diversity. In fact, the work shows that if the diversity is low enough, BALD can perform worse than random sampling.

Although there are several other batch querying methods, BatchBALD has been one of the leading algorithms in the literature.

\subsection{Maximum Coverage Problem}
\label{Appendix A:Maximum Coverage Problem}

The \textit{maximum coverage problem} is a well known combinatorial problem that involves choosing $k$ sets, from a group of $N$ sets of elements. The task is to choose the $k$ sets whose union has the maximum possible size. For example if the sets are as follows: $S_1 = \{1,2,3\}$, $S_2 = \{2,3,4\}$, $S_3 = \{4,5\}$, $S_4 = \{6\}$, where here $N=4$ and $k=2$, then the optimal choice of sets here are sets $S_1$ and $S_3$, as their union $\{1,2,3,4,5\}$ is larger than the union of $S_1$ and $S_2$: $\{1,2,3,4\}$ or $S_2$ and $S_3 = \{2,3,4,5\}$, and all pairs with $S_4$ make at max 4. This problem is known to be {\bf NP}-hard to solve optimally. 

There is an extension of this problem called the \textit{weighted maximum cover problem} where everything is the same except each integer has attached to it a weight, and the goal is now to maximize the sum of the weight of the union rather than simply the size of the union. In the example above, if the integers 1,2,3,4, and 5 all had weight 1, but 6 had weight 10, then now we certainly would wish to include $S_4$ as part of one of the sets we choose, in this case you could choose $S_1$ and $S_4$, or $S_2$ and $S_4$, as both would give you a total highest weight of 13.

If you have a set of data instances $D$, then each instance $d_i \in D$, provides some amount of information. Any two data instances $d_i$ and $d_j$ also have an intersection to consider. Indeed for any batch of data points, if we consider information as area that is covered by a shape in the information space we can represent it visually as in Figure 4.

\begin{figure}[h]
  \centering
  \includegraphics[scale=0.2]{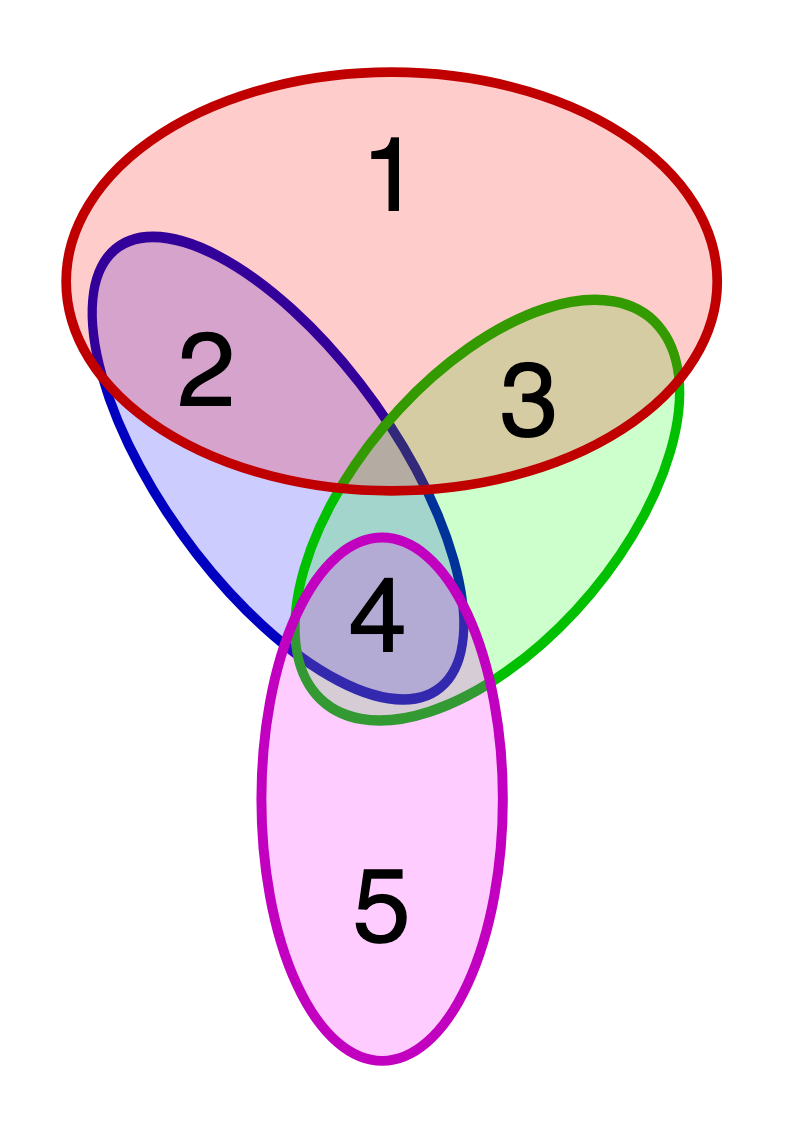}
  \caption{Representation of the maximum coverage problem visually. Each unique intersection between two sets is an intersection we denote as an element part of a set.} 
  \label{Figure 2}
\end{figure}

Assume in Figure~\ref{Figure 2} that blue represents a data instance in information space, while red represented another instance, and so on for many colors. We now label each unique intersection as an integer, then we now have a bunch of sets, each filled with integers. We can give each integer a weight as the amount of area they cover in information space (the amount of information they are expected to give the model). For example the integer ``2" in the figure would have area

\begin{equation}
I(\omega,red) + I(\omega,blue) - I(\omega, \text{red} \cup \text{blue})
\end{equation}

In the batch active learning case, the task now becomes choosing the k data instances (which are sets in this case) out of the N total data points that give maximal information cover maximal area). Thus we can simplify the active learning problem using this formulation of mutual information provided by BatchBALD into the weighted maximum  coverage problem discussed above, we provide a more formal proof of this in Section~\ref{Appen: Reductions}. This is a very useful representation of the problem as it allows an easier way to think about the machine learning problem as a simpler combinatorial one.

If the information function is a submodular function, then the greedy algorithm does very well~\cite{BatchBALD_Paper}. 

\begin{definition}
\label{def:submodular}
A set function \( f: 2^N \to \Re\) defined on the subsets of a finite set \( N \) is called \textbf{submodular} if for every \( A, B \subseteq N \),

$$
f(A) + f(B) \geq f(A \cup B) + f(A \cap B)
$$

Equivalently, \( f \) is submodular if it satisfies the \textbf{diminishing returns} property: for every \( A \subseteq B \subseteq N \) and \( x \in N \setminus B \),

$$
f(A \cup \{x\}) - f(A) \geq f(B \cup \{x\}) - f(B)
$$

\end{definition}

Importantly, with the submodular property shown above, mutual information has a property that the sum of marginal values is always greater than or equal to the union. For example if A and B are two instances, then: 

$$
I(A,\omega) + I(B,\omega) \geq I(A \cup B, \omega)
$$

This property is critical to the reduction to the maximum coverage problem described above and allows us to use the greedy approach defined in \cite{BatchBALD_Paper}.

The greedy approach for the maximum coverage problem involves iteratively selecting the element or set that provides the largest marginal gain in coverage at each step. It gives the highest known guaranteed lower bound of all polynomial time approximation algorithms. In fact proving there is a polynomial time approximation algorithm that achieves a higher lower bound is equivalent to proving {\bf P} = {\bf NP} \cite{Budgeted_MC}. The greedy algorithm discussed above achieves a lower bound of $1-1/e$ (around 63$\%$) of the optimal, and is the algorithm used in BatchBALD for selecting it's batch \cite{BatchBALD_Paper}.

\subsubsection{Further Research on Maximum Coverage}
\label{A.2.1}

The Maximum Coverage Problem is {\bf NP}-hard, and finding exact solutions is infeasible (by current known methods) for large problem instances. However, the problem can often be effectively approximated using heuristic methods.

Several methods (aside from the aformentioned greedy method) have been proposed to approximate the problem efficiently, including:

\begin{itemize}
    
    \item \textbf{Local Search:} 
    This method starts with an initial solution and iteratively improves it by making local changes, such as swapping, adding, or removing subsets. While local search can yield high-quality solutions, its effectiveness depends heavily on the chosen neighborhood structure and optimization strategy.
    
    \item \textbf{Metaheuristics:} 
    Techniques like simulated annealing, genetic algorithms, and tabu search explore a large solution space to identify near-optimal solutions. These methods are particularly useful for large or complex instances where greedy approaches may fall short.
    
    \item \textbf{Linear Programming Relaxation:} 
    The maximum coverage problem can be formulated as an integer linear program (ILP). By relaxing the integrality constraints, the problem reduces to a linear program (LP) that can be solved efficiently. Rounding the LP solution provides a good approximation to the optimal result.
\end{itemize}

Despite all of these different methods, we chose to use the greedy method for a few reasons. For one, the lower bound guarantee of the greedy method is optimal (as far as current theory knows) and thus none of these methods can do better. Secondly, the greedy method is more interpretable, more well accepted (and tested) in the literature for AL settings, and easier to generalize theoretically to the non-uniform setting in this work.

\subsection{Reductions of the computational Problem}
\label{Appen: Reductions}
\subsubsection{Reducing the Weighted Maximum Coverage Problem to the Budgeted Learning Computation problem}
\label{Appen: WMCP to BLCP}

In this section, we formally reduce the weighted maximum coverage problem (WMCP), under the same initial assumptions outlined in Section~\ref{sec:Formulation of the problem}, to the Budgeted Learning Computation problem (BLCP). The difference between the Budgeted Learning problem  and BLCP is that BLCP assumes the existence of an oracle $I$ with submodular and efficient computation properties that can assign an information value to any batch of data instances $B$. In the context of the paper, this is achieved using $BB_{surv}$. Assuming we have this oracle, selecting the future batch of instances to learn about becomes a purely computational problem.
\\\\
\textbf{The Budgeted Learning Computation problem is defined as:}
\\

Let $D_{pool}=\{x_1,\dots,x_n\}$ be a pool of $n$ instances, each endowed with a cost $c_i>0$, and let $B>0$ denote the total budget.  
Assume the existence of a polynomial‑time oracle that evaluates a non‑negative submodular function
\[
\mathcal{I}:\ 2^{D_{\text{pool}}}\ \longrightarrow\ \mathbb{R}_{\ge 0}.
\]
Our objective is to select a subset
\[
Q\subseteq D_{pool}
\]
that maximizes $\mathcal{I}(Q)$ while respecting the knapsack constraint
\[
\sum_{x_i\in Q} c_i \;\le\; B
\]\\\\\\
\textbf{The Weighted Maximal Coverage problem is defined as:}
\\

Let $\mathcal{U}=\{e_1,\dots,e_m\}$ be a ground set of element. Each element $e_i$ carries a non‑negative weight $w(e_i)\in\mathbb{R}_{\ge 0}$. 
Consider a family of subsets $\mathcal{S}=\{S_1,\dots,S_n\}\subseteq 2^{\mathcal{U}}$. We can assume here that $|U| = poly(|\mathcal{S}|)$.

The weighted maximal coverage problem seeks a set

\[
\mathcal{A}\subseteq\mathcal{S}
\]
where $|\mathcal{A}| = K$ for some $K \in \mathbb{Z}_{+}$,
that maximises the total weight of the elements it covers,
\[
\max_{\mathcal{A}}\;
\bigl\{\,W(\mathcal{A}) \;=\; \sum_{e\in\cup_{S\in\mathcal{A}}S} w(e)\bigr\}.
\]


\begin{theorem}
\label{thm:bigtheorem}
The weighted maximal coverage problem, can be reduced in polynomial time to the Budgeted Learning computational problem assuming the cost for all instances is the same.
\end{theorem}

\begin{proof} 

Since we assume the cost is the same, the BLCP problem now involves choosing $K = \lfloor B / c \rfloor$ instances. Now for the reduction, for a given instance of inputs to the WCMP: $\mathcal{U},\mathcal{S}$, we have to find a corresponding set of instances $D_{pool}$ and function $\mathcal{I}$ such that solving this BLCP means solving the WCMP instance. This can be done by:

\begin{enumerate}
    \item $D_{pool} = \{x_1,\dots,x_{|\mathcal{S}|}\}$
    \item For any subset $Q \subseteq D_{pool}$, $\mathcal{I}(Q) = W(\mathcal{A})$ where $A = \{S_i|x_i \in Q\}$
\end{enumerate}

In the first point we ensure that the number of instances in BLCP equals the number of sets in WCMP. In the second point we can set our information function $\mathcal{I}$ to be equal to the weight of the union of elements of corresponding sets in WCMP. Notice this does not break any of our assumptions about $\mathcal{I}$ as $W$ is both submodular and computed in polynomial time.

\end{proof}

This form of reduction has been referenced in other active learning literature \cite{AL_through_covering, BatchBALD_Paper}, although not as formally as defined above. However, our setting differs from traditional active learning in two key ways:

\begin{enumerate}
    \item \textbf{Survival data and probe depths:} Unlike typical settings, we deal with survival data and probe depths. Fortunately, as long as the information metric (e.g., BatchBALD) can be adjusted to account for these aspects—demonstrated through $BB_{\text{surv}}$—the rest of the reduction remains independent of these features.

    \item \textbf{Budgeted Learning:} Instead of standard active learning, we operate within the framework of Budgeted Learning, where a specific budget is enforced. This distinction is already incorporated in the reduction above, as it determines how many sets can be selected. In contrast, traditional batch active learning typically involves a predefined batch size for each query to the oracle. Moreover, the aim of active learning is to minimize the total number of oracle queries, rather than minimizing the loss (or maximizing the information gain) from a single query. This fundamental difference means that active learning cannot be directly reduced to the weighted maximum coverage problem in the way presented here.
\end{enumerate}

\subsubsection{Reducing the BCLP to a WCMP variant}

We can also reduce the BLCP to a problem very similar to the WCMP. However, this reduction is not polynomial in the size of the sets. 
\\\\
\textbf{We define the WCMP variant as:}
\\

Let $\mathcal{U}=\{e_1,\dots,e_m\}$ be a ground set where each element carries a non‑negative weight $w(e_j)\in\mathbb{R}_{\ge 0}$. 
Consider a family of subsets $\mathcal{S}=\{S_1,\dots,S_n\}\subseteq 2^{\mathcal{U}}$. Here we assume that $|U| \leq 2^{|\mathcal{S}|}$

The weighted maximal coverage variant seeks a set

\[
\mathcal{A}\subseteq\mathcal{S}
\]
where $|\mathcal{A}| = K$ for some $K \in \mathbb{Z}_{+}$,
that maximises the total weight of the elements it covers,
\[
\max_{\mathcal{A}}\;
\bigl\{\,W(\mathcal{A}) \;=\; \sum_{e\in\cup_{S\in\mathcal{A}}S} w(e)\bigr\}.
\]

Furthermore, in this variant we introduce an oracle that can inform about $W(\mathcal{A})$ in $poly(n)$ time.

The only differences between WCMP and this variant is that we allow the size of the set of elements to be exponential in n, and also we provide an oracle $O$ such that allowing this increase in size of elements does not come at additional cost to calcluating $W$. We can now reduce BCLP to this variant.

\begin{theorem}
\label{thm:bigtheorem}
BCLP, can be reduced in polynomial time to the WCMP variant assuming the cost for all instances is the same.
\end{theorem}

\begin{proof} 

To establish this proof we need to do the following steps.

\begin{enumerate}
    \item Define the number of sets we can select, which in our case is given by $\lfloor B / c \rfloor$ in the uniform setting.
    \item We must define the elements we are working with for the WCMP variant.
    \item We must use the defined elements to form sets in the WCMP variant, and these sets should map to choosing instances in the BLCP.
\end{enumerate}

For step 2, for our pool of instances $D_{pool}$ we can define $2^{|D_{pool}|}$ elements in the WCMP variant by taking all possible combination of indices from $D_{pool}$ to be an element. For example $\{i,j\}$ is one element in this maximum coverage setting as long as $i,j \in D_{pool}$. Define the  function $G({S}) = I(D_{pool}) - I(D_{pool} \setminus S)$. where $I(S)$ represents the information metric associated with the subset $S$. We can now inductively give weight to each element (subset of instances) as: \begin{equation}
w(E) \quad = \quad \sum_{k=1}^{|E|} (-1)^{k+1} \sum_{\substack{S \subseteq E \\ |S| = k}} G(S).
\end{equation}

For step 3, we can now define the matching set for instance $i$ as: \begin{equation}
\mathcal{P}(D_{pool}, i) \quad = \quad \{ S \subseteq D_{pool} \mid i \in S \},
\end{equation}

which is the set of all subsets containing instance $i$.

Notice that we have now assigned a weight to every possible combination of instances in $D_{pool}$ (every element). The weight of an element $e$ corresponds to the information contained in the intersection of instances in BLCP that define that element $\mathcal{I}(e)$. Thus for any set $\mathcal{A}$ composed of some $\mathcal{P}$ sets in WCMP, we can consider the set of instances corresponding to $\mathcal{P}$ as $Q$. We can see then that $W(A) = \mathcal{I}(Q)$. Therefore, any solution to this WCMP variant would select $Q \subseteq D_{pool}$ in BLCP that maximizes the $\mathcal{I}(Q)$.

\end{proof}

\subsection{Changing the Budgeted Maximum Coverage Algorithm}
\label{Appen: alg2}

In \cite{Budgeted_MC}, a novel greedy algorithm for the budgeted maximum cover case is provided that still provides the same lower bound guarantees as the original greedy algorithm. We illustrate this algorithm in Algorithm 2.

\begin{algorithm}
\caption{Optimal (1 - 1/e) -approximate algorithm for Budgeted Maximal Coverage}
\label{Alg: inefficient alg}
\begin{algorithmic}[1]
\Require Pool of points \( S \), budget \( B \), weights \( w_i \), costs \( c_i \), subset size \( z \)
\State \( H_1 \gets \arg\max \{ w(G) : G \subseteq S , |G| < z, c(G) \leq B \} \)
\State \( H_2 \gets \emptyset \)

\ForAll{$G \subseteq S$ such that $|G| = z$ and $c(G) \leq B$}
    \State \( U \gets S \setminus G \)
    \Repeat
        \State Select \( x_i \in U \) that maximizes \( \frac{w'_i}{c_i} \)
        \If{$c(G) + c_i \leq B$}
            \State \( G \gets G \cup x_i \)
            \State \( U \gets U \setminus x_i \)
        \EndIf
    \Until{$U = \emptyset$ or $c(G) + c_i > B$}
    \If{$w(G) > w(H_2)$}
        \State \( H_2 \gets G \)
    \EndIf
\EndFor
\If{$w(H_1) > w(H_2)$}
    \State \textbf{Output:} \( H_1 \)
\Else
    \State \textbf{Output:} \( H_2 \)
\EndIf
\end{algorithmic}
\end{algorithm}

In our own notation, the pool of points $S = D_{pool}$, and the weights $w_i$ is estimated using our acquisition function $a_{BB_{surv}}$. Here, $z$ is meant to be a parameter chosen by the user where a higher $z$ yields better performance however at a higher computational cost. We can take $z=3$ which is the lowest $z$ that guarantees the optimal $1-1/e$ bound \cite{Budgeted_MC}. This new greedy algorithm is very computationally expensive. The first part of the algorithm relies on finding the set out of all sets of size less than $z=3$ size (via brute force) that maximizes the weight, and assigning it to $H1$ which has a complexity of $O(n^2)$ if we denote $n = |D_{pool}|$. The rest of the algorithm involves for every possible initial set of size $z=3$ instances, greedily adding instances to this selected batch based off of the ratio of the weight to cost of an instance, which is of order $O(n^3)$. \citet{kuller_with_submodular_functions} generalize the work of \citet{Budgeted_MC} to work with submodular functions. In their work they argue a reduction of Algorithm~\ref{Alg: inefficient alg} to one where there is no need to account for all subsets of size $z$ in the algorithm. The proposed algorithm maintains a good approximation and is represented in our work as Algorithm 1.

\subsection{Choosing the Same Instance Multiple Times in One Query} 
\label{multiple probes}

Since in our setting the Oracle does not give full information. It opens up the question as to if you can choose the same instance multiple times in one query. For example, if the Oracle gives 5 years of new information for Alice, you might want to ask the Oracle to give you 10 years but for double the cost of course. This is not often seen in the real world, perhaps if you are working with a software that has all the instances but only gives incremental updates. Or maybe when dealing with ``label delays", covered in \cite{Incremental_AL}.

Fortunately, our formulation allows for generalization over this as well. Instead of considering how many times we should ask for the ``Alice" instance, we can see it such that asking for two increments of Alice, is a separate instance with its own cost as asking for one increment of Alice. Furthermore, the entire information that can be gained from for the one increment instance is contained in asking for the 2 increment instance. The question now becomes is the extra information worth the cost? Since the costs between these two instances are different, we again have a case of the budgeted maximum coverage problem, for which we show in the paper that there already exists an optimal approximation greedy algorithm for this setting.

A final note is that there is never an incentive to ask for both the Alice with 2 increments instance and the Alice with 1 increment instance, but the optimal greedy algorithm does not assume this and may take this action. Thus we can actually create an even better algorithm than greedy in this specific case by only allowing it to ever choose the maximum number of increments instance each time.

\section{Additional Data}
\label{sec:Additional Data}

We show here some additional results for the two settings.

\begin{figure}[H]
\label{selectme}
\begin{center}
\includegraphics[width=0.8\linewidth]{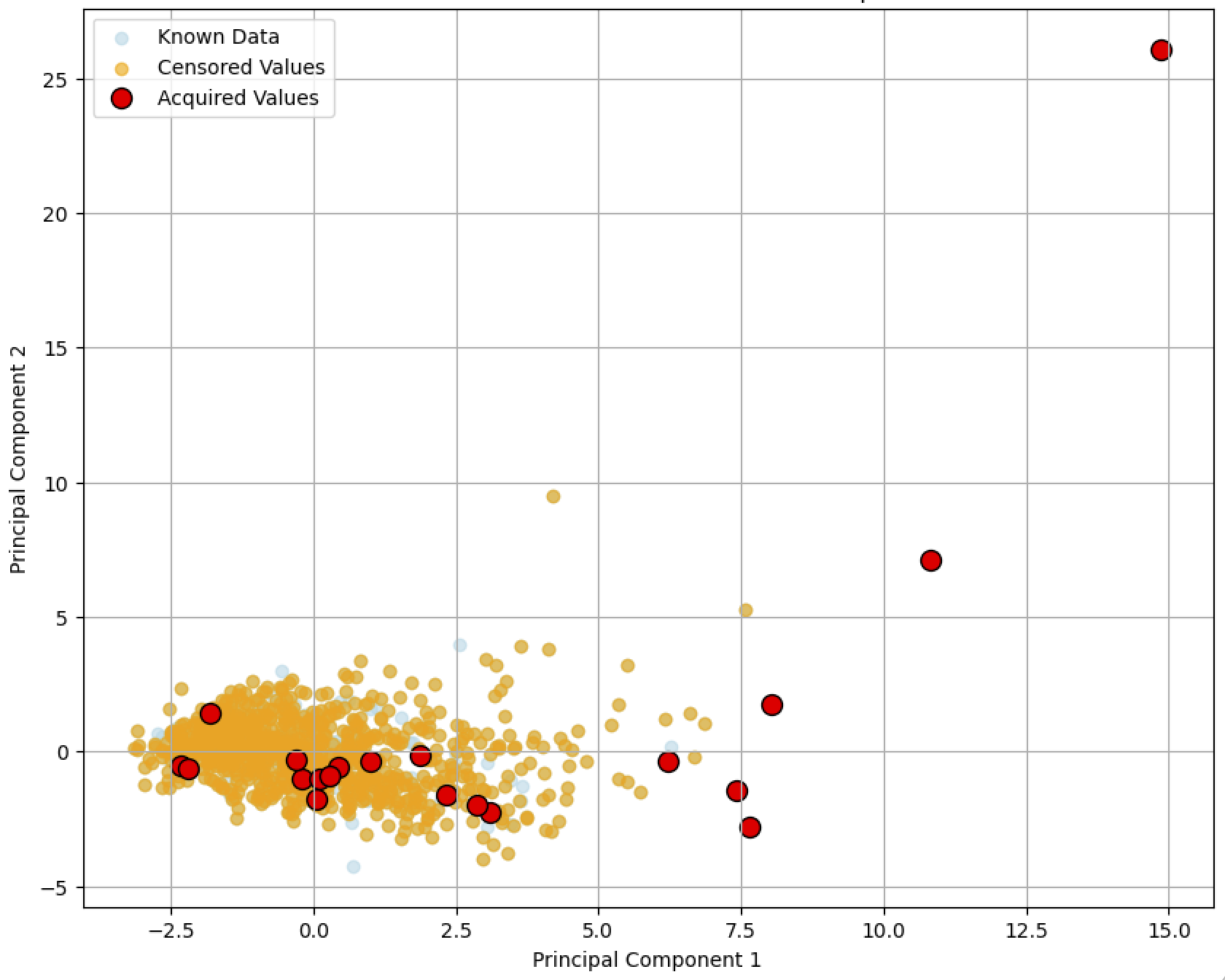}
\end{center}
\caption{PCA visualization of $BB_{surv}$ choosing 20 instances from NACD with probe depth = 10, and budget = 20. Uniform setting.}
\end{figure}

\begin{figure}[H]
\label{selectnotme}
\begin{center}
\includegraphics[width=0.8\linewidth]{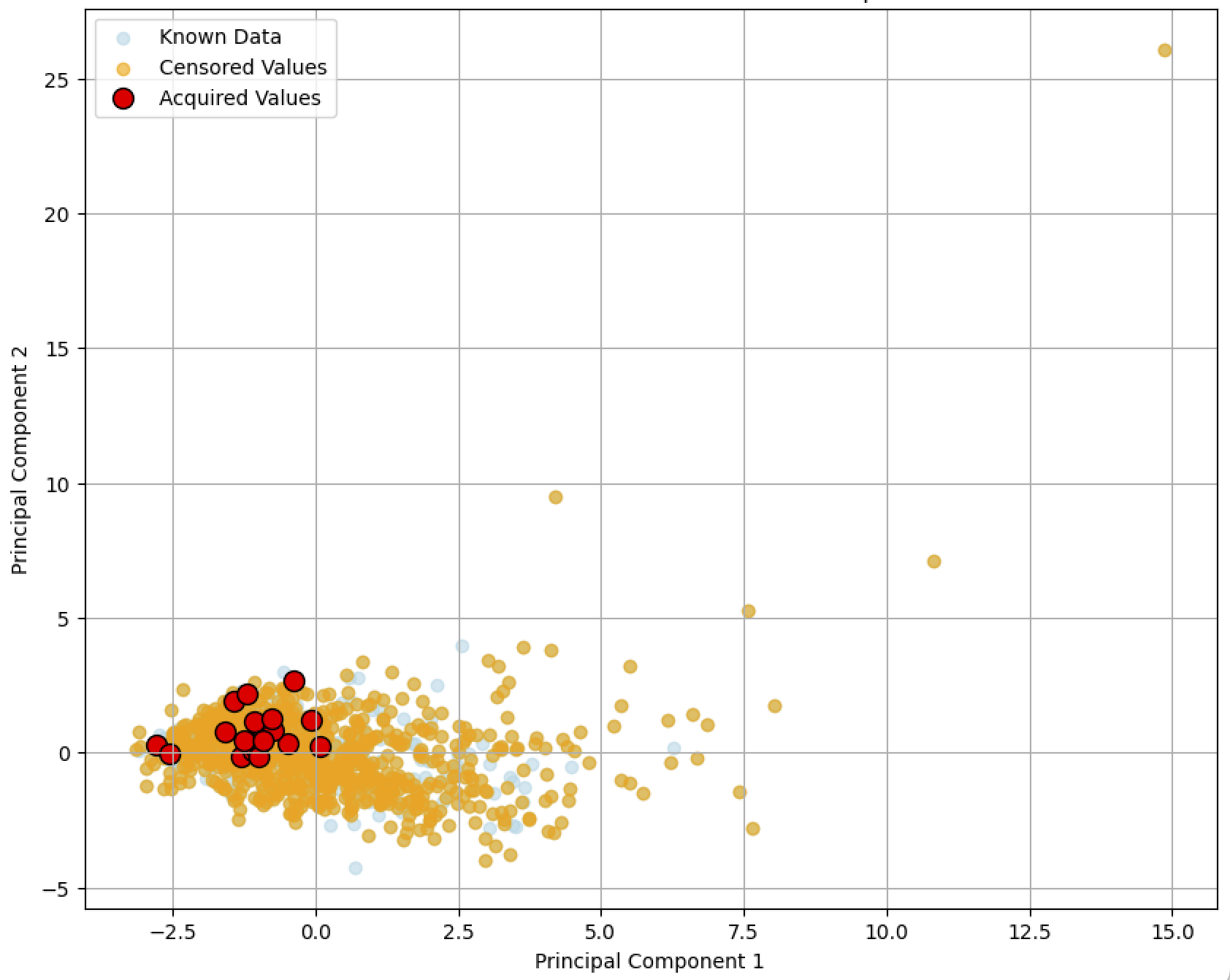}
\end{center}
\caption{PCA Visualization of mean closest to half choosing 20 instances from NACD with probe depth = 10, and budget = 20. Uniform setting.}
\end{figure}


\begin{figure}[t]
\includegraphics[width=\linewidth]{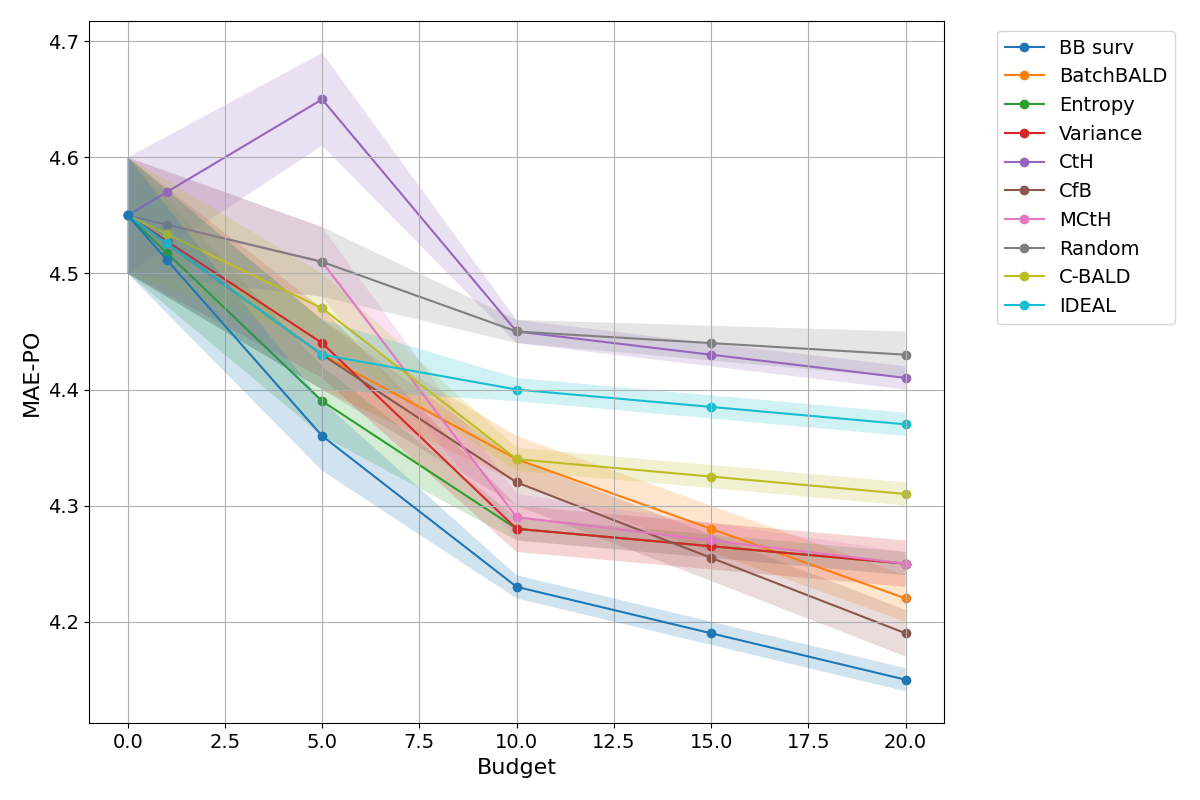}
\caption{Plot of MAE-PO evaluation across budgets of different acquisition functions with error bars as average of 40 predictions by the model. The plot uses the MIMIC dataset with 300 uncensored and 2700 censored points in the uniform setting with a probe depth of 5 years.}
\label{fig:mae_po_uniform_M_5}
\end{figure}

\begin{figure}[t]
\includegraphics[width=\linewidth]{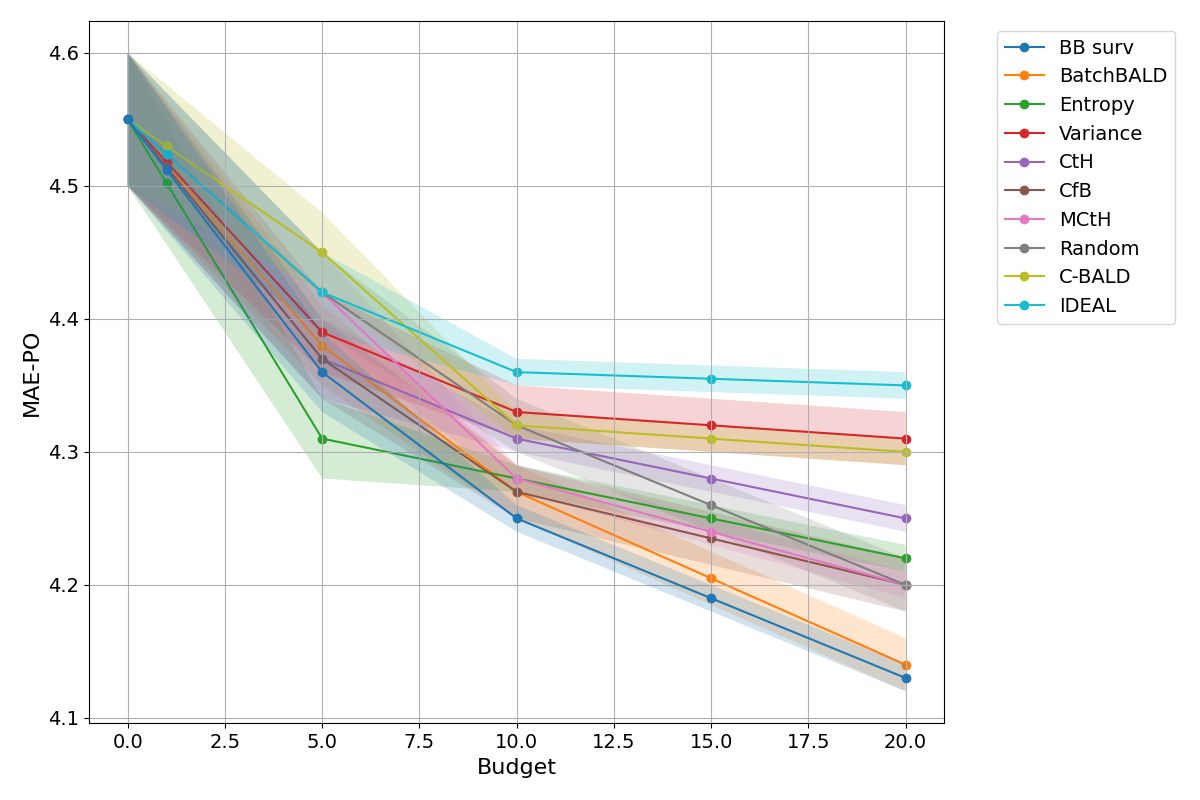}
\caption{Plot of MAE-PO evaluation across budgets of different acquisition functions with error bars as average of 40 predictions by the model. The plot uses the MIMIC dataset with 300 uncensored and 2700 censored points in the uniform setting with a probe depth of 10 years.}
\label{fig:mae_po_uniform_M_10}
\end{figure}

\begin{figure}[t]
\includegraphics[width=\linewidth]{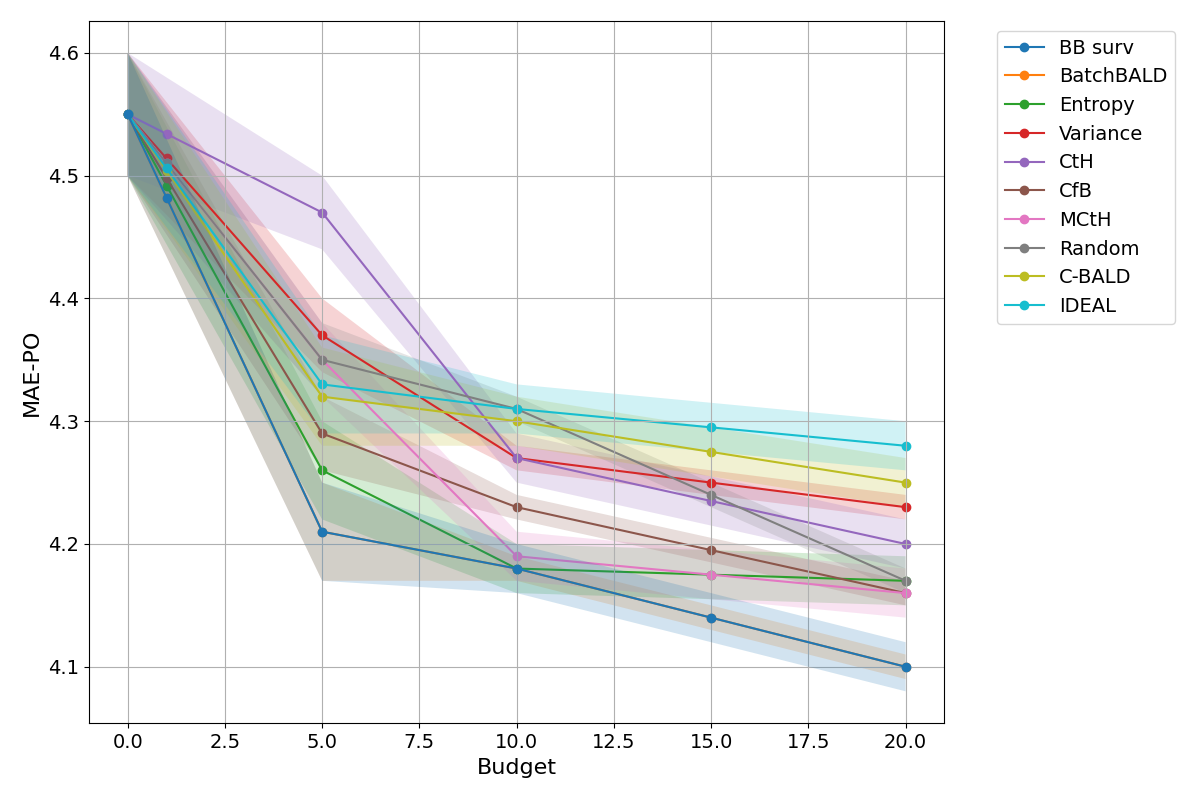}
\caption{Plot of MAE-PO evaluation across budgets of different acquisition functions with error bars as average of 40 predictions by the model. The plot uses the MIMIC dataset with 300 uncensored and 2700 censored points in the uniform setting with a probe depth of 100 years.}
\label{fig:mae_po_uniform_M_100}
\end{figure}

\begin{figure}[t]
\includegraphics[width=\linewidth]{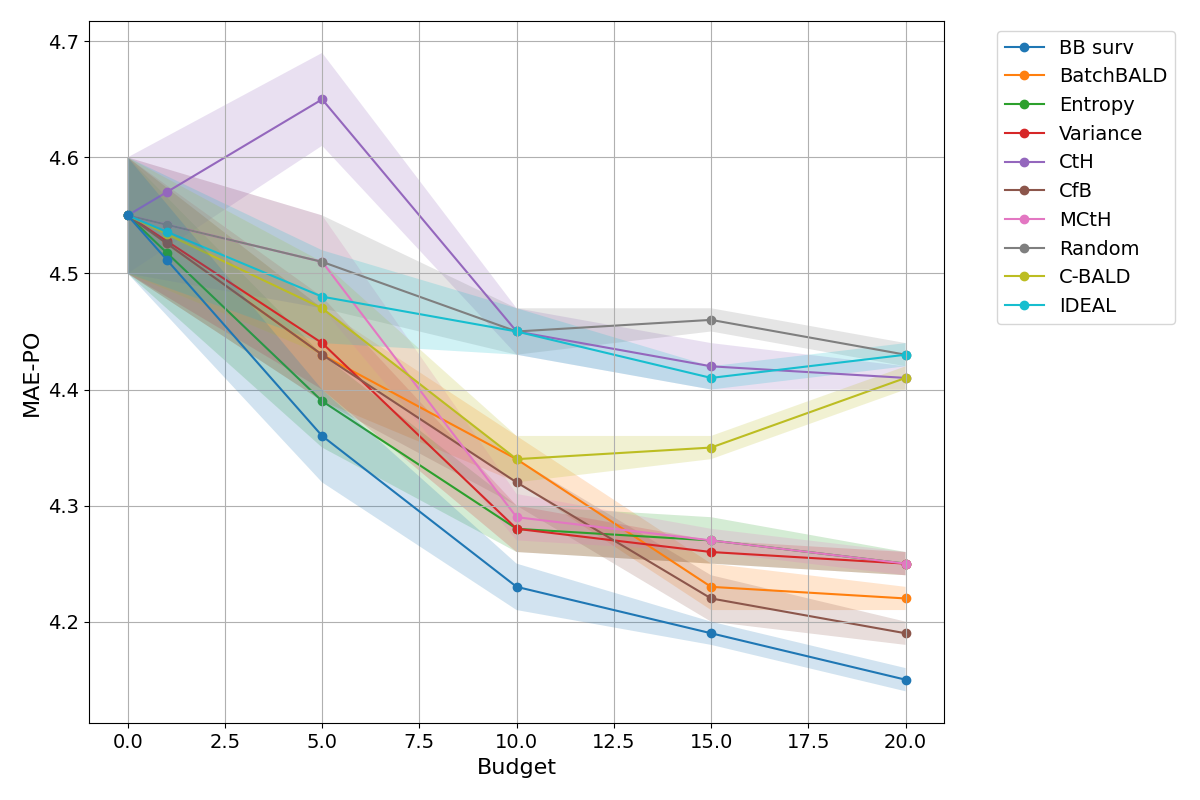}
\caption{Plot of MAE-PO evaluation across budgets of different acquisition functions with error bars as average of 40 predictions by the model. The plot uses the MIMIC dataset with 300 uncensored and 2700 censored points in the non-uniform setting with a probe depth of 5 years.}
\label{fig:mae_po_nonuniform_M_5}
\end{figure}

\begin{figure}[t]
\includegraphics[width=\linewidth]{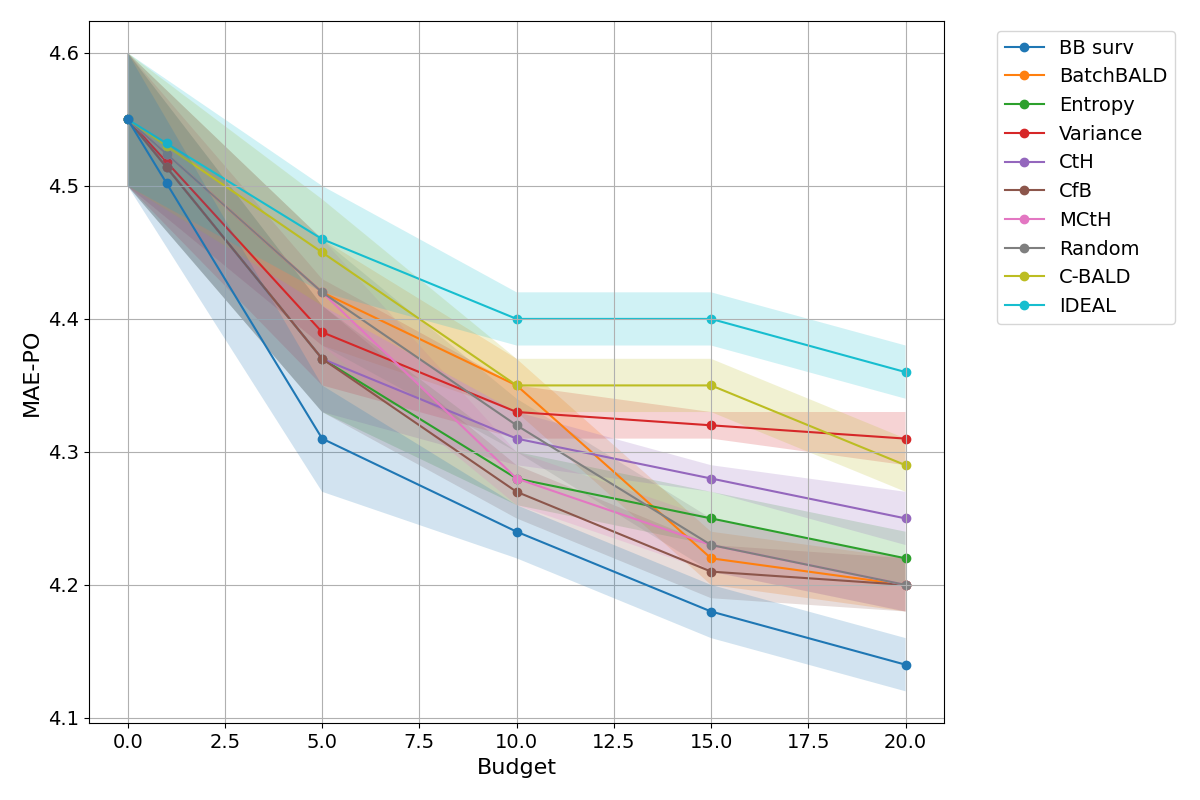}
\caption{Plot of MAE-PO evaluation across budgets of different acquisition functions with error bars as average of 40 predictions by the model. The plot uses the MIMIC dataset with 300 uncensored and 2700 censored points in the non-uniform setting with a probe depth of 10 years.}
\label{fig:mae_po_nonuniform_M_10}
\end{figure}

\begin{figure}[t]
\includegraphics[width=\linewidth]{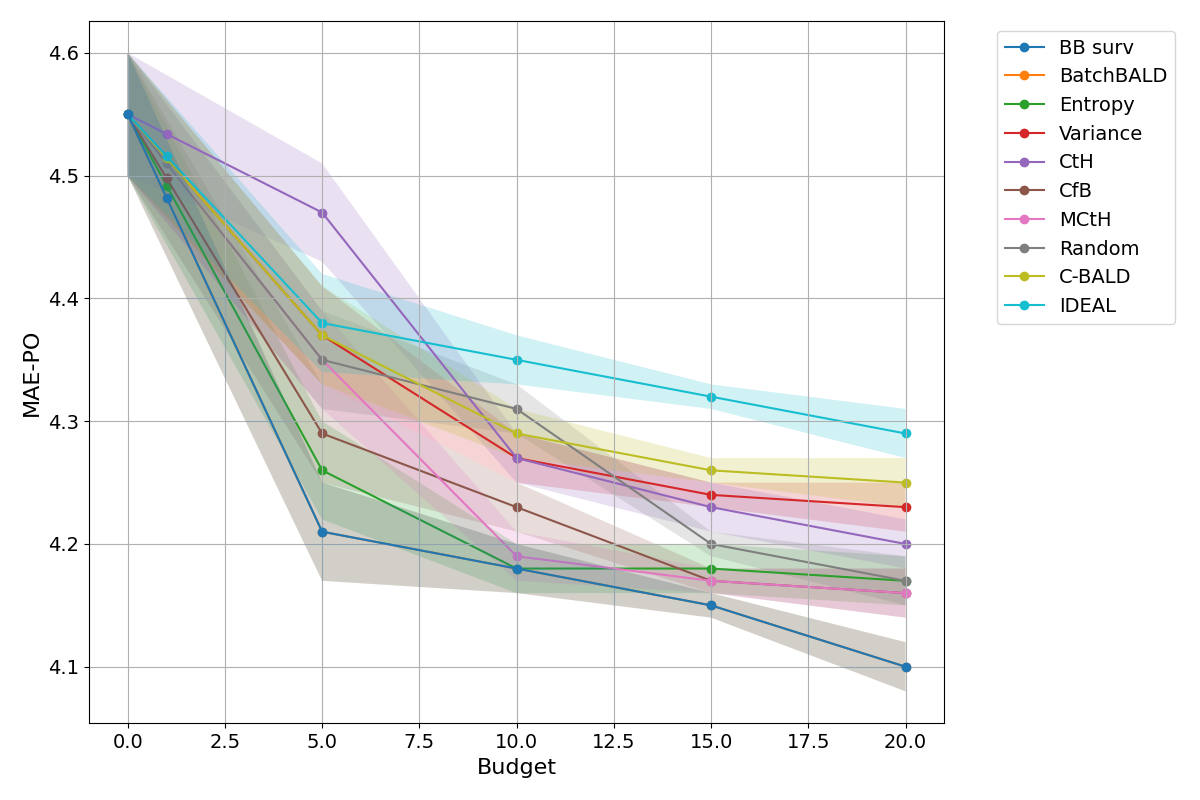}
\caption{Plot of MAE-PO evaluation across budgets of different acquisition functions with error bars as average of 40 predictions by the model. The plot uses the MIMIC dataset with 300 uncensored and 2700 censored points in the non-uniform setting with a probe depth of 100 years.}
\label{fig:mae_po_nonuniform_M_100}
\end{figure}

\begin{figure}[t]
\includegraphics[width=\linewidth]{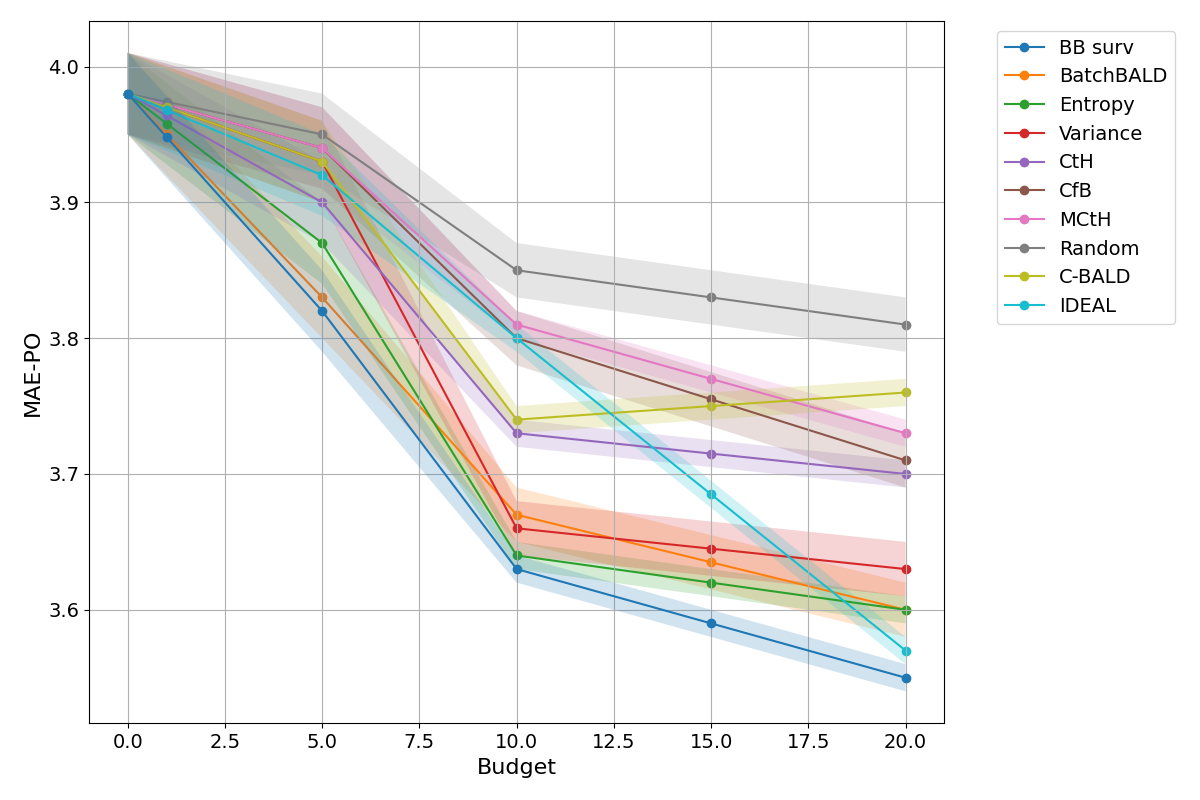}
\caption{Plot of MAE-PO evaluation across budgets of different acquisition functions with error bars as average of 40 predictions by the model. The plot uses the NACD dataset with 100 uncensored and 900 censored points in the uniform setting with a probe depth of 5 years.}
\label{fig:mae_po_uniform_N_5}
\end{figure}

\begin{figure}[t]
\includegraphics[width=\linewidth]{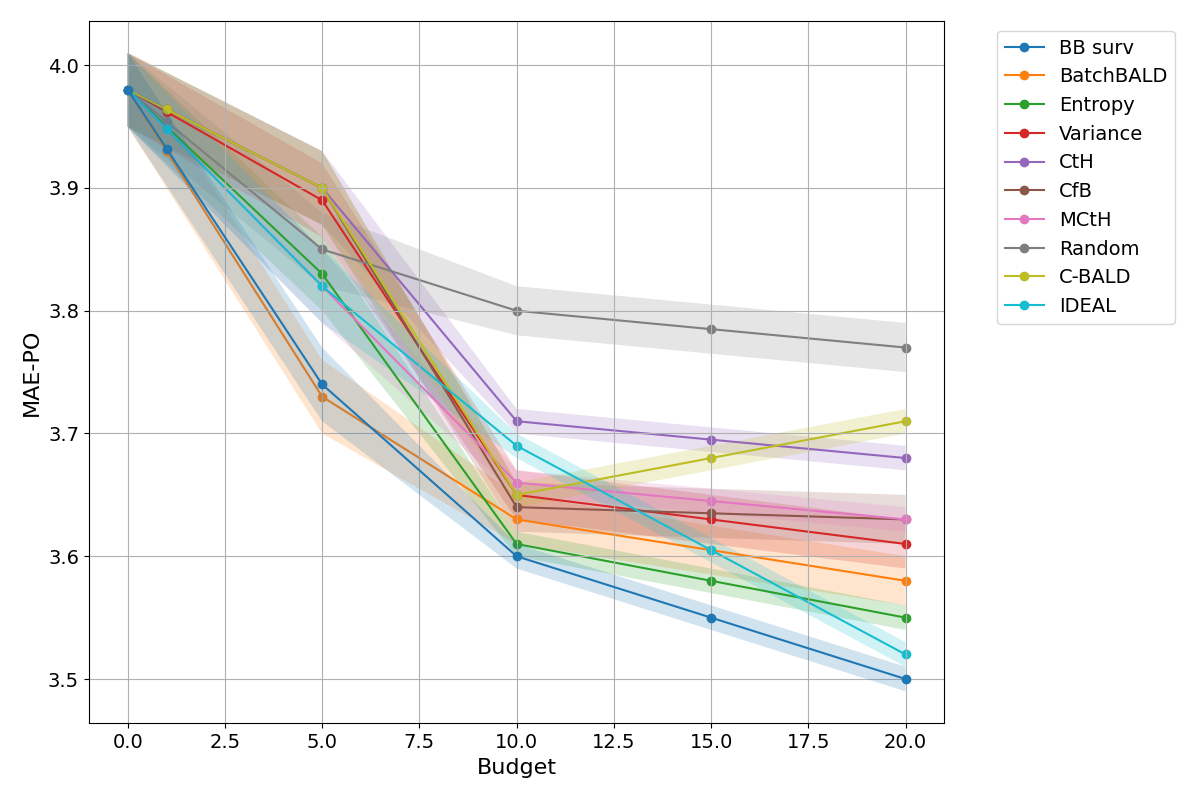}
\caption{Plot of MAE-PO evaluation across budgets of different acquisition functions with error bars as average of 40 predictions by the model. The plot uses the NACD dataset with 100 uncensored and 900 censored points in the uniform setting with a probe depth of 10 years.}
\label{fig:mae_po_uniform_N_10}
\end{figure}

\begin{figure}[t]
\includegraphics[width=\linewidth]{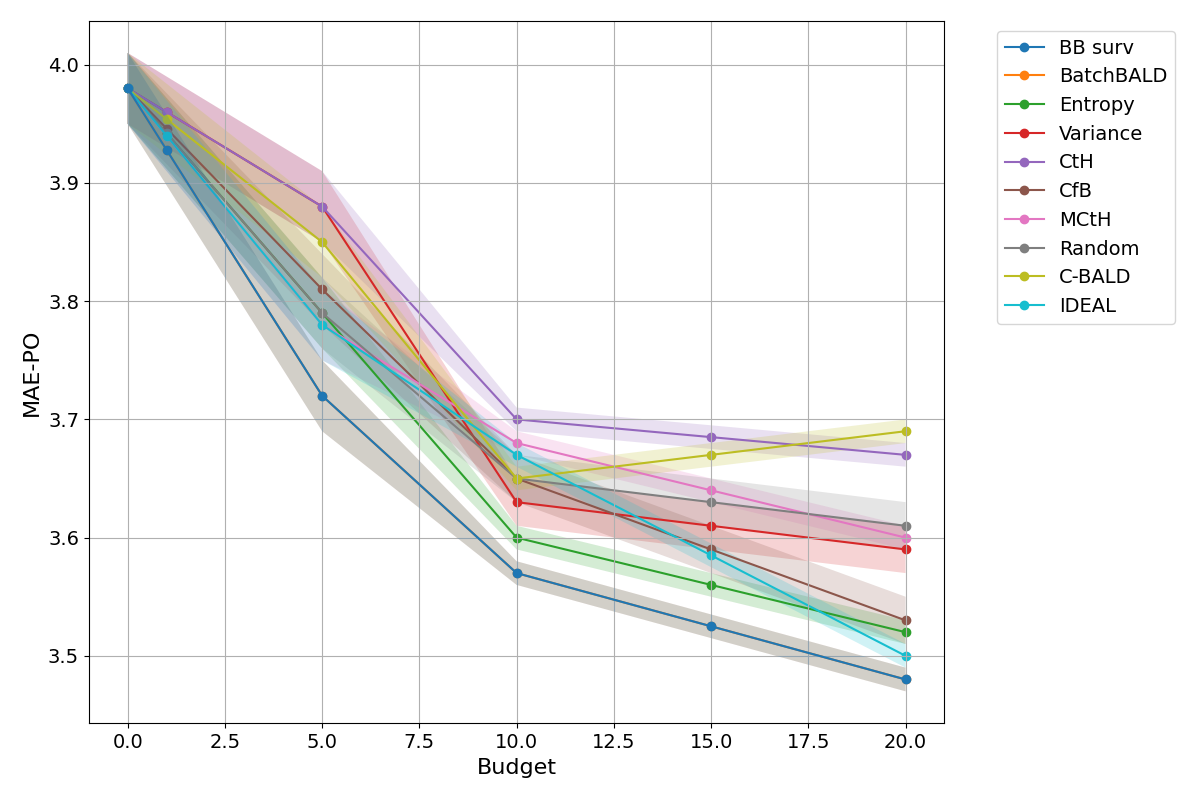}
\caption{Plot of MAE-PO evaluation across budgets of different acquisition functions with error bars as average of 40 predictions by the model. The plot uses the NACD dataset with 100 uncensored and 900 censored points in the uniform setting with a probe depth of 100 years.}
\label{fig:mae_po_uniform_N_100}
\end{figure}

\begin{figure}[t]
\includegraphics[width=\linewidth]{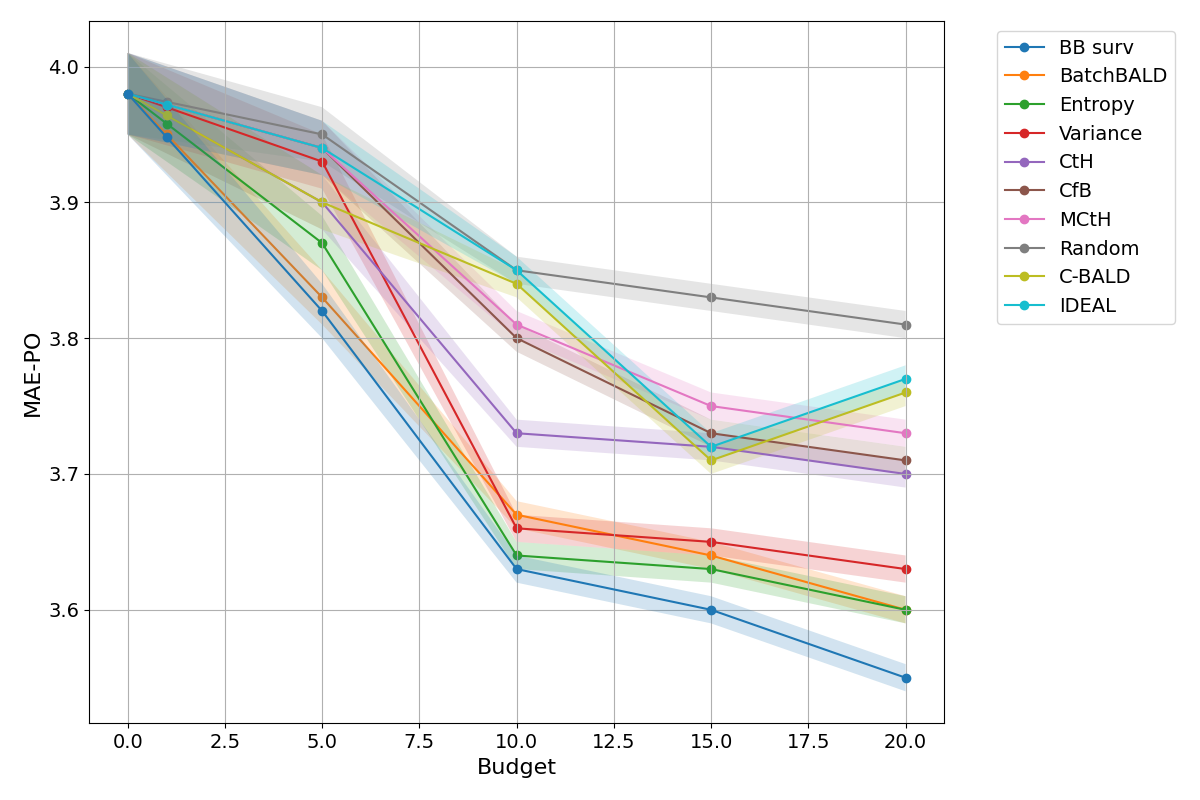}
\caption{Plot of MAE-PO evaluation across budgets of different acquisition functions with error bars as average of 40 predictions by the model. The plot uses the NACD dataset with 100 uncensored and 900 censored points in the non-uniform setting with a probe depth of 5 years.}
\label{fig:mae_po_nonuniform_N_5}
\end{figure}

\begin{figure}[t]
\includegraphics[width=\linewidth]{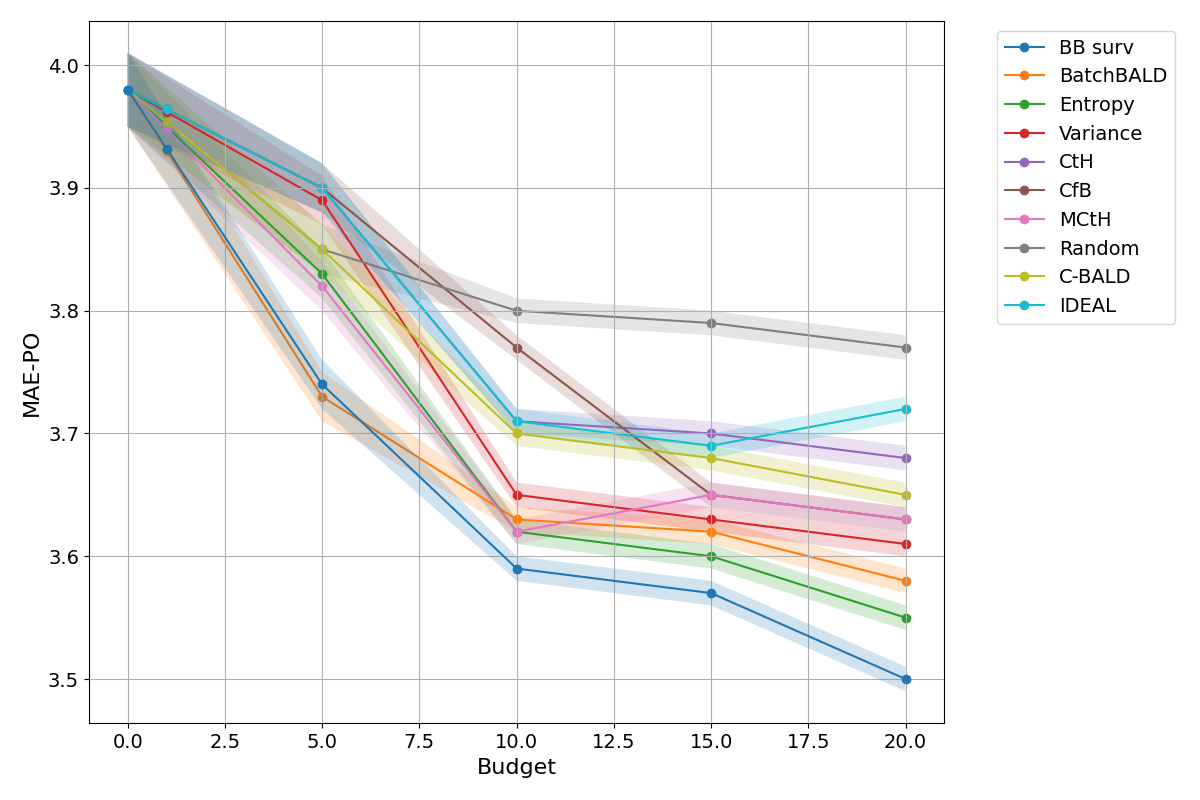}
\caption{Plot of MAE-PO evaluation across budgets of different acquisition functions with error bars as average of 40 predictions by the model. The plot uses the NACD dataset with 100 uncensored and 900 censored points in the non-uniform setting with a probe depth of 10 years.}
\label{fig:mae_po_nonuniform_N_10}
\end{figure}

\begin{figure}[t]
\includegraphics[width=\linewidth]{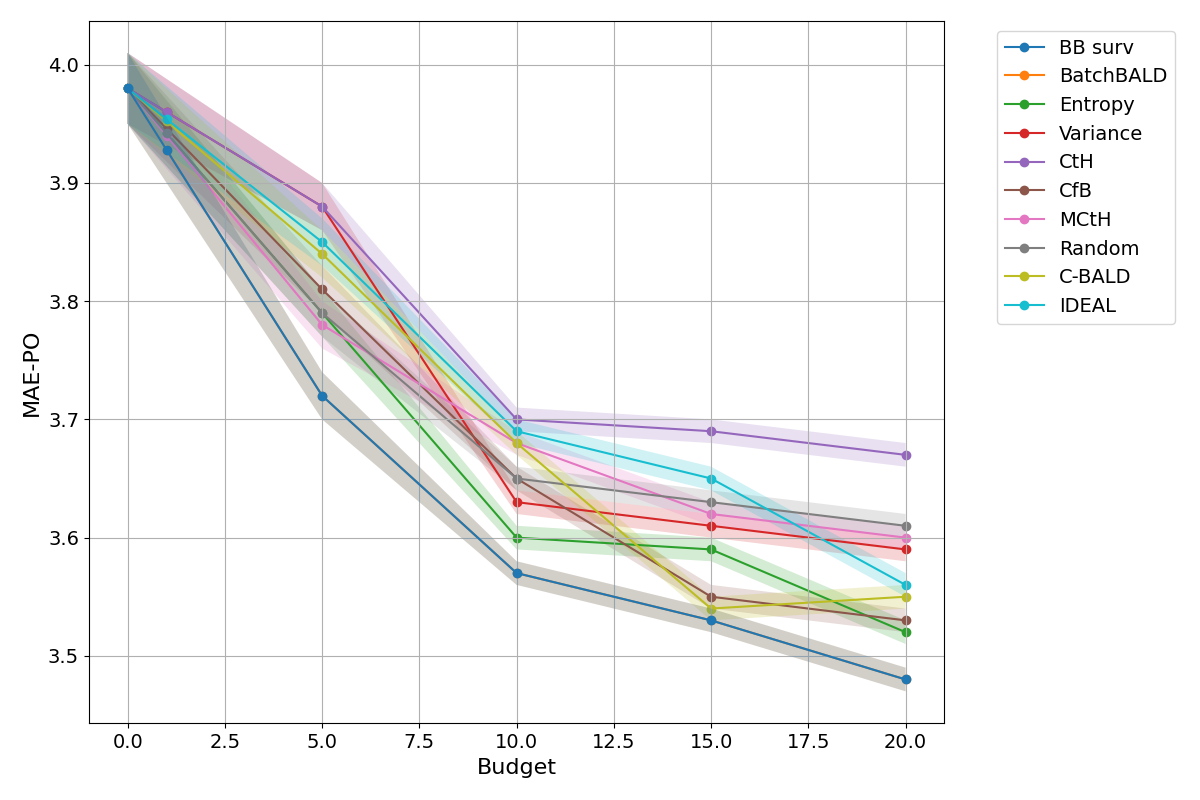}
\caption{Plot of MAE-PO evaluation across budgets of different acquisition functions with error bars as average of 40 predictions by the model. The plot uses the NACD dataset with 100 uncensored and 900 censored points in the non-uniform setting with a probe depth of 100 years.}
\label{fig:mae_po_nonuniform_N_100}
\end{figure}

\begin{figure}[t]
\includegraphics[width=\linewidth]{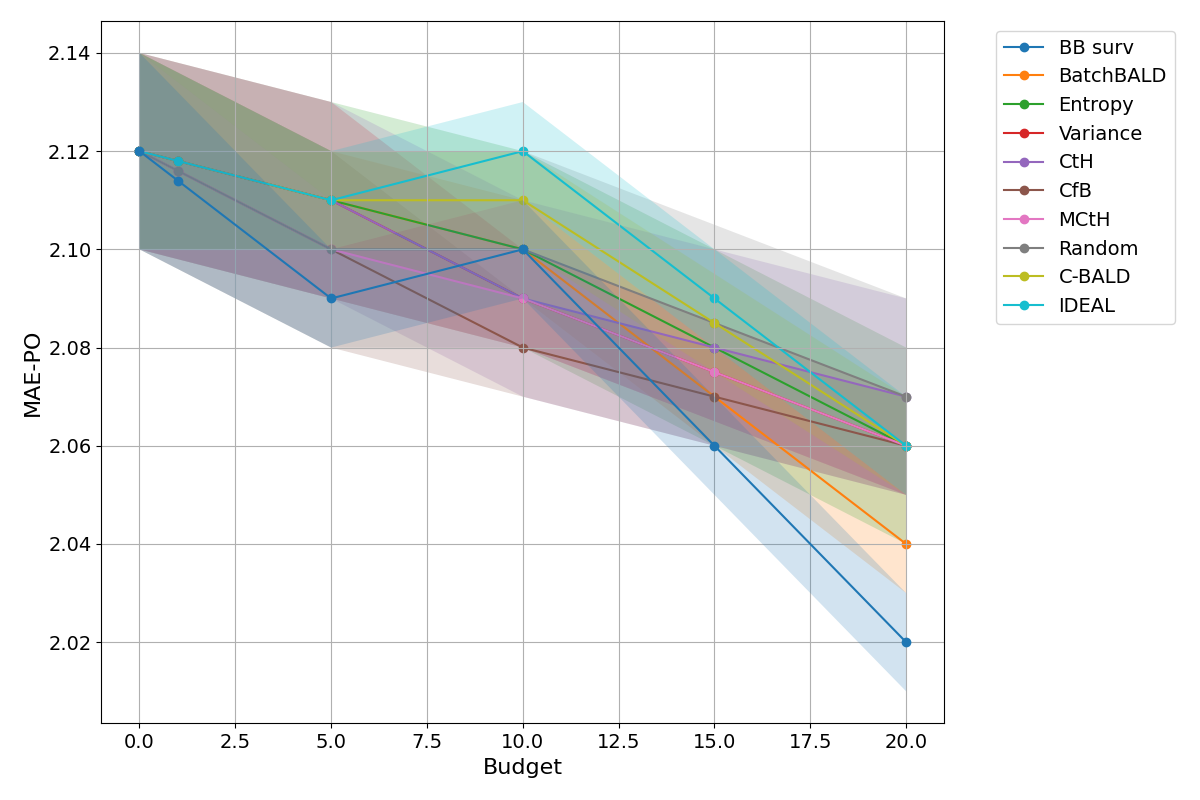}
\caption{Plot of MAE-PO evaluation across budgets of different acquisition functions with error bars as average of 40 predictions by the model. The plot uses the SUPPORT dataset with 100 uncensored and 900 censored points in the uniform setting with a probe depth of 5 years.}
\label{fig:mae_po_uniform_S_5}
\end{figure}

\begin{figure}[t]
\includegraphics[width=\linewidth]{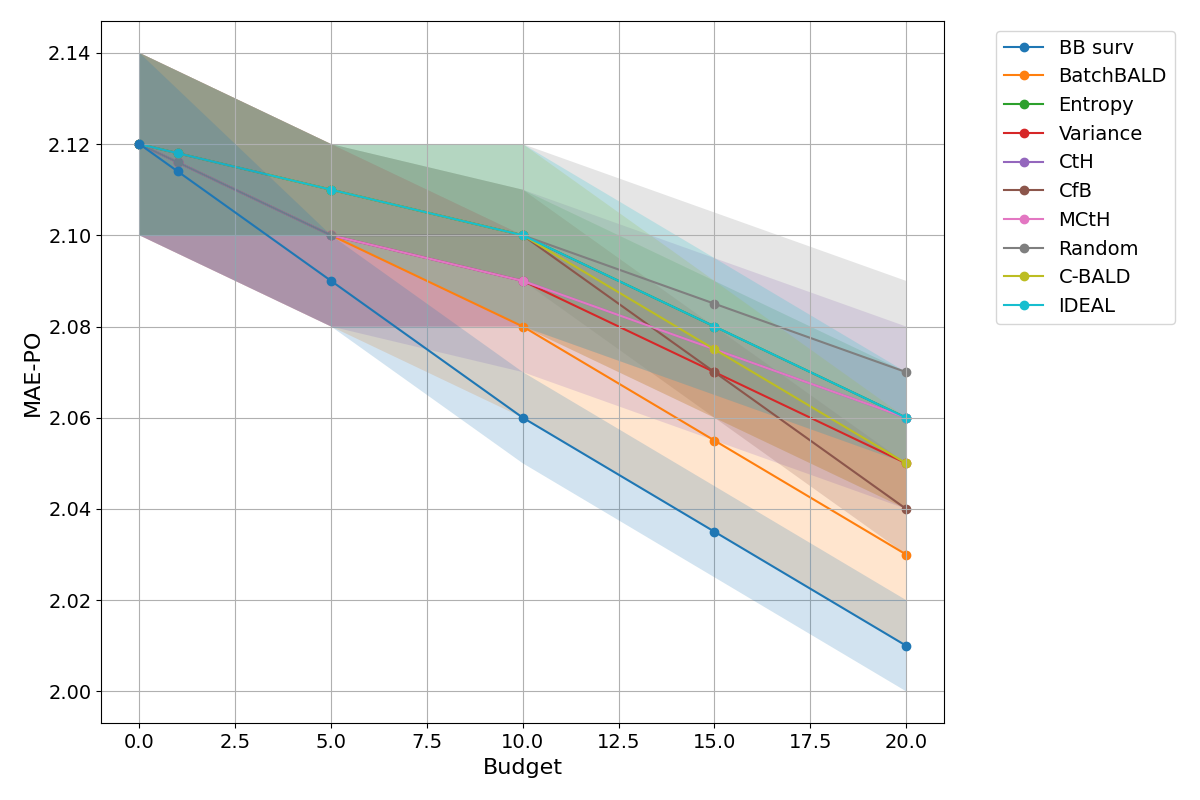}
\caption{Plot of MAE-PO evaluation across budgets of different acquisition functions with error bars as average of 40 predictions by the model. The plot uses the SUPPORT dataset with 100 uncensored and 900 censored points in the uniform setting with a probe depth of 10 years.}
\label{fig:mae_po_uniform_S_10}
\end{figure}

\begin{figure}[t]
\includegraphics[width=\linewidth]{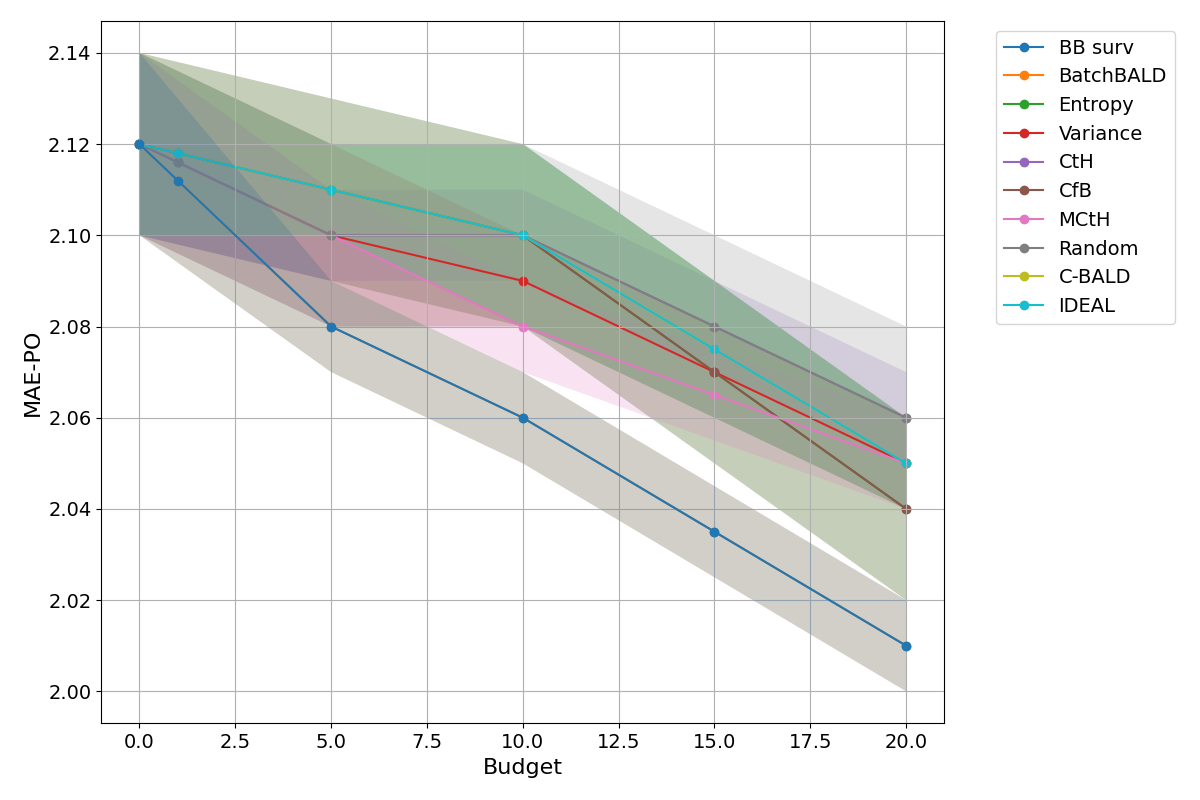}
\caption{Plot of MAE-PO evaluation across budgets of different acquisition functions with error bars as average of 40 predictions by the model. The plot uses the SUPPORT dataset with 100 uncensored and 900 censored points in the uniform setting with a probe depth of 100 years.}
\label{fig:mae_po_uniform_S_100}
\end{figure}

\begin{figure}[t]
\includegraphics[width=\linewidth]{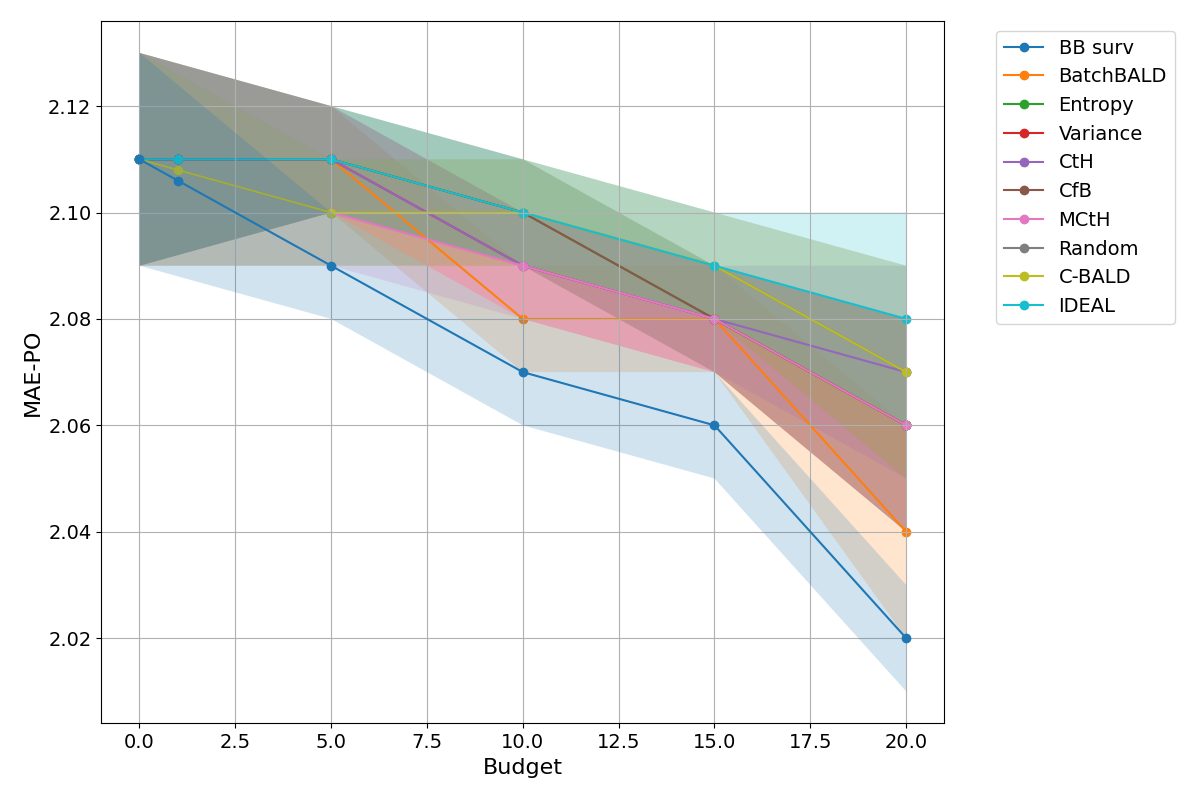}
\caption{Plot of MAE-PO evaluation across budgets of different acquisition functions with error bars as average of 40 predictions by the model. The plot uses the SUPPORT dataset with 100 uncensored and 900 censored points in the non-uniform setting with a probe depth of 5 years.}
\label{fig:mae_po_nonuniform_S_5}
\end{figure}

\begin{figure}[t]
\includegraphics[width=\linewidth]{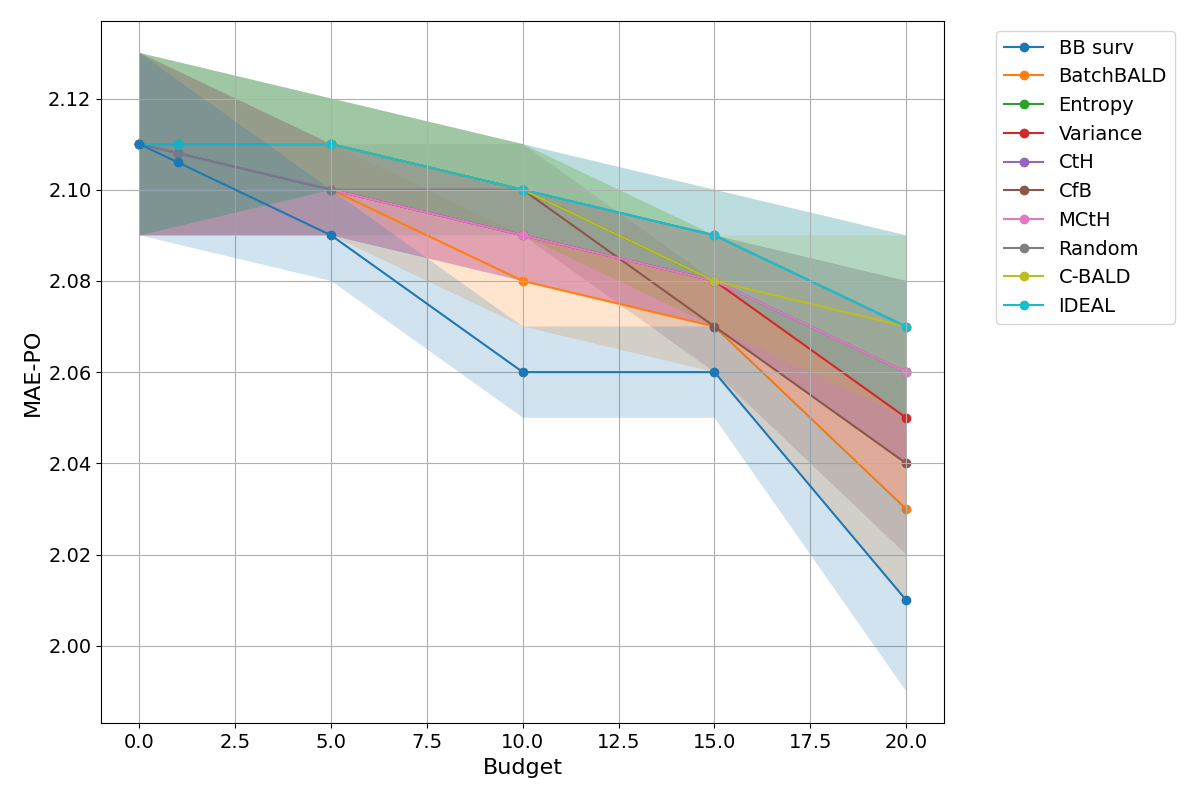}
\caption{Plot of MAE-PO evaluation across budgets of different acquisition functions with error bars as average of 40 predictions by the model. The plot uses the SUPPORT dataset with 100 uncensored and 900 censored points in the non-uniform setting with a probe depth of 10 years.}
\label{fig:mae_po_nonuniform_S_10}
\end{figure}

\begin{figure}[t]
\includegraphics[width=\linewidth]{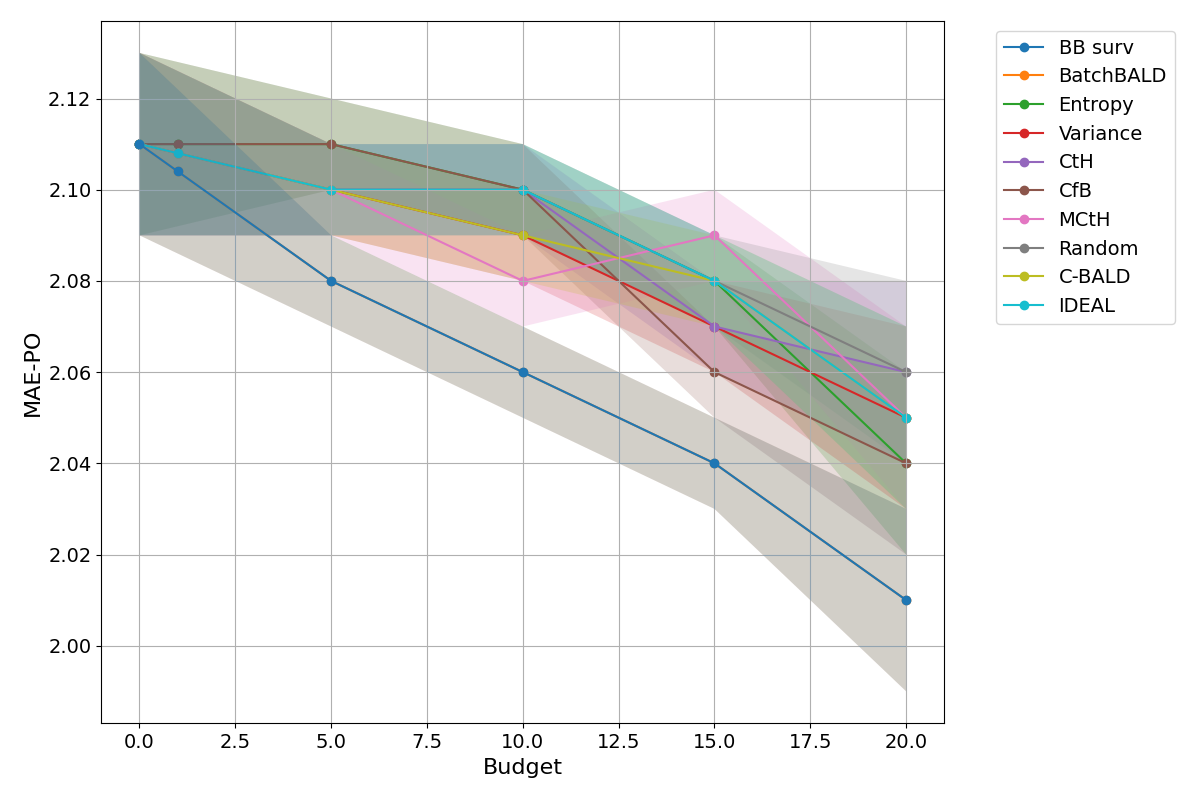}
\caption{Plot of MAE-PO evaluation across budgets of different acquisition functions with error bars as average of 40 predictions by the model. The plot uses the SUPPORT dataset with 100 uncensored and 900 censored points in the non-uniform setting with a probe depth of 100 years.}
\label{fig:mae_po_nonuniform_S_100}
\end{figure}


\begin{figure}[t]
\includegraphics[width=\linewidth]{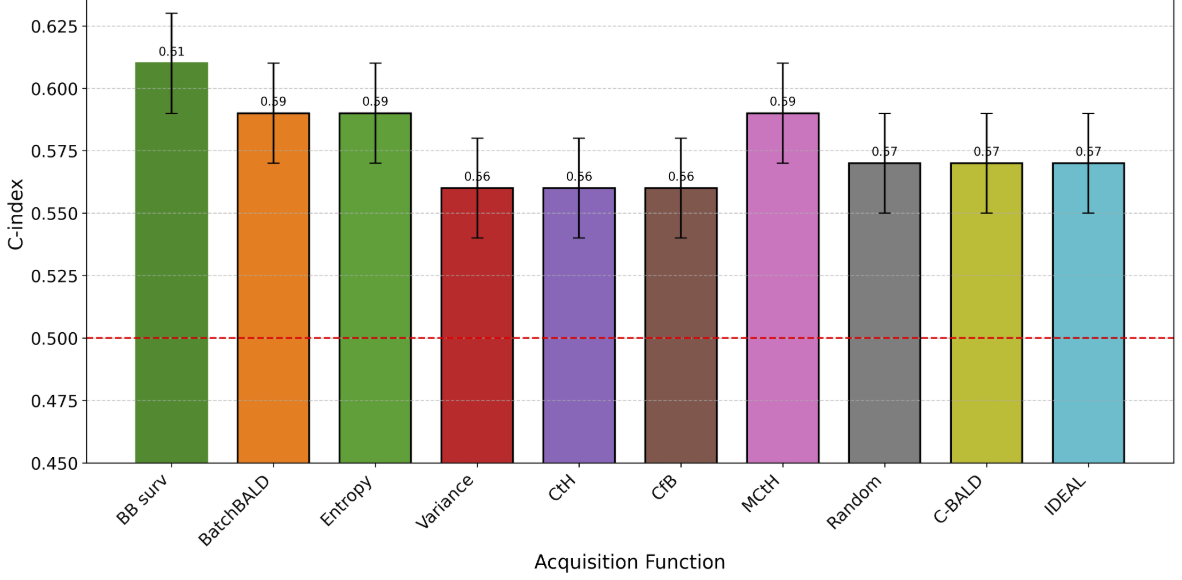}
\caption{Bar Plot of C-Index evaluation of different acquisition functions with error bars as average of 40 predictions by the model. The plot uses the NACD dataset with 100 uncensored and 900 censored points in the uniform setting with a probe depth of 5 years and a Budget equal to 20.}
\label{fig:mae_po_uniform_S_5}
\end{figure}


\begin{table}[t]
\centering
\caption{MAE-PO across datasets and probe depths, budget = 0. Uniform setting.}
\label{tab:uniform0}
\resizebox{\textwidth}{!}{%
\begin{tabular}{@{}lcccccccccc@{}}
\toprule
\textbf{Dataset} & \textbf{BB surv} & \textbf{BatchBALD} & \textbf{Entropy} & \textbf{Var} & \textbf{CtH} & \textbf{CfB} & \textbf{MCtM} & \textbf{Random} & \textbf{C-BALD} & \textbf{IDEAL}\\
\midrule
M +5 years   & 4.55 $\pm$ 0.05 & 4.55 $\pm$ 0.05 & 4.55 $\pm$ 0.05 & 4.55 $\pm$ 0.05 & 4.55 $\pm$ 0.05 & 4.55 $\pm$ 0.05 & 4.55 $\pm$ 0.05 & 4.55 $\pm$ 0.05 & 4.55 $\pm$ 0.05 & 4.55 $\pm$ 0.05\\
M +10 years  & 4.55 $\pm$ 0.05 & 4.55 $\pm$ 0.05 & 4.55 $\pm$ 0.05 & 4.55 $\pm$ 0.05 & 4.55 $\pm$ 0.05 & 4.55 $\pm$ 0.05 & 4.55 $\pm$ 0.05 & 4.55 $\pm$ 0.05 & 4.55 $\pm$ 0.05 & 4.55 $\pm$ 0.05\\
M +100 years & 4.55 $\pm$ 0.05 & 4.55 $\pm$ 0.05 & 4.55 $\pm$ 0.05 & 4.55 $\pm$ 0.05 & 4.55 $\pm$ 0.05 & 4.55 $\pm$ 0.05 & 4.55 $\pm$ 0.05 & 4.55 $\pm$ 0.05 & 4.55 $\pm$ 0.05 & 4.55 $\pm$ 0.05\\
N +5 years   & 3.98 $\pm$ 0.03 & 3.98 $\pm$ 0.03 & 3.98 $\pm$ 0.03 & 3.98 $\pm$ 0.03 & 3.98 $\pm$ 0.03 & 3.98 $\pm$ 0.03 & 3.98 $\pm$ 0.03 & 3.98 $\pm$ 0.03 & 3.98 $\pm$ 0.03 & 3.98 $\pm$ 0.03\\
N +10 years  & 3.98 $\pm$ 0.03 & 3.98 $\pm$ 0.03 & 3.98 $\pm$ 0.03 & 3.98 $\pm$ 0.03 & 3.98 $\pm$ 0.03 & 3.98 $\pm$ 0.03 & 3.98 $\pm$ 0.03 & 3.98 $\pm$ 0.03 & 3.98 $\pm$ 0.03 & 3.98 $\pm$ 0.03\\
N +100 years & 3.98 $\pm$ 0.03 & 3.98 $\pm$ 0.03 & 3.98 $\pm$ 0.03 & 3.98 $\pm$ 0.03 & 3.98 $\pm$ 0.03 & 3.98 $\pm$ 0.03 & 3.98 $\pm$ 0.03 & 3.98 $\pm$ 0.03 & 3.98 $\pm$ 0.03 & 3.98 $\pm$ 0.03\\
S +5 years   & 2.12 $\pm$ 0.02 & 2.12 $\pm$ 0.02 & 2.12 $\pm$ 0.02 & 2.12 $\pm$ 0.02 & 2.12 $\pm$ 0.02 & 2.12 $\pm$ 0.02 & 2.12 $\pm$ 0.02 & 2.12 $\pm$ 0.02 & 2.12 $\pm$ 0.02 & 2.12 $\pm$ 0.02\\
S +10 years  & 2.12 $\pm$ 0.02 & 2.12 $\pm$ 0.02 & 2.12 $\pm$ 0.02 & 2.12 $\pm$ 0.02 & 2.12 $\pm$ 0.02 & 2.12 $\pm$ 0.02 & 2.12 $\pm$ 0.02 & 2.12 $\pm$ 0.02 & 2.12 $\pm$ 0.02 & 2.12 $\pm$ 0.02\\
S +100 years & 2.12 $\pm$ 0.02 & 2.12 $\pm$ 0.02 & 2.12 $\pm$ 0.02 & 2.12 $\pm$ 0.02 & 2.12 $\pm$ 0.02 & 2.12 $\pm$ 0.02 & 2.12 $\pm$ 0.02 & 2.12 $\pm$ 0.02 & 2.12 $\pm$ 0.02 & 2.12 $\pm$ 0.02\\
\bottomrule
\end{tabular}}
\end{table}

\begin{table}[t]
\centering
\caption{MAE-PO across datasets and probe depths, budget = 1. Uniform setting.}
\label{tab:uniform1}
\resizebox{\textwidth}{!}{%
\begin{tabular}{@{}lcccccccccc@{}}
\toprule
\textbf{Dataset} & \textbf{BB surv} & \textbf{BatchBALD} & \textbf{Entropy} & \textbf{Variance} & \textbf{CtH} & \textbf{CfB} & \textbf{MCtM} & \textbf{Random} & \textbf{C-BALD} & \textbf{IDEAL}\\
\midrule
M +5 years   & 4.53 $\pm$ 0.03 & 4.53 $\pm$ 0.03 & 4.53 $\pm$ 0.03 & 4.55 $\pm$ 0.03 & 4.54 $\pm$ 0.04 & 4.54 $\pm$ 0.03 & 4.54 $\pm$ 0.03 & 4.54 $\pm$ 0.03 & 4.54 $\pm$ 0.03 & 4.54 $\pm$ 0.03\\
M +10 years  & 4.53 $\pm$ 0.03 & 4.53 $\pm$ 0.03 & 4.53 $\pm$ 0.03 & 4.53 $\pm$ 0.03 & 4.53 $\pm$ 0.03 & 4.52 $\pm$ 0.03 & 4.53 $\pm$ 0.03 & 4.54 $\pm$ 0.03 & 4.54 $\pm$ 0.03 & 4.54 $\pm$ 0.03\\
M +100 years & 4.52 $\pm$ 0.04 & 4.52 $\pm$ 0.04 & 4.52 $\pm$ 0.04 & 4.53 $\pm$ 0.03 & 4.52 $\pm$ 0.03 & 4.52 $\pm$ 0.03 & 4.52 $\pm$ 0.03 & 4.52 $\pm$ 0.03 & 4.53 $\pm$ 0.03 & 4.53 $\pm$ 0.03\\
N +5 years   & 3.94 $\pm$ 0.03 & 3.94 $\pm$ 0.03 & 3.97 $\pm$ 0.03 & 3.96 $\pm$ 0.03 & 3.97 $\pm$ 0.03 & 3.96 $\pm$ 0.03 & 3.96 $\pm$ 0.03 & 3.96 $\pm$ 0.03 & 3.95 $\pm$ 0.03 & 3.96 $\pm$ 0.03\\
N +10 years  & 3.94 $\pm$ 0.03 & 3.94 $\pm$ 0.03 & 3.97 $\pm$ 0.03 & 3.96 $\pm$ 0.03 & 3.95 $\pm$ 0.03 & 3.94 $\pm$ 0.03 & 3.94 $\pm$ 0.03 & 3.94 $\pm$ 0.03 & 3.95 $\pm$ 0.03 & 3.96 $\pm$ 0.03\\
N +100 years & 3.93 $\pm$ 0.03 & 3.93 $\pm$ 0.03 & 3.94 $\pm$ 0.03 & 3.94 $\pm$ 0.03 & 3.95 $\pm$ 0.03 & 3.93 $\pm$ 0.03 & 3.94 $\pm$ 0.03 & 3.94 $\pm$ 0.03 & 3.94 $\pm$ 0.03 & 3.95 $\pm$ 0.03\\
S +5 years   & 2.12 $\pm$ 0.01 & 2.13 $\pm$ 0.01 & 2.12 $\pm$ 0.02 & 2.13 $\pm$ 0.02 & 2.12 $\pm$ 0.02 & 2.12 $\pm$ 0.02 & 2.12 $\pm$ 0.01 & 2.12 $\pm$ 0.02 & 2.13 $\pm$ 0.02 & 2.12 $\pm$ 0.02\\
S +10 years  & 2.12 $\pm$ 0.01 & 2.12 $\pm$ 0.02 & 2.11 $\pm$ 0.01 & 2.12 $\pm$ 0.02 & 2.12 $\pm$ 0.02 & 2.12 $\pm$ 0.01 & 2.12 $\pm$ 0.02 & 2.12 $\pm$ 0.02 & 2.13 $\pm$ 0.02 & 2.12 $\pm$ 0.02\\
S +100 years & 2.12 $\pm$ 0.01 & 2.12 $\pm$ 0.02 & 2.11 $\pm$ 0.02 & 2.12 $\pm$ 0.02 & 2.12 $\pm$ 0.01 & 2.11 $\pm$ 0.02 & 2.12 $\pm$ 0.01 & 2.11 $\pm$ 0.02 & 2.13 $\pm$ 0.02 & 2.12 $\pm$ 0.02\\
\bottomrule
\end{tabular}}
\end{table}

\begin{table}[t]
\centering
\caption{MAE-PO across datasets and probe depths, budget = 5. Uniform setting.}
\label{tab:uniform5}
\resizebox{\textwidth}{!}{%
\begin{tabular}{@{}lcccccccccc@{}}
\toprule
\textbf{Dataset} & \textbf{BB surv} & \textbf{BatchBALD} & \textbf{Entropy} & \textbf{Variance} & \textbf{CtH} & \textbf{CfB} & \textbf{MCtM} & \textbf{Random} & \textbf{C-BALD} & \textbf{IDEAL}\\
\midrule
M +5 years   & \textbf{4.36 $\pm$ 0.03} & 4.43 $\pm$ 0.03 & 4.39 $\pm$ 0.03 & 4.44 $\pm$ 0.03 & 4.65 $\pm$ 0.04 & 4.43 $\pm$ 0.03 & 4.51 $\pm$ 0.03 & 4.51 $\pm$ 0.03 & 4.47 $\pm$ 0.03 & 4.43 $\pm$ 0.03\\
M +10 years  & 4.36 $\pm$ 0.03 & 4.38 $\pm$ 0.03 & \textbf{4.31 $\pm$ 0.03} & 4.39 $\pm$ 0.03 & 4.37 $\pm$ 0.03 & 4.37 $\pm$ 0.03 & 4.42 $\pm$ 0.03 & 4.42 $\pm$ 0.03 & 4.45 $\pm$ 0.03 & 4.42 $\pm$ 0.03\\
M +100 years & \textbf{4.21 $\pm$ 0.04} & \textbf{4.21 $\pm$ 0.04} & 4.26 $\pm$ 0.04 & 4.37 $\pm$ 0.03 & 4.47 $\pm$ 0.03 & 4.29 $\pm$ 0.03 & 4.35 $\pm$ 0.03 & 4.35 $\pm$ 0.03 & 4.32 $\pm$ 0.04 & 4.33 $\pm$ 0.04\\
N +5 years   & \textbf{3.82 $\pm$ 0.03} & \textbf{3.83 $\pm$ 0.03} & 3.87 $\pm$ 0.03 & 3.93 $\pm$ 0.03 & 3.90 $\pm$ 0.03 & 3.94 $\pm$ 0.03 & 3.94 $\pm$ 0.03 & 3.95 $\pm$ 0.03 & 3.93 $\pm$ 0.03 & 3.92 $\pm$ 0.03\\
N +10 years  & \textbf{3.74 $\pm$ 0.03} & \textbf{3.73 $\pm$ 0.03} & 3.83 $\pm$ 0.03 & 3.89 $\pm$ 0.03 & 3.90 $\pm$ 0.03 & 3.90 $\pm$ 0.03 & 3.82 $\pm$ 0.03 & 3.85 $\pm$ 0.03 & 3.90 $\pm$ 0.03 & 3.82 $\pm$ 0.03\\
N +100 years & \textbf{3.72 $\pm$ 0.03} & \textbf{3.72 $\pm$ 0.03} & 3.79 $\pm$ 0.03 & 3.88 $\pm$ 0.03 & 3.88 $\pm$ 0.03 & 3.81 $\pm$ 0.03 & 3.78 $\pm$ 0.03 & 3.79 $\pm$ 0.03 & 3.85 $\pm$ 0.03 & 3.78 $\pm$ 0.03\\
S +5 years   & \textbf{2.09 $\pm$ 0.01} & 2.11 $\pm$ 0.01 & 2.11 $\pm$ 0.02 & 2.11 $\pm$ 0.02 & 2.11 $\pm$ 0.02 & \textbf{2.10 $\pm$ 0.02} & \textbf{2.10 $\pm$ 0.01} & \textbf{2.10 $\pm$ 0.02} & 2.11 $\pm$ 0.01 & 2.11 $\pm$ 0.01\\
S +10 years  & \textbf{2.09 $\pm$ 0.01} & 2.10 $\pm$ 0.02 & 2.11 $\pm$ 0.01 & 2.10 $\pm$ 0.02 & 2.10 $\pm$ 0.02 & 2.11 $\pm$ 0.01 & 2.10 $\pm$ 0.02 & 2.10 $\pm$ 0.02 & 2.11 $\pm$ 0.01 & 2.11 $\pm$ 0.01\\
S +100 years & \textbf{2.08 $\pm$ 0.01} & \textbf{2.08 $\pm$ 0.01} & 2.11 $\pm$ 0.02 & 2.10 $\pm$ 0.02 & 2.10 $\pm$ 0.01 & 2.11 $\pm$ 0.02 & 2.10 $\pm$ 0.01 & 2.10 $\pm$ 0.02 & 2.11 $\pm$ 0.01 & 2.11 $\pm$ 0.01\\
\bottomrule
\end{tabular}}
\end{table}

\begin{table}[t]
\centering
\caption{MAE-PO across datasets and probe depths, budget = 10. Uniform setting.}
\label{tab:uniform10}
\resizebox{\textwidth}{!}{%
\begin{tabular}{@{}lcccccccccc@{}}
\toprule
\textbf{Dataset} & \textbf{BB surv} & \textbf{BatchBALD} & \textbf{Entropy} & \textbf{Var} & \textbf{CtH} & \textbf{CfB} & \textbf{MCtM} & \textbf{Random} & \textbf{C-BALD} & \textbf{IDEAL}\\
\midrule
M +5y   & \textbf{4.23 $\pm$ 0.01} & 4.34 $\pm$ 0.02 & 4.28 $\pm$ 0.01 & 4.28 $\pm$ 0.02 & 4.45 $\pm$ 0.01 & 4.32 $\pm$ 0.02 & 4.29 $\pm$ 0.02 & 4.45 $\pm$ 0.01 & 4.34 $\pm$ 0.01 & 4.40 $\pm$ 0.01\\
M +10y  & \textbf{4.25 $\pm$ 0.01} & \textbf{4.27 $\pm$ 0.02} & 4.28 $\pm$ 0.01 & 4.33 $\pm$ 0.02 & 4.31 $\pm$ 0.01 & 4.27 $\pm$ 0.02 & 4.28 $\pm$ 0.01 & 4.32 $\pm$ 0.02 & 4.32 $\pm$ 0.01 & 4.36 $\pm$ 0.01\\
M +100y & \textbf{4.18 $\pm$ 0.02} & \textbf{4.18 $\pm$ 0.01} & \textbf{4.18 $\pm$ 0.02} & 4.27 $\pm$ 0.01 & 4.27 $\pm$ 0.02 & 4.23 $\pm$ 0.01 & 4.19 $\pm$ 0.02 & 4.31 $\pm$ 0.01 & 4.30 $\pm$ 0.02 & 4.31 $\pm$ 0.02\\
N +5y   & \textbf{3.63 $\pm$ 0.01} & 3.67 $\pm$ 0.02 & 3.64 $\pm$ 0.01 & 3.66 $\pm$ 0.02 & 3.73 $\pm$ 0.01 & 3.80 $\pm$ 0.02 & 3.81 $\pm$ 0.01 & 3.85 $\pm$ 0.02 & 3.74 $\pm$ 0.01 & 3.80 $\pm$ 0.01\\
N +10y  & \textbf{3.60 $\pm$ 0.01} & 3.63 $\pm$ 0.02 & \textbf{3.61 $\pm$ 0.01} & 3.65 $\pm$ 0.02 & 3.71 $\pm$ 0.01 & 3.64 $\pm$ 0.02 & 3.66 $\pm$ 0.01 & 3.80 $\pm$ 0.02 & 3.65 $\pm$ 0.01 & 3.69 $\pm$ 0.01\\
N +100y & \textbf{3.57 $\pm$ 0.01} & \textbf{3.57 $\pm$ 0.01} & 3.60 $\pm$ 0.01 & 3.63 $\pm$ 0.02 & 3.70 $\pm$ 0.01 & 3.65 $\pm$ 0.02 & 3.68 $\pm$ 0.01 & 3.65 $\pm$ 0.02 & 3.65 $\pm$ 0.01 & 3.67 $\pm$ 0.01\\
S +5y   & 2.10 $\pm$ 0.01 & 2.10 $\pm$ 0.01 & 2.10 $\pm$ 0.02 & 2.09 $\pm$ 0.01 & 2.09 $\pm$ 0.02 & \textbf{2.08 $\pm$ 0.01} & 2.09 $\pm$ 0.01 & 2.10 $\pm$ 0.02 & 2.11 $\pm$ 0.01 & 2.12 $\pm$ 0.01\\
S +10y  & \textbf{2.06 $\pm$ 0.01} & 2.08 $\pm$ 0.02 & 2.10 $\pm$ 0.01 & 2.09 $\pm$ 0.01 & 2.09 $\pm$ 0.02 & 2.10 $\pm$ 0.01 & 2.09 $\pm$ 0.01 & 2.10 $\pm$ 0.02 & 2.10 $\pm$ 0.02 & 2.10 $\pm$ 0.02\\
S +100y & \textbf{2.06 $\pm$ 0.01} & \textbf{2.06 $\pm$ 0.01} & 2.10 $\pm$ 0.02 & 2.09 $\pm$ 0.01 & 2.10 $\pm$ 0.01 & 2.10 $\pm$ 0.02 & 2.08 $\pm$ 0.01 & 2.10 $\pm$ 0.02 & 2.10 $\pm$ 0.02 & 2.10 $\pm$ 0.02\\
\bottomrule
\end{tabular}}
\end{table}

\begin{table}[t]
\centering
\caption{MAE-PO across datasets and probe depths, budget = 15. Uniform setting.}
\label{tab:uniform15}
\resizebox{\textwidth}{!}{%
\begin{tabular}{@{}lcccccccccc@{}}
\toprule
\textbf{Dataset} & \textbf{BB surv} & \textbf{BatchBALD} & \textbf{Entropy} & \textbf{Var} & \textbf{CtH} & \textbf{CfB} & \textbf{MCtM} & \textbf{Random} & \textbf{C-BALD} & \textbf{IDEAL}\\
\midrule
M +5y   & \textbf{4.19 $\pm$ 0.01} & 4.23 $\pm$ 0.02 & 4.27 $\pm$ 0.01 & 4.26 $\pm$ 0.02 & 4.42 $\pm$ 0.01 & 4.22 $\pm$ 0.02 & 4.27 $\pm$ 0.02 & 4.46 $\pm$ 0.01 & 4.20 $\pm$ 0.01 & 4.21 $\pm$ 0.01\\
M +10y  & \textbf{4.18 $\pm$ 0.01} & \textbf{4.20 $\pm$ 0.02} & 4.25 $\pm$ 0.01 & 4.32 $\pm$ 0.02 & 4.28 $\pm$ 0.01 & 4.21 $\pm$ 0.02 & 4.23 $\pm$ 0.01 & 4.23 $\pm$ 0.02 & 4.19 $\pm$ 0.01 & 4.20 $\pm$ 0.01\\
M +100y & \textbf{4.15 $\pm$ 0.02} & \textbf{4.15 $\pm$ 0.01} & 4.18 $\pm$ 0.02 & 4.24 $\pm$ 0.01 & 4.23 $\pm$ 0.02 & \textbf{4.17 $\pm$ 0.01} & \textbf{4.17 $\pm$ 0.02} & 4.20 $\pm$ 0.01 & 4.16 $\pm$ 0.02 & 4.17 $\pm$ 0.02\\
N +5y   & \textbf{3.60 $\pm$ 0.01} & 3.64 $\pm$ 0.02 & 3.63 $\pm$ 0.01 & 3.65 $\pm$ 0.02 & 3.72 $\pm$ 0.01 & 3.73 $\pm$ 0.02 & 3.75 $\pm$ 0.01 & 3.83 $\pm$ 0.02 & 3.61 $\pm$ 0.01 & 3.62 $\pm$ 0.01\\
N +10y  & \textbf{3.57 $\pm$ 0.01} & 3.62 $\pm$ 0.02 & 3.60 $\pm$ 0.01 & 3.63 $\pm$ 0.02 & 3.70 $\pm$ 0.01 & 3.65 $\pm$ 0.02 & 3.65 $\pm$ 0.01 & 3.79 $\pm$ 0.02 & 3.58 $\pm$ 0.01 & 3.59 $\pm$ 0.01\\
N +100y & \textbf{3.53 $\pm$ 0.01} & \textbf{3.53 $\pm$ 0.01} & 3.59 $\pm$ 0.01 & 3.61 $\pm$ 0.02 & 3.69 $\pm$ 0.01 & \textbf{3.55 $\pm$ 0.02} & 3.62 $\pm$ 0.01 & 3.63 $\pm$ 0.02 & 3.54 $\pm$ 0.01 & 3.55 $\pm$ 0.01\\
S +5y   & \textbf{2.06 $\pm$ 0.01} & 2.08 $\pm$ 0.01 & 2.08 $\pm$ 0.02 & 2.08 $\pm$ 0.01 & 2.08 $\pm$ 0.02 & 2.08 $\pm$ 0.01 & 2.08 $\pm$ 0.01 & 2.09 $\pm$ 0.02 & 2.07 $\pm$ 0.01 & 2.08 $\pm$ 0.01\\
S +10y  & \textbf{2.06 $\pm$ 0.01} & \textbf{2.07 $\pm$ 0.02} & 2.08 $\pm$ 0.01 & 2.08 $\pm$ 0.01 & 2.08 $\pm$ 0.02 & \textbf{2.07 $\pm$ 0.01} & 2.08 $\pm$ 0.01 & 2.09 $\pm$ 0.02 & 2.07 $\pm$ 0.01 & 2.08 $\pm$ 0.01\\
S +100y & \textbf{2.04 $\pm$ 0.01} & \textbf{2.04 $\pm$ 0.01} & 2.08 $\pm$ 0.02 & 2.07 $\pm$ 0.01 & 2.07 $\pm$ 0.01 & 2.06 $\pm$ 0.02 & 2.09 $\pm$ 0.01 & 2.08 $\pm$ 0.02 & 2.05 $\pm$ 0.01 & 2.06 $\pm$ 0.01\\
\bottomrule
\end{tabular}}
\end{table}

\begin{table}[t]
\centering
\caption{MAE-PO across datasets and probe depths, budget = 20. Uniform setting.}
\label{tab:uniform20}
\resizebox{\textwidth}{!}{%
\begin{tabular}{@{}lcccccccccc@{}}
\toprule
\textbf{Dataset} & \textbf{BB surv} & \textbf{BatchBALD} & \textbf{Entropy} & \textbf{Var} & \textbf{CtH} & \textbf{CfB} & \textbf{MCtM} & \textbf{Random} & \textbf{C-BALD} & \textbf{IDEAL}\\
\midrule
M +5y   & \textbf{4.15 $\pm$ 0.01} & 4.22 $\pm$ 0.02 & 4.25 $\pm$ 0.01 & 4.25 $\pm$ 0.02 & 4.41 $\pm$ 0.01 & 4.19 $\pm$ 0.02 & 4.25 $\pm$ 0.01 & 4.43 $\pm$ 0.02 & 4.31 $\pm$ 0.01 & 4.37 $\pm$ 0.01\\
M +10y  & \textbf{4.13 $\pm$ 0.01} & \textbf{4.14 $\pm$ 0.02} & 4.22 $\pm$ 0.01 & 4.31 $\pm$ 0.02 & 4.25 $\pm$ 0.01 & 4.20 $\pm$ 0.02 & 4.20 $\pm$ 0.01 & 4.20 $\pm$ 0.02 & 4.30 $\pm$ 0.01 & 4.35 $\pm$ 0.01\\
M +100y & \textbf{4.10 $\pm$ 0.02} & \textbf{4.10 $\pm$ 0.01} & 4.17 $\pm$ 0.02 & 4.23 $\pm$ 0.01 & 4.20 $\pm$ 0.02 & 4.16 $\pm$ 0.01 & 4.16 $\pm$ 0.02 & 4.17 $\pm$ 0.01 & 4.25 $\pm$ 0.02 & 4.28 $\pm$ 0.02\\
N +5y   & \textbf{3.55 $\pm$ 0.01} & 3.60 $\pm$ 0.02 & 3.60 $\pm$ 0.01 & 3.63 $\pm$ 0.02 & 3.70 $\pm$ 0.01 & 3.71 $\pm$ 0.02 & 3.73 $\pm$ 0.01 & 3.81 $\pm$ 0.02 & 3.76 $\pm$ 0.01 & 3.57 $\pm$ 0.01\\
N +10y  & \textbf{3.50 $\pm$ 0.01} & 3.58 $\pm$ 0.02 & 3.55 $\pm$ 0.01 & 3.61 $\pm$ 0.02 & 3.68 $\pm$ 0.01 & 3.63 $\pm$ 0.02 & 3.63 $\pm$ 0.01 & 3.77 $\pm$ 0.02 & 3.71 $\pm$ 0.01 & 3.52 $\pm$ 0.01\\
N +100y & \textbf{3.48 $\pm$ 0.01} & \textbf{3.48 $\pm$ 0.01} & 3.52 $\pm$ 0.01 & 3.59 $\pm$ 0.02 & 3.67 $\pm$ 0.01 & 3.53 $\pm$ 0.02 & 3.60 $\pm$ 0.01 & 3.61 $\pm$ 0.02 & 3.69 $\pm$ 0.01 & 3.50 $\pm$ 0.01\\
S +5y   & \textbf{2.02 $\pm$ 0.01} & 2.04 $\pm$ 0.01 & 2.06 $\pm$ 0.02 & 2.06 $\pm$ 0.01 & 2.07 $\pm$ 0.02 & 2.06 $\pm$ 0.01 & 2.06 $\pm$ 0.01 & 2.07 $\pm$ 0.02 & 2.06 $\pm$ 0.01 & 2.06 $\pm$ 0.01\\
S +10y  & \textbf{2.01 $\pm$ 0.01} & \textbf{2.03 $\pm$ 0.02} & 2.06 $\pm$ 0.01 & 2.05 $\pm$ 0.01 & 2.06 $\pm$ 0.02 & 2.04 $\pm$ 0.01 & 2.06 $\pm$ 0.01 & 2.07 $\pm$ 0.02 & 2.05 $\pm$ 0.01 & 2.06 $\pm$ 0.01\\
S +100y & \textbf{2.01 $\pm$ 0.01} & \textbf{2.01 $\pm$ 0.01} & 2.04 $\pm$ 0.02 & 2.05 $\pm$ 0.01 & 2.06 $\pm$ 0.01 & 2.04 $\pm$ 0.02 & 2.05 $\pm$ 0.01 & 2.06 $\pm$ 0.02 & 2.05 $\pm$ 0.01 & 2.05 $\pm$ 0.01\\
\bottomrule
\end{tabular}}
\end{table}

\begin{table}[t]
\centering
\caption{MAE-PO across datasets and probe depths, budget = 500. Uniform setting.}
\label{tab:uniform500}
\resizebox{\textwidth}{!}{%
\begin{tabular}{@{}lcccccccccc@{}}
\toprule
\textbf{Dataset} & \textbf{BB surv} & \textbf{BatchBALD} & \textbf{Entropy} & \textbf{Variance} & \textbf{CtH} & \textbf{CfB} & \textbf{MCtM} & \textbf{Random} & \textbf{C-BALD} & \textbf{IDEAL}\\
\midrule
M +5 years   & 3.89 $\pm$ 0.02 & 3.89 $\pm$ 0.02 & 3.89 $\pm$ 0.02 & 3.89 $\pm$ 0.02 & 3.89 $\pm$ 0.02 & 3.89 $\pm$ 0.02 & 3.89 $\pm$ 0.02 & 3.89 $\pm$ 0.02 & 3.89 $\pm$ 0.02 & 3.89 $\pm$ 0.02\\
M +10 years  & 3.89 $\pm$ 0.02 & 3.89 $\pm$ 0.02 & 3.89 $\pm$ 0.02 & 3.89 $\pm$ 0.02 & 3.88 $\pm$ 0.02 & 3.89 $\pm$ 0.02 & 3.89 $\pm$ 0.02 & 3.88 $\pm$ 0.02 & 3.89 $\pm$ 0.02 & 3.89 $\pm$ 0.02\\
M +100 years & 3.85 $\pm$ 0.02 & 3.85 $\pm$ 0.02 & 3.85 $\pm$ 0.02 & 3.85 $\pm$ 0.02 & 3.85 $\pm$ 0.02 & 3.85 $\pm$ 0.02 & 3.85 $\pm$ 0.02 & 3.85 $\pm$ 0.02 & 3.85 $\pm$ 0.02 & 3.85 $\pm$ 0.02\\
N +5 years   & 3.30 $\pm$ 0.02 & 3.30 $\pm$ 0.02 & 3.30 $\pm$ 0.02 & 3.30 $\pm$ 0.02 & 3.31 $\pm$ 0.02 & 3.30 $\pm$ 0.02 & 3.30 $\pm$ 0.02 & 3.30 $\pm$ 0.02 & 3.31 $\pm$ 0.02 & 3.30 $\pm$ 0.02\\
N +10 years  & 3.27 $\pm$ 0.02 & 3.27 $\pm$ 0.02 & 3.27 $\pm$ 0.02 & 3.27 $\pm$ 0.02 & 3.28 $\pm$ 0.02 & 3.27 $\pm$ 0.02 & 3.27 $\pm$ 0.02 & 3.28 $\pm$ 0.02 & 3.27 $\pm$ 0.02 & 3.27 $\pm$ 0.02\\
N +100 years & 3.33 $\pm$ 0.02 & 3.33 $\pm$ 0.02 & 3.33 $\pm$ 0.02 & 3.33 $\pm$ 0.02 & 3.34 $\pm$ 0.02 & 3.33 $\pm$ 0.02 & 3.33 $\pm$ 0.02 & 3.34 $\pm$ 0.02 & 3.33 $\pm$ 0.02 & 3.33 $\pm$ 0.02\\
S +5 years   & 1.88 $\pm$ 0.02 & 1.88 $\pm$ 0.02 & 1.88 $\pm$ 0.02 & 1.78 $\pm$ 0.01 & 1.88 $\pm$ 0.01 & 1.78 $\pm$ 0.01 & 1.88 $\pm$ 0.02 & 1.88 $\pm$ 0.01 & 1.88 $\pm$ 0.02 & 1.88 $\pm$ 0.02\\
S +10 years  & 1.78 $\pm$ 0.01 & 1.88 $\pm$ 0.01 & 1.81 $\pm$ 0.01 & 1.78 $\pm$ 0.01 & 1.88 $\pm$ 0.01 & 1.88 $\pm$ 0.01 & 1.88 $\pm$ 0.01 & 1.81 $\pm$ 0.01 & 1.88 $\pm$ 0.01 & 1.88 $\pm$ 0.01\\
S +100 years & 1.78 $\pm$ 0.02 & 1.78 $\pm$ 0.01 & 1.78 $\pm$ 0.02 & 1.79 $\pm$ 0.01 & 1.79 $\pm$ 0.02 & 1.78 $\pm$ 0.01 & 1.89 $\pm$ 0.02 & 1.79 $\pm$ 0.02 & 1.80 $\pm$ 0.02 & 1.81 $\pm$ 0.02\\
\bottomrule
\end{tabular}}
\end{table}

\begin{table}[t]
\centering
\caption{MAE-PO across datasets and probe depths, budget = 0. Non-uniform setting.}
\label{tab:nonuniform0}
\resizebox{\textwidth}{!}{%
\begin{tabular}{@{}lccccccccc@{}}
\toprule
\textbf{Dataset} & \textbf{BB surv} & \textbf{BatchBALD} & \textbf{Entropy} & \textbf{Variance} & \textbf{CtH} & \textbf{CfB} & \textbf{MCtM} & \textbf{C-BALD} & \textbf{IDEAL}\\
\midrule
M +5 years   & 4.55 $\pm$ 0.05 & 4.55 $\pm$ 0.05 & 4.55 $\pm$ 0.05 & 4.55 $\pm$ 0.05 & 4.55 $\pm$ 0.05 & 4.55 $\pm$ 0.05 & 4.55 $\pm$ 0.05 & 4.55 $\pm$ 0.05 & 4.55 $\pm$ 0.05\\
M +10 years  & 4.55 $\pm$ 0.05 & 4.55 $\pm$ 0.05 & 4.55 $\pm$ 0.05 & 4.55 $\pm$ 0.05 & 4.55 $\pm$ 0.05 & 4.55 $\pm$ 0.05 & 4.55 $\pm$ 0.05 & 4.55 $\pm$ 0.05 & 4.55 $\pm$ 0.05\\
M +100 years & 4.55 $\pm$ 0.05 & 4.55 $\pm$ 0.05 & 4.55 $\pm$ 0.05 & 4.55 $\pm$ 0.05 & 4.55 $\pm$ 0.05 & 4.55 $\pm$ 0.05 & 4.55 $\pm$ 0.05 & 4.55 $\pm$ 0.05 & 4.55 $\pm$ 0.05\\
N +5 years   & 3.98 $\pm$ 0.03 & 3.98 $\pm$ 0.03 & 3.98 $\pm$ 0.03 & 3.98 $\pm$ 0.03 & 3.98 $\pm$ 0.03 & 3.98 $\pm$ 0.03 & 3.98 $\pm$ 0.03 & 3.98 $\pm$ 0.03 & 3.98 $\pm$ 0.03\\
N +10 years  & 3.98 $\pm$ 0.03 & 3.98 $\pm$ 0.03 & 3.98 $\pm$ 0.03 & 3.98 $\pm$ 0.03 & 3.98 $\pm$ 0.03 & 3.98 $\pm$ 0.03 & 3.98 $\pm$ 0.03 & 3.98 $\pm$ 0.03 & 3.98 $\pm$ 0.03\\
N +100 years & 3.98 $\pm$ 0.03 & 3.98 $\pm$ 0.03 & 3.98 $\pm$ 0.03 & 3.98 $\pm$ 0.03 & 3.98 $\pm$ 0.03 & 3.98 $\pm$ 0.03 & 3.98 $\pm$ 0.03 & 3.98 $\pm$ 0.03 & 3.98 $\pm$ 0.03\\
S +5 years   & 2.11 $\pm$ 0.02 & 2.11 $\pm$ 0.02 & 2.11 $\pm$ 0.02 & 2.11 $\pm$ 0.02 & 2.11 $\pm$ 0.02 & 2.11 $\pm$ 0.02 & 2.11 $\pm$ 0.02 & 2.11 $\pm$ 0.02 & 2.11 $\pm$ 0.02\\
S +10 years  & 2.11 $\pm$ 0.02 & 2.11 $\pm$ 0.02 & 2.11 $\pm$ 0.02 & 2.11 $\pm$ 0.02 & 2.11 $\pm$ 0.02 & 2.11 $\pm$ 0.02 & 2.11 $\pm$ 0.02 & 2.11 $\pm$ 0.02 & 2.11 $\pm$ 0.02\\
S +100 years & 2.11 $\pm$ 0.02 & 2.11 $\pm$ 0.02 & 2.11 $\pm$ 0.02 & 2.11 $\pm$ 0.02 & 2.11 $\pm$ 0.02 & 2.11 $\pm$ 0.02 & 2.11 $\pm$ 0.02 & 2.11 $\pm$ 0.02 & 2.11 $\pm$ 0.02\\
\bottomrule
\end{tabular}}
\end{table}

\begin{table}[t]
\centering
\caption{MAE-PO across datasets and probe depths, budget = 1. Non-uniform setting.}
\label{tab:nonuniform1}
\resizebox{\textwidth}{!}{%
\begin{tabular}{@{}lcccccccccc@{}}
\toprule
\textbf{Dataset} & \textbf{BB surv} & \textbf{BatchBALD} & \textbf{Entropy} & \textbf{Variance} & \textbf{CtH} & \textbf{CfB} & \textbf{MCtM} & \textbf{Random} & \textbf{C-BALD} & \textbf{IDEAL}\\
\midrule
M +5 years   & 4.53 $\pm$ 0.05 & 4.53 $\pm$ 0.05 & 4.53 $\pm$ 0.05 & 4.55 $\pm$ 0.05 & 4.54 $\pm$ 0.05 & 4.54 $\pm$ 0.05 & 4.54 $\pm$ 0.05 & 4.54 $\pm$ 0.05 & 4.54 $\pm$ 0.05 & 4.55 $\pm$ 0.05\\
M +10 years  & 4.53 $\pm$ 0.05 & 4.53 $\pm$ 0.05 & 4.53 $\pm$ 0.05 & 4.53 $\pm$ 0.05 & 4.53 $\pm$ 0.05 & 4.52 $\pm$ 0.05 & 4.53 $\pm$ 0.05 & 4.54 $\pm$ 0.05 & 4.54 $\pm$ 0.05 & 4.55 $\pm$ 0.05\\
M +100 years & 4.52 $\pm$ 0.05 & 4.52 $\pm$ 0.05 & 4.52 $\pm$ 0.05 & 4.53 $\pm$ 0.05 & 4.52 $\pm$ 0.05 & 4.52 $\pm$ 0.05 & 4.52 $\pm$ 0.05 & 4.52 $\pm$ 0.05 & 4.53 $\pm$ 0.05 & 4.53 $\pm$ 0.05\\
N +5 years   & 3.94 $\pm$ 0.03 & 3.94 $\pm$ 0.03 & 3.97 $\pm$ 0.03 & 3.96 $\pm$ 0.03 & 3.97 $\pm$ 0.03 & 3.96 $\pm$ 0.03 & 3.96 $\pm$ 0.03 & 3.96 $\pm$ 0.03 & 3.95 $\pm$ 0.03 & 3.96 $\pm$ 0.03\\
N +10 years  & 3.94 $\pm$ 0.03 & 3.94 $\pm$ 0.03 & 3.97 $\pm$ 0.03 & 3.96 $\pm$ 0.03 & 3.95 $\pm$ 0.03 & 3.94 $\pm$ 0.03 & 3.94 $\pm$ 0.03 & 3.94 $\pm$ 0.03 & 3.95 $\pm$ 0.03 & 3.96 $\pm$ 0.03\\
N +100 years & 3.93 $\pm$ 0.03 & 3.93 $\pm$ 0.03 & 3.94 $\pm$ 0.03 & 3.94 $\pm$ 0.03 & 3.95 $\pm$ 0.03 & 3.93 $\pm$ 0.03 & 3.94 $\pm$ 0.03 & 3.94 $\pm$ 0.03 & 3.95 $\pm$ 0.03 & 3.95 $\pm$ 0.03\\
S +5 years   & 2.12 $\pm$ 0.02 & 2.13 $\pm$ 0.02 & 2.12 $\pm$ 0.02 & 2.13 $\pm$ 0.02 & 2.12 $\pm$ 0.02 & 2.12 $\pm$ 0.02 & 2.12 $\pm$ 0.02 & 2.12 $\pm$ 0.02 & 2.12 $\pm$ 0.02 & 2.12 $\pm$ 0.02\\
S +10 years  & 2.12 $\pm$ 0.02 & 2.12 $\pm$ 0.02 & 2.11 $\pm$ 0.02 & 2.12 $\pm$ 0.02 & 2.12 $\pm$ 0.02 & 2.12 $\pm$ 0.02 & 2.12 $\pm$ 0.02 & 2.12 $\pm$ 0.02 & 2.12 $\pm$ 0.02 & 2.12 $\pm$ 0.02\\
S +100 years & 2.12 $\pm$ 0.02 & 2.12 $\pm$ 0.02 & 2.11 $\pm$ 0.02 & 2.12 $\pm$ 0.02 & 2.12 $\pm$ 0.02 & 2.11 $\pm$ 0.02 & 2.12 $\pm$ 0.02 & 2.11 $\pm$ 0.02 & 2.13 $\pm$ 0.02 & 2.12 $\pm$ 0.02\\
\bottomrule
\end{tabular}}
\end{table}

\begin{table}[t]
\centering
\caption{MAE-PO across datasets and probe depths, budget = 5. Non-uniform setting.}
\label{tab:nonuniform5}
\resizebox{\textwidth}{!}{%
\begin{tabular}{@{}lcccccccccc@{}}
\toprule
\textbf{Dataset} & \textbf{BB surv} & \textbf{BatchBALD} & \textbf{Entropy} & \textbf{Variance} & \textbf{CtH} & \textbf{CfB} & \textbf{MCtM} & \textbf{Random} & \textbf{C-BALD} & \textbf{IDEAL}\\
\midrule
M +5 years   & \textbf{4.36 $\pm$ 0.04} & 4.43 $\pm$ 0.04 & 4.39 $\pm$ 0.04 & 4.44 $\pm$ 0.04 & 4.65 $\pm$ 0.04 & 4.43 $\pm$ 0.04 & 4.51 $\pm$ 0.04 & 4.51 $\pm$ 0.04 & 4.47 $\pm$ 0.04 & 4.48 $\pm$ 0.04\\
M +10 years  & 4.31 $\pm$ 0.04 & 4.42 $\pm$ 0.04 & 4.37 $\pm$ 0.04 & 4.39 $\pm$ 0.04 & 4.37 $\pm$ 0.04 & 4.37 $\pm$ 0.04 & 4.42 $\pm$ 0.04 & 4.42 $\pm$ 0.04 & 4.45 $\pm$ 0.04 & 4.46 $\pm$ 0.04\\
M +100 years & \textbf{4.21 $\pm$ 0.04} & \textbf{4.21 $\pm$ 0.04} & 4.26 $\pm$ 0.04 & 4.37 $\pm$ 0.04 & 4.47 $\pm$ 0.04 & 4.29 $\pm$ 0.04 & 4.35 $\pm$ 0.04 & 4.35 $\pm$ 0.04 & 4.37 $\pm$ 0.04 & 4.38 $\pm$ 0.04\\
N +5 years   & \textbf{3.82 $\pm$ 0.02} & \textbf{3.83 $\pm$ 0.02} & 3.87 $\pm$ 0.02 & 3.93 $\pm$ 0.02 & 3.90 $\pm$ 0.02 & 3.94 $\pm$ 0.02 & 3.94 $\pm$ 0.02 & 3.95 $\pm$ 0.02 & 3.90 $\pm$ 0.02 & 3.94 $\pm$ 0.02\\
N +10 years  & \textbf{3.74 $\pm$ 0.02} & \textbf{3.73 $\pm$ 0.02} & 3.83 $\pm$ 0.02 & 3.89 $\pm$ 0.02 & 3.90 $\pm$ 0.02 & 3.90 $\pm$ 0.02 & 3.82 $\pm$ 0.02 & 3.85 $\pm$ 0.02 & 3.85 $\pm$ 0.02 & 3.90 $\pm$ 0.02\\
N +100 years & \textbf{3.72 $\pm$ 0.02} & \textbf{3.72 $\pm$ 0.02} & 3.79 $\pm$ 0.02 & 3.88 $\pm$ 0.02 & 3.88 $\pm$ 0.02 & 3.81 $\pm$ 0.02 & 3.78 $\pm$ 0.02 & 3.79 $\pm$ 0.02 & 3.84 $\pm$ 0.02 & 3.85 $\pm$ 0.02\\
S +5 years   & \textbf{2.09 $\pm$ 0.01} & 2.11 $\pm$ 0.01 & 2.11 $\pm$ 0.01 & 2.11 $\pm$ 0.01 & 2.11 $\pm$ 0.01 & 2.11 $\pm$ 0.01 & 2.10 $\pm$ 0.01 & 2.10 $\pm$ 0.01 & 2.10 $\pm$ 0.01 & 2.11 $\pm$ 0.01\\
S +10 years  & \textbf{2.09 $\pm$ 0.01} & \textbf{2.10 $\pm$ 0.01} & 2.11 $\pm$ 0.01 & 2.10 $\pm$ 0.01 & 2.10 $\pm$ 0.01 & 2.11 $\pm$ 0.01 & 2.10 $\pm$ 0.01 & 2.10 $\pm$ 0.01 & 2.11 $\pm$ 0.01 & 2.11 $\pm$ 0.01\\
S +100 years & \textbf{2.08 $\pm$ 0.01} & \textbf{2.08 $\pm$ 0.01} & 2.11 $\pm$ 0.01 & 2.10 $\pm$ 0.01 & 2.10 $\pm$ 0.01 & 2.11 $\pm$ 0.01 & 2.10 $\pm$ 0.01 & 2.10 $\pm$ 0.01 & 2.10 $\pm$ 0.01 & 2.10 $\pm$ 0.01\\
\bottomrule
\end{tabular}}
\end{table}

\begin{table}[t]
\centering
\caption{MAE-PO across datasets and probe depths, budget = 10. Non-uniform setting.}
\label{tab:nonuniform10}
\resizebox{\textwidth}{!}{%
\begin{tabular}{@{}lcccccccccc@{}}
\toprule
\textbf{Dataset} & \textbf{BB surv} & \textbf{BatchBALD} & \textbf{Entropy} & \textbf{Variance} & \textbf{CtH} & \textbf{CfB} & \textbf{MCtM} & \textbf{Random} & \textbf{C-BALD} & \textbf{IDEAL}\\
\midrule
M +5 years   & \textbf{4.23 $\pm$ 0.02} & 4.34 $\pm$ 0.02 & 4.28 $\pm$ 0.02 & 4.28 $\pm$ 0.02 & 4.45 $\pm$ 0.02 & 4.32 $\pm$ 0.02 & 4.29 $\pm$ 0.02 & 4.45 $\pm$ 0.02 & 4.34 $\pm$ 0.02 & 4.45 $\pm$ 0.02\\
M +10 years  & \textbf{4.24 $\pm$ 0.02} & 4.35 $\pm$ 0.02 & 4.28 $\pm$ 0.02 & 4.33 $\pm$ 0.02 & 4.31 $\pm$ 0.02 & 4.27 $\pm$ 0.02 & 4.28 $\pm$ 0.02 & 4.32 $\pm$ 0.02 & 4.35 $\pm$ 0.02 & 4.40 $\pm$ 0.02\\
M +100 years & \textbf{4.18 $\pm$ 0.02} & \textbf{4.18 $\pm$ 0.02} & \textbf{4.18 $\pm$ 0.02} & 4.27 $\pm$ 0.02 & 4.27 $\pm$ 0.02 & 4.23 $\pm$ 0.02 & 4.19 $\pm$ 0.02 & 4.31 $\pm$ 0.02 & 4.29 $\pm$ 0.02 & 4.35 $\pm$ 0.02\\
N +5 years   & \textbf{3.63 $\pm$ 0.01} & 3.67 $\pm$ 0.01 & 3.64 $\pm$ 0.01 & 3.66 $\pm$ 0.01 & 3.73 $\pm$ 0.01 & 3.80 $\pm$ 0.01 & 3.81 $\pm$ 0.01 & 3.85 $\pm$ 0.01 & 3.84 $\pm$ 0.01 & 3.85 $\pm$ 0.01\\
N +10 years  & \textbf{3.59 $\pm$ 0.01} & 3.63 $\pm$ 0.01 & 3.62 $\pm$ 0.01 & 3.65 $\pm$ 0.01 & 3.71 $\pm$ 0.01 & 3.77 $\pm$ 0.01 & 3.62 $\pm$ 0.01 & 3.80 $\pm$ 0.01 & 3.70 $\pm$ 0.01 & 3.71 $\pm$ 0.01\\
N +100 years & \textbf{3.57 $\pm$ 0.01} & \textbf{3.57 $\pm$ 0.01} & 3.60 $\pm$ 0.01 & 3.63 $\pm$ 0.01 & 3.70 $\pm$ 0.01 & 3.65 $\pm$ 0.01 & 3.68 $\pm$ 0.01 & 3.65 $\pm$ 0.01 & 3.68 $\pm$ 0.01 & 3.69 $\pm$ 0.01\\
S +5 years   & \textbf{2.07 $\pm$ 0.01} & \textbf{2.08 $\pm$ 0.01} & 2.10 $\pm$ 0.01 & 2.09 $\pm$ 0.01 & 2.09 $\pm$ 0.01 & 2.10 $\pm$ 0.01 & 2.09 $\pm$ 0.01 & 2.10 $\pm$ 0.01 & 2.10 $\pm$ 0.01 & 2.10 $\pm$ 0.01\\
S +10 years  & \textbf{2.06 $\pm$ 0.01} & \textbf{2.08 $\pm$ 0.01} & 2.10 $\pm$ 0.01 & 2.09 $\pm$ 0.01 & 2.09 $\pm$ 0.01 & 2.10 $\pm$ 0.01 & 2.09 $\pm$ 0.01 & 2.10 $\pm$ 0.01 & 2.10 $\pm$ 0.01 & 2.10 $\pm$ 0.01\\
S +100 years & \textbf{2.06 $\pm$ 0.01} & \textbf{2.06 $\pm$ 0.01} & 2.10 $\pm$ 0.01 & 2.09 $\pm$ 0.01 & 2.10 $\pm$ 0.01 & 2.10 $\pm$ 0.01 & 2.08 $\pm$ 0.01 & 2.10 $\pm$ 0.01 & 2.09 $\pm$ 0.01 & 2.10 $\pm$ 0.01\\
\bottomrule
\end{tabular}}
\end{table}

\begin{table}[t]
\centering
\caption{MAE-PO across datasets and probe depths, budget = 15. Non-uniform setting.}
\label{tab:nonuniform15}
\resizebox{\textwidth}{!}{%
\begin{tabular}{@{}lcccccccccc@{}}
\toprule
\textbf{Dataset} & \textbf{BB surv} & \textbf{BatchBALD} & \textbf{Entropy} & \textbf{Variance} & \textbf{CtH} & \textbf{CfB} & \textbf{MCtM} & \textbf{Random} & \textbf{C-BALD} & \textbf{IDEAL}\\
\midrule
M +5 years   & \textbf{4.19 $\pm$ 0.01} & 4.23 $\pm$ 0.02 & 4.27 $\pm$ 0.02 & 4.26 $\pm$ 0.01 & 4.42 $\pm$ 0.02 & 4.22 $\pm$ 0.02 & 4.27 $\pm$ 0.01 & 4.46 $\pm$ 0.01 & 4.35 $\pm$ 0.01 & 4.41 $\pm$ 0.01\\
M +10 years  & \textbf{4.18 $\pm$ 0.02} & \textbf{4.22 $\pm$ 0.02} & 4.25 $\pm$ 0.02 & 4.32 $\pm$ 0.01 & 4.28 $\pm$ 0.01 & 4.21 $\pm$ 0.02 & 4.23 $\pm$ 0.02 & 4.23 $\pm$ 0.02 & 4.35 $\pm$ 0.02 & 4.40 $\pm$ 0.02\\
M +100 years & \textbf{4.15 $\pm$ 0.01} & \textbf{4.15 $\pm$ 0.01} & 4.18 $\pm$ 0.02 & 4.24 $\pm$ 0.01 & 4.23 $\pm$ 0.02 & 4.17 $\pm$ 0.01 & 4.17 $\pm$ 0.01 & 4.20 $\pm$ 0.01 & 4.26 $\pm$ 0.01 & 4.32 $\pm$ 0.01\\
N +5 years   & \textbf{3.60 $\pm$ 0.01} & 3.64 $\pm$ 0.01 & 3.63 $\pm$ 0.01 & 3.65 $\pm$ 0.01 & 3.72 $\pm$ 0.01 & 3.73 $\pm$ 0.01 & 3.75 $\pm$ 0.01 & 3.83 $\pm$ 0.01 & 3.71 $\pm$ 0.01 & 3.72 $\pm$ 0.01\\
N +10 years  & \textbf{3.57 $\pm$ 0.01} & 3.62 $\pm$ 0.01 & 3.60 $\pm$ 0.01 & 3.63 $\pm$ 0.01 & 3.70 $\pm$ 0.01 & 3.65 $\pm$ 0.01 & 3.65 $\pm$ 0.01 & 3.79 $\pm$ 0.01 & 3.68 $\pm$ 0.01 & 3.69 $\pm$ 0.01\\
N +100 years & \textbf{3.53 $\pm$ 0.01} & \textbf{3.53 $\pm$ 0.01} & 3.59 $\pm$ 0.01 & 3.61 $\pm$ 0.01 & 3.69 $\pm$ 0.01 & 3.55 $\pm$ 0.01 & 3.62 $\pm$ 0.01 & 3.63 $\pm$ 0.01 & 3.54 $\pm$ 0.01 & 3.65 $\pm$ 0.01\\
S +5 years   & \textbf{2.06 $\pm$ 0.01} & 2.08 $\pm$ 0.01 & 2.08 $\pm$ 0.01 & 2.08 $\pm$ 0.01 & 2.08 $\pm$ 0.01 & 2.08 $\pm$ 0.01 & 2.08 $\pm$ 0.01 & 2.09 $\pm$ 0.01 & 2.09 $\pm$ 0.01 & 2.09 $\pm$ 0.01\\
S +10 years  & \textbf{2.06 $\pm$ 0.01} & \textbf{2.07 $\pm$ 0.01} & 2.08 $\pm$ 0.01 & 2.08 $\pm$ 0.01 & 2.08 $\pm$ 0.01 & 2.07 $\pm$ 0.01 & 2.08 $\pm$ 0.01 & 2.09 $\pm$ 0.01 & 2.08 $\pm$ 0.01 & 2.09 $\pm$ 0.01\\
S +100 years & \textbf{2.04 $\pm$ 0.01} & \textbf{2.04 $\pm$ 0.01} & 2.08 $\pm$ 0.01 & 2.07 $\pm$ 0.01 & 2.07 $\pm$ 0.01 & 2.06 $\pm$ 0.01 & 2.09 $\pm$ 0.01 & 2.08 $\pm$ 0.01 & 2.08 $\pm$ 0.01 & 2.08 $\pm$ 0.01\\
\bottomrule
\end{tabular}}
\end{table}

\begin{table}[t]
\centering
\caption{MAE-PO across datasets and probe depths, budget = 20. Non-uniform setting.}
\label{tab:nonuniform20}
\resizebox{\textwidth}{!}{%
\begin{tabular}{@{}lcccccccccc@{}}
\toprule
\textbf{Dataset} & \textbf{BB surv} & \textbf{BatchBALD} & \textbf{Entropy} & \textbf{Variance} & \textbf{CtH} & \textbf{CfB} & \textbf{MCtM} & \textbf{Random} & \textbf{C-BALD} & \textbf{IDEAL}\\
\midrule
M +5 years   & \textbf{4.15 $\pm$ 0.01} & 4.22 $\pm$ 0.01 & 4.25 $\pm$ 0.01 & 4.25 $\pm$ 0.01 & 4.41 $\pm$ 0.01 & 4.19 $\pm$ 0.01 & 4.25 $\pm$ 0.01 & 4.43 $\pm$ 0.01 & 4.41 $\pm$ 0.01 & 4.43 $\pm$ 0.01\\
M +10 years  & \textbf{4.14 $\pm$ 0.02} & 4.20 $\pm$ 0.02 & 4.22 $\pm$ 0.02 & 4.31 $\pm$ 0.02 & 4.25 $\pm$ 0.02 & 4.20 $\pm$ 0.02 & 4.20 $\pm$ 0.02 & 4.20 $\pm$ 0.02 & 4.29 $\pm$ 0.02 & 4.36 $\pm$ 0.02\\
M +100 years & \textbf{4.10 $\pm$ 0.02} & \textbf{4.10 $\pm$ 0.02} & 4.17 $\pm$ 0.02 & 4.23 $\pm$ 0.02 & 4.20 $\pm$ 0.02 & 4.16 $\pm$ 0.02 & 4.16 $\pm$ 0.02 & 4.17 $\pm$ 0.02 & 4.25 $\pm$ 0.02 & 4.29 $\pm$ 0.02\\
N +5 years   & \textbf{3.55 $\pm$ 0.01} & 3.60 $\pm$ 0.01 & 3.60 $\pm$ 0.01 & 3.63 $\pm$ 0.01 & 3.70 $\pm$ 0.01 & 3.71 $\pm$ 0.01 & 3.73 $\pm$ 0.01 & 3.81 $\pm$ 0.01 & 3.76 $\pm$ 0.01 & 3.77 $\pm$ 0.01\\
N +10 years  & \textbf{3.50 $\pm$ 0.01} & 3.58 $\pm$ 0.01 & 3.55 $\pm$ 0.01 & 3.61 $\pm$ 0.01 & 3.68 $\pm$ 0.01 & 3.63 $\pm$ 0.01 & 3.63 $\pm$ 0.01 & 3.77 $\pm$ 0.01 & 3.65 $\pm$ 0.01 & 3.72 $\pm$ 0.01\\
N +100 years & \textbf{3.48 $\pm$ 0.01} & \textbf{3.48 $\pm$ 0.01} & 3.52 $\pm$ 0.01 & 3.59 $\pm$ 0.01 & 3.67 $\pm$ 0.01 & 3.53 $\pm$ 0.01 & 3.60 $\pm$ 0.01 & 3.61 $\pm$ 0.01 & 3.55 $\pm$ 0.01 & 3.56 $\pm$ 0.01\\
S +5 years   & \textbf{2.02 $\pm$ 0.01} & 2.04 $\pm$ 0.02 & 2.06 $\pm$ 0.02 & 2.06 $\pm$ 0.02 & 2.07 $\pm$ 0.02 & 2.06 $\pm$ 0.02 & 2.06 $\pm$ 0.02 & 2.07 $\pm$ 0.02 & 2.07 $\pm$ 0.02 & 2.08 $\pm$ 0.02\\
S +10 years  & \textbf{2.01 $\pm$ 0.02} & \textbf{2.03 $\pm$ 0.02} & 2.06 $\pm$ 0.02 & 2.05 $\pm$ 0.02 & 2.06 $\pm$ 0.02 & 2.04 $\pm$ 0.02 & 2.06 $\pm$ 0.02 & 2.07 $\pm$ 0.02 & 2.07 $\pm$ 0.02 & 2.07 $\pm$ 0.02\\
S +100 years & \textbf{2.01 $\pm$ 0.02} & \textbf{2.01 $\pm$ 0.02} & 2.04 $\pm$ 0.02 & 2.05 $\pm$ 0.02 & 2.06 $\pm$ 0.02 & 2.04 $\pm$ 0.02 & 2.05 $\pm$ 0.02 & 2.06 $\pm$ 0.02 & 2.05 $\pm$ 0.02 & 2.05 $\pm$ 0.02\\
\bottomrule
\end{tabular}}
\end{table}

\begin{table}[t]
\centering
\caption{MAE-PO across datasets and probe depths, budget = 500. Non-Uniform setting.}
\label{tab:nonuniform500}
\resizebox{\textwidth}{!}{%
\begin{tabular}{@{}lcccccccccc@{}}
\toprule
\textbf{Dataset} & \textbf{BB surv} & \textbf{BatchBALD} & \textbf{Entropy} & \textbf{Variance} & \textbf{CtH} & \textbf{CfB} & \textbf{MCtM} & \textbf{Random} & \textbf{C-BALD} & \textbf{IDEAL}\\
\midrule
M +5 years   & 3.89 $\pm$ 0.02 & 3.89 $\pm$ 0.02 & 3.89 $\pm$ 0.02 & 3.89 $\pm$ 0.02 & 3.89 $\pm$ 0.02 & 3.89 $\pm$ 0.02 & 3.89 $\pm$ 0.02 & 3.89 $\pm$ 0.02 & 3.89 $\pm$ 0.02 & 3.89 $\pm$ 0.02\\
M +10 years  & 3.89 $\pm$ 0.02 & 3.89 $\pm$ 0.02 & 3.89 $\pm$ 0.02 & 3.89 $\pm$ 0.02 & 3.88 $\pm$ 0.02 & 3.89 $\pm$ 0.02 & 3.89 $\pm$ 0.02 & 3.88 $\pm$ 0.02 & 3.89 $\pm$ 0.02 & 3.89 $\pm$ 0.02\\
M +100 years & 3.85 $\pm$ 0.02 & 3.85 $\pm$ 0.02 & 3.85 $\pm$ 0.02 & 3.85 $\pm$ 0.02 & 3.85 $\pm$ 0.02 & 3.85 $\pm$ 0.02 & 3.85 $\pm$ 0.02 & 3.85 $\pm$ 0.02 & 3.85 $\pm$ 0.02 & 3.85 $\pm$ 0.02\\
N +5 years   & 3.30 $\pm$ 0.02 & 3.30 $\pm$ 0.02 & 3.30 $\pm$ 0.02 & 3.30 $\pm$ 0.02 & 3.31 $\pm$ 0.02 & 3.30 $\pm$ 0.02 & 3.30 $\pm$ 0.02 & 3.30 $\pm$ 0.02 & 3.31 $\pm$ 0.02 & 3.30 $\pm$ 0.02\\
N +10 years  & 3.27 $\pm$ 0.02 & 3.27 $\pm$ 0.02 & 3.27 $\pm$ 0.02 & 3.27 $\pm$ 0.02 & 3.28 $\pm$ 0.02 & 3.27 $\pm$ 0.02 & 3.27 $\pm$ 0.02 & 3.28 $\pm$ 0.02 & 3.27 $\pm$ 0.02 & 3.27 $\pm$ 0.02\\
N +100 years & 3.33 $\pm$ 0.02 & 3.33 $\pm$ 0.02 & 3.33 $\pm$ 0.02 & 3.33 $\pm$ 0.02 & 3.34 $\pm$ 0.02 & 3.33 $\pm$ 0.02 & 3.33 $\pm$ 0.02 & 3.34 $\pm$ 0.02 & 3.33 $\pm$ 0.02 & 3.33 $\pm$ 0.02\\
S +5 years   & 1.88 $\pm$ 0.02 & 1.88 $\pm$ 0.02 & 1.88 $\pm$ 0.02 & 1.78 $\pm$ 0.01 & 1.88 $\pm$ 0.01 & 1.78 $\pm$ 0.01 & 1.88 $\pm$ 0.02 & 1.88 $\pm$ 0.01 & 1.88 $\pm$ 0.02 & 1.88 $\pm$ 0.02\\
S +10 years  & 1.78 $\pm$ 0.01 & 1.88 $\pm$ 0.01 & 1.81 $\pm$ 0.01 & 1.78 $\pm$ 0.01 & 1.88 $\pm$ 0.01 & 1.88 $\pm$ 0.01 & 1.88 $\pm$ 0.01 & 1.81 $\pm$ 0.01 & 1.88 $\pm$ 0.01 & 1.88 $\pm$ 0.01\\
S +100 years & 1.78 $\pm$ 0.02 & 1.78 $\pm$ 0.01 & 1.78 $\pm$ 0.02 & 1.79 $\pm$ 0.01 & 1.79 $\pm$ 0.02 & 1.78 $\pm$ 0.01 & 1.89 $\pm$ 0.02 & 1.79 $\pm$ 0.02 & 1.80 $\pm$ 0.02 & 1.81 $\pm$ 0.02\\
\bottomrule
\end{tabular}}
\end{table}

\begin{table}[t]
\centering
\caption{Performance metrics across datasets at a probe depth of $+10$ years with a budget of 20 (uniform setting).}
\label{Other-performance-metrics-table-U}
\resizebox{\textwidth}{!}{%
\begin{tabular}{@{}llcccccccccc@{}}
\toprule
\textbf{Dataset} & \textbf{Metric} &
\textbf{BB surv} & \textbf{BatchBALD} & \textbf{Entropy} & \textbf{Variance} &
\textbf{CtH} & \textbf{CfB} & \textbf{MCtH} & \textbf{Random} &
\textbf{C-BALD} & \textbf{IDEAL} \\
\midrule
\multirow{3}{*}{MIMIC +10 y}%
  & C-index                  & 0.59 $\pm$ 0.02 & 0.57 $\pm$ 0.02 & 0.57 $\pm$ 0.02 & 0.57 $\pm$ 0.02 & 0.57 $\pm$ 0.02 & 0.60 $\pm$ 0.02 & 0.58 $\pm$ 0.02 & 0.56 $\pm$ 0.02 & 0.57 $\pm$ 0.02 & 0.57 $\pm$ 0.02 \\
  & Integrated Brier Score   & \textbf{0.25 $\pm$ 0.01} & 0.28 $\pm$ 0.01 & 0.31 $\pm$ 0.01 & 0.30 $\pm$ 0.01 & 0.29 $\pm$ 0.01 & 0.26 $\pm$ 0.01 & 0.28 $\pm$ 0.01 & 0.28 $\pm$ 0.01 & 0.28 $\pm$ 0.01 & 0.27 $\pm$ 0.01 \\
  & Uncensored MAE-PO        & \textbf{4.21 $\pm$ 0.03} & 4.27 $\pm$ 0.02 & 4.30 $\pm$ 0.03 & 4.36 $\pm$ 0.02 & 4.32 $\pm$ 0.03 & 4.27 $\pm$ 0.03 & \textbf{4.25 $\pm$ 0.03} & 4.28 $\pm$ 0.01 & 4.27 $\pm$ 0.01 & 4.29 $\pm$ 0.01 \\
\midrule
\multirow{3}{*}{NACD +10 y}%
  & C-index                  & \textbf{0.61 $\pm$ 0.02} & \textbf{0.59 $\pm$ 0.02} & \textbf{0.59 $\pm$ 0.02} & 0.56 $\pm$ 0.02 & 0.56 $\pm$ 0.02 & 0.56 $\pm$ 0.02 & \textbf{0.59 $\pm$ 0.02} & 0.57 $\pm$ 0.02 & 0.57 $\pm$ 0.02 & 0.57 $\pm$ 0.02 \\
  & Integrated Brier Score   & \textbf{0.24 $\pm$ 0.01} & 0.30 $\pm$ 0.01 & 0.32 $\pm$ 0.01 & 0.32 $\pm$ 0.01 & 0.30 $\pm$ 0.01 & 0.28 $\pm$ 0.01 & 0.27 $\pm$ 0.01 & 0.29 $\pm$ 0.01 & 0.26 $\pm$ 0.01 & 0.27 $\pm$ 0.01 \\
  & Uncensored MAE-PO        & \textbf{3.66 $\pm$ 0.03} & 3.70 $\pm$ 0.03 & \textbf{3.69 $\pm$ 0.03} & 3.74 $\pm$ 0.03 & 3.84 $\pm$ 0.03 & 3.76 $\pm$ 0.03 & 3.73 $\pm$ 0.03 & 3.74 $\pm$ 0.03 & 3.71 $\pm$ 0.03 & 3.73 $\pm$ 0.03 \\
\midrule
\multirow{3}{*}{SUPPORT +10 y}%
  & C-index                  & \textbf{0.58 $\pm$ 0.02} & 0.56 $\pm$ 0.02 & 0.56 $\pm$ 0.02 & 0.55 $\pm$ 0.02 & 0.56 $\pm$ 0.02 & \textbf{0.58 $\pm$ 0.02} & 0.56 $\pm$ 0.02 & 0.55 $\pm$ 0.02 & 0.54 $\pm$ 0.02 & 0.56 $\pm$ 0.02 \\
  & Integrated Brier Score   & \textbf{0.27 $\pm$ 0.01} & 0.33 $\pm$ 0.01 & 0.36 $\pm$ 0.01 & 0.35 $\pm$ 0.01 & 0.30 $\pm$ 0.01 & 0.32 $\pm$ 0.01 & 0.29 $\pm$ 0.01 & 0.29 $\pm$ 0.01 & 0.30 $\pm$ 0.01 & 0.29 $\pm$ 0.01 \\
  & Uncensored MAE-PO        & \textbf{2.07 $\pm$ 0.01} & \textbf{2.07 $\pm$ 0.02} & 2.12 $\pm$ 0.02 & 2.09 $\pm$ 0.01 & 2.10 $\pm$ 0.02 & 2.10 $\pm$ 0.02 & 2.10 $\pm$ 0.02 & 2.12 $\pm$ 0.03 & 2.09 $\pm$ 0.02 & 2.10 $\pm$ 0.03 \\
\bottomrule
\end{tabular}}
\end{table}

\begin{table}[t]
\centering
\caption{Performance metrics across datasets at a probe depth of $+10$ years with a budget of 20 (non-uniform setting).}
\label{Other-performance-metrics-table-NU}
\resizebox{\textwidth}{!}{%
\begin{tabular}{@{}llcccccccccc@{}}
\toprule
\textbf{Dataset} & \textbf{Metric} &
\textbf{BB surv} & \textbf{BatchBALD} & \textbf{Entropy} & \textbf{Variance} &
\textbf{CtH} & \textbf{CfB} & \textbf{MCtH} & \textbf{Random} &
\textbf{C-BALD} & \textbf{IDEAL} \\
\midrule
\multirow{3}{*}{MIMIC +10 y}%
  & C-index                  & \textbf{0.57 $\pm$ 0.01} & 0.55 $\pm$ 0.01 & 0.56 $\pm$ 0.01 & 0.55 $\pm$ 0.03 & 0.56 $\pm$ 0.01 & 0.56 $\pm$ 0.01 & 0.55 $\pm$ 0.03 & 0.55 $\pm$ 0.02 & 0.56 $\pm$ 0.01 & 0.54 $\pm$ 0.01 \\
  & Integrated Brier Score   & \textbf{0.26 $\pm$ 0.01} & 0.30 $\pm$ 0.01 & 0.33 $\pm$ 0.02 & 0.32 $\pm$ 0.01 & 0.30 $\pm$ 0.02 & 0.28 $\pm$ 0.01 & 0.29 $\pm$ 0.02 & 0.29 $\pm$ 0.01 & 0.28 $\pm$ 0.02 & 0.28 $\pm$ 0.01 \\
  & Uncensored MAE-PO        & \textbf{4.24 $\pm$ 0.02} & 4.30 $\pm$ 0.02 & 4.33 $\pm$ 0.03 & 4.39 $\pm$ 0.02 & 4.36 $\pm$ 0.02 & 4.30 $\pm$ 0.02 & 4.28 $\pm$ 0.02 & 4.31 $\pm$ 0.03 & 4.31 $\pm$ 0.02 & 4.33 $\pm$ 0.02 \\
\midrule
\multirow{3}{*}{NACD +10 y}%
  & C-index                  & \textbf{0.59 $\pm$ 0.02} & \textbf{0.58 $\pm$ 0.01} & \textbf{0.58 $\pm$ 0.02} & 0.54 $\pm$ 0.02 & 0.55 $\pm$ 0.01 & 0.55 $\pm$ 0.02 & 0.57 $\pm$ 0.03 & 0.55 $\pm$ 0.01 & 0.54 $\pm$ 0.02 & 0.55 $\pm$ 0.01 \\
  & Integrated Brier Score   & \textbf{0.25 $\pm$ 0.01} & 0.32 $\pm$ 0.01 & 0.34 $\pm$ 0.02 & 0.33 $\pm$ 0.01 & 0.31 $\pm$ 0.02 & 0.30 $\pm$ 0.01 & 0.29 $\pm$ 0.01 & 0.31 $\pm$ 0.02 & 0.33 $\pm$ 0.01 & 0.31 $\pm$ 0.02 \\
  & Uncensored MAE-PO        & \textbf{3.66 $\pm$ 0.02} & 3.73 $\pm$ 0.03 & 3.72 $\pm$ 0.04 & 3.78 $\pm$ 0.02 & 3.86 $\pm$ 0.03 & 3.79 $\pm$ 0.02 & 3.75 $\pm$ 0.04 & 3.77 $\pm$ 0.03 & 3.78 $\pm$ 0.03 & 3.72 $\pm$ 0.02 \\
\midrule
\multirow{3}{*}{SUPPORT +10 y}%
  & C-index                  & \textbf{0.56 $\pm$ 0.01} & 0.54 $\pm$ 0.03 & 0.55 $\pm$ 0.02 & 0.53 $\pm$ 0.02 & 0.55 $\pm$ 0.01 & \textbf{0.57 $\pm$ 0.02} & 0.55 $\pm$ 0.01 & 0.54 $\pm$ 0.03 & 0.55 $\pm$ 0.01 & 0.55 $\pm$ 0.01 \\
  & Integrated Brier Score   & 0.31 $\pm$ 0.01 & 0.34 $\pm$ 0.02 & 0.37 $\pm$ 0.02 & 0.36 $\pm$ 0.01 & 0.32 $\pm$ 0.01 & 0.33 $\pm$ 0.01 & 0.30 $\pm$ 0.02 & 0.30 $\pm$ 0.01 & 0.33 $\pm$ 0.02 & 0.32 $\pm$ 0.01 \\
  & Uncensored MAE-PO        & \textbf{2.10 $\pm$ 0.02} & \textbf{2.10 $\pm$ 0.03} & 2.15 $\pm$ 0.04 & 2.12 $\pm$ 0.02 & 2.13 $\pm$ 0.03 & 2.14 $\pm$ 0.02 & 2.14 $\pm$ 0.03 & 2.15 $\pm$ 0.02 & 2.12 $\pm$ 0.01 & 2.13 $\pm$ 0.01 \\
\bottomrule
\end{tabular}}
\end{table}

\end{document}